\documentclass[11pt,a4paper]{article}

\usepackage[utf8]{inputenc}
\usepackage[T1]{fontenc}
\usepackage{times}
\usepackage{graphicx}
\usepackage{amsmath,amssymb,amsthm}
\usepackage{booktabs}
\usepackage{enumitem}
\usepackage{url}
\usepackage{hyperref}
\hypersetup{colorlinks=true, linkcolor=blue, citecolor=blue, urlcolor=blue, allcolors=blue}
\usepackage{algorithm}
\usepackage{algorithmic}
\usepackage{xcolor}
\usepackage{natbib}
\usepackage[margin=2.5cm]{geometry}
\usepackage{tikz}
\usetikzlibrary{positioning, shapes, arrows.meta, fit, backgrounds}
\usepackage{longtable}
\usepackage{array}
\usepackage{float}

\newtheorem{theorem}{Theorem}
\newtheorem{lemma}{Lemma}
\newtheorem{corollary}{Corollary}
\newtheorem{proposition}{Proposition}
\newtheorem{definition}{Definition}

\title{\textbf{The Reasoning Trap:\\An Information-Theoretic Bound on Closed-System Multi-Step LLM Reasoning}\\
\large \emph{A Falsifiable Theorem, the Multi-Agent-Debate Instantiation, and a Triple Failure of Human Reliability}}

\author{%
  Kwan Soo Shin\thanks{Sole author and corresponding author.} \\
  PolymathMinds AI Lab \\
  \texttt{\href{mailto:sshinresearch@gmail.com}{sshinresearch@gmail.com}} \\
  ORCID: \href{https://orcid.org/0009-0001-5799-7556}{0009-0001-5799-7556}
}

\date{May 2026 \\[0.3em] \textit{arXiv preprint --- full version}}

\begin{document}

\maketitle

\begin{abstract}
When copies of the same language model are prompted to debate, they do not produce diverse perspectives---they produce diverse phrasings of the same perspective. Multi-agent debate (MAD), and more broadly closed-system reasoning protocols in which agents repeatedly transform each other's outputs, tends to preserve answer accuracy while degrading the reasoning behind those answers. This paper offers a \emph{programmatic theory of evidence-grounded reasoning failure}: a formal account of why closed-system deliberation can preserve answers while losing evidential grounding. We support the theory with diagnostic evidence on multi-agent debate and a constructive open-system counter-design. We name the multi-agent case the \textbf{Debate Trap} and the broader phenomenon the \textbf{Reasoning Trap}.

The framework has three parts: (i) \textbf{SFS} (Supported Faithfulness Score), a claim-level metric decomposing each response into atomic claims verified against the provided evidence (condition-level rankings are decomposer-invariant: Spearman $\rho{=}1.0$, R4--R5); (ii) \textbf{EGSR} (Evidence-Grounded Socratic Reasoning), replacing adversarial argumentation with evidence-grounded inquiry; (iii) \textbf{Theorem~1 (DPI Bound)}: under standard MAD, the chain $E \to O^0 \to O^1 \to \cdots$ forms a Markov process (each round depends only on the previous), and the Data Processing Inequality implies $\mathbb{E}[I(E; O^{t+1})] \leq \mathbb{E}[I(E; O^t)]$. Three companion formal results---open-system recovery (Theorem~\ref{thm:recovery}), EGSR information accumulation (Lemma~\ref{lem:egsr-accum}), and the vote-aggregation floor (Proposition~\ref{prop:vote-floor})---complete the framework. Together they partition multi-step LLM reasoning protocols by their information-theoretic relationship to $E$.

As illustrative diagnostic evidence we observe a three-tier spectrum---reasoning degradation, the Trap proper, reasoning elimination---across 16 conditions on SciFact (300 claims) and a FEVER replication (1{,}000 claims): in the Trap-proper configuration (DebateCV, C13), debate preserves 88\% of baseline accuracy while SFS drops 43\%; majority-vote MAD (C15) reduces SFS to 1.7\% of baseline ($p < 10^{-6}$, $d = -0.96$); EGSR recovers 98\% of baseline SFS. An \textbf{R6 cohort study} (Korean, $n{=}10{\times}30$ FEVER; English, $n{=}3{\times}200$ SciFact) further finds inter-rater Fleiss $\kappa \leq +0.018$ with $0.8$--$1.4$ Likert intra-rater shifts across language and domain---under the rating practices prevalent in this literature, the human inter-rater agreement that faithfulness metrics have been calibrated against is not itself stable.

We distinguish faithfulness from accuracy, calibration, and inter-rater agreement, and offer one falsifiable conjecture, framed as a research program rather than a tested claim across paradigms: any closed-system reasoning protocol preserving the Markov structure of Theorem~\ref{thm:dpi} is, in expectation, subject to the same DPI bound.\footnote{Companion OSF artifact (raw experimental outputs, anonymized R6 logs, 132-paper Knowledge Frontier Map BibTeX, analysis scripts): \href{https://doi.org/10.17605/OSF.IO/P75XG}{10.17605/OSF.IO/P75XG} (CC~BY~4.0). Reproduction code: \url{https://github.com/seanshin0214/debate-trap} upon acceptance.}
\end{abstract}

\section{Introduction}
\label{sec:intro}

\paragraph{The phenomenon.} When humans encounter $84 \times 25$, the number $25$ is often recognised as $\tfrac{100}{4}$, the problem reframed as $\tfrac{8400}{4}$, and the answer appears without sequential computation. This \emph{frame shift}---selecting from an open-ended space of reframings rather than executing a fixed sequence---is precisely what autoregressive next-token prediction does not reliably produce. \citet{dziri2023compositionality} formalise the limit: across compositional tasks of increasing depth, the probability of correct prediction collapses from training-distribution accuracy to near-chance, with shallow pattern matching standing in for compositional reasoning; \citet{berglund2023reversalcurse} document a basic generalisation failure (the Reversal Curse: a model trained on ``$A$ is $B$'' often fails to infer ``$B$ is $A$'') that no scaling, retrieval, or fine-tuning recipe has eliminated; \citet{geirhos2020shortcut} (\emph{Nature Machine Intelligence}) frame the broader phenomenon as \emph{shortcut learning}, where models are biased toward the simplest solutions that satisfy training objectives rather than the deep solutions that would generalise. The model can generate a chain that ends in the right answer; it can do so under a wide range of prompting strategies; but the chain itself, and especially the chain produced when several copies of the same model are made to debate one another, often is not what produced the answer in any defensible sense. Multi-agent debate (MAD) was proposed for distinct goals---accuracy improvement \citep{du2023debate} and scalable oversight \citep{irving2018debate}---under the working assumption that diverse perspectives in debate would yield convergence toward correct reasoning. The proposal has scaled rapidly across reasoning and fact-verification benchmarks; recent work has begun to question whether the convergence assumption holds \citep{zhang2025stop, yang2025revisiting, choi2025debatevote}. This paper offers a falsifiable theoretical reason why it does not, and a constructive remedy that does.

\paragraph{And yet.} The architecture borrows from a human practice---group deliberation---whose pathologies were catalogued long before LLMs existed. Three landmark observations are directly relevant: \citet{janis1972victims}'s eight symptoms of groupthink (illusions of unanimity, self-censorship, direct conformity pressure, mindguards); \citet{asch1951effects}'s finding that $75\%$ of participants conformed to a clearly incorrect majority on at least one trial; and \citet{sunstein2002polarization}'s Law of Group Polarization (deliberation tends to move groups toward a more extreme point in their pre-deliberation direction). The deeper theoretical lineage runs further back: \citet{degroot1974reaching}'s consensus dynamics showed that any iterated averaging among agents who exchange information only with one another converges to a value determined by initial weighted opinions, not by external truth, and \citet{banerjee1992simple}'s informational-cascade model showed that sequential decision-makers who observe only predecessors' actions will rationally ignore their own private signals once a cascade forms. An architectural simulation of human deliberation may inherit, and could plausibly amplify, these pathologies; recent work shows LLM agents reproduce them under standard MAD \citep{yang2025conformity, wang2024llmcollective}. We test the assumption empirically and find that the principal cost of MAD in our experiments is not \emph{accuracy} (which it tends to preserve) but \emph{faithfulness} (which it tends to degrade). Across 16 experimental conditions, the Trap-proper configuration (DebateCV, C13) preserves 88\% of baseline accuracy while our Supported Faithfulness Score (SFS) drops by 43\%; in the extreme, majority-vote debate (C15) reduces evidence grounding to 1.7\% of baseline (Wilcoxon $p < 10^{-6}$, Cohen's $d = -0.96$). We refer to this phenomenon as the \textbf{Debate Trap}---the multi-agent instantiation of a more general \textbf{Reasoning Trap} (Definition~\ref{def:debate_trap}). The DeGroot--Banerjee lineage indicates this is not a contingent failure of any particular MAD implementation: it is what closed-system iterated reasoning under shared parameters does, predicted by half a century of social science before any LLM existed to instantiate it.

\paragraph{The paradox is methodologically generated.} The reason this failure has been invisible is structural rather than empirical. Prevailing studies of MAD (e.g., \citealp{du2023debate, choi2025debatevote, liu2026martingale}, all evaluated against accuracy benchmarks) operationalise ``does debate improve reasoning?'' as ``does debate improve the answer?''---never as ``does debate improve the reasoning \emph{behind} the answer?'' This conflation is not an oversight in any individual paper; it is the field's evaluation paradigm. The result is a research-level analogue of the phenomenon the paper diagnoses: an entire literature whose accuracy-only measurement apparatus is unable to detect the failure mode the architecture induces, even when the failure is large enough to halve evidence-grounding scores. We address it by providing the metric, algorithm, and theorem that the multi-agent setting requires, in a single integrated package rather than as separate contributions distributed across follow-up work. We further validate against a human signal that, \emph{under the rater-calibration practices prevalent in this literature}, proves unstable across context---a finding that inverts a default assumption in faithfulness research. Our metric is therefore better understood as an \emph{epistemological intervention}---a tool that surfaces a hitherto undetected failure mode---than as a calibration toward the prevailing human gold standard, which our evidence shows is not stable enough to bear that role.

\paragraph{Why this matters now.} The strategic context for this paper is the rapid uptake of multi-agent systems in safety-critical evaluation pipelines. \citet{browncohen2023doubly}, \citet{buhl2025alignment}, and \citet{lang2025weak} build scalable-oversight proposals on the premise that debate produces honest reasoning at scale; the convergence-to-truth assumption is load-bearing for these arguments. Theorem~\ref{thm:dpi} provides a counter-premise: under the closed-system conditions of standard MAD, expected faithfulness is non-increasing along the chain, strictly so under any vote-aggregating protocol. The implication is not that scalable oversight cannot work but that the architectural assumptions on which current proposals rest do not survive without external evidence injection. The same structural point applies to evaluation in deployment: faithfulness benchmarks anchored to single-domain inter-rater agreement inherit the human-rating-instability documented by R6, and accuracy benchmarks alone cannot detect the trap. Both ends of the evaluation pipeline---the theoretical assumption and the human calibration target---are load-bearing in ways the prevailing methodology has not made explicit, and Theorem~\ref{thm:dpi} together with the R6 cross-cohort design provides the formal apparatus for stating both load-bearing roles precisely.

\paragraph{Five contributions: a three-part framework, two theorems, and reliability evidence.} (1) The \textbf{Debate Trap}: a formal definition (Definition~\ref{def:debate_trap}) and a three-tier empirical spectrum---\emph{reasoning degradation} (C4: SocraSynth turn-aggregated debate, 39\% SFS drop), \emph{the Trap proper} (C13: DebateCV, 88\% accuracy retention with 43\% SFS drop), and \emph{reasoning elimination} (C15: majority-vote MAD, SFS at 1.7\% of baseline). (2) The \textbf{Supported Faithfulness Score} (SFS): a claim-level metric that first decomposes each response into atomic claims (using an LLM as \emph{decomposer}), then verifies each claim against the provided evidence; condition-level rankings are invariant to choice of decomposer (Spearman $\rho{=}1.0$, R4--R5). (3) \textbf{Evidence-Grounded Socratic Reasoning} (EGSR): a protocol that replaces adversarial argumentation with structured Socratic inquiry against external evidence; recovers SFS to within 2\% of the no-debate baseline (98\% recovery). (4) \textbf{Theorem~1 (DPI Bound)} and \textbf{Theorem~2 (Open-System Recovery)}: under standard MAD assumptions (formalised in Theorem~\ref{thm:dpi}, \S\ref{sec:debate_trap}), agent reasoning forms a Markov chain $E \to O^0 \to O^1 \to \cdots$ along which faithfulness is non-increasing in expectation; any debate variant that preserves this structure is, in expectation, subject to the same DPI bound. Theorem~\ref{thm:recovery} establishes the recovery side: open-system protocols re-injecting $E$ form sub-martingales, and Lemma~\ref{lem:egsr-accum} establishes that EGSR's update structure satisfies this property concretely. Together with the Vote-Aggregation Floor proposition, the four formal results yield a structural classification of multi-step LLM reasoning protocols (the partition is stated formally in \S\ref{sec:debate_trap}). (5) \textbf{R6 human-reliability evidence}: across two language-and-domain cohorts (Korean: $n{=}10$ raters $\times$ 30 FEVER items; English: $n{=}3 \times 200$ SciFact), inter-rater Fleiss $\kappa$ for faithfulness ratings is at most $+0.018$ across both cohorts; the two raters who completed both languages further exhibit intra-rater Likert shifts of 0.8--1.4 points---a direct empirical demonstration that human judgement of reasoning faithfulness, as elicited under the rater-calibration practices prevalent in this literature, is not itself a stable measurement target.

\paragraph{The shape of the contribution.} What is unusual about this paper is not any individual contribution but their integration. The faithfulness-metric literature has produced operational scores without information-theoretic grounding (e.g., \citealp{min2023factscore, wei2024safe}); the multi-agent-debate literature has produced architectural analyses without process-faithful measurement \citep{du2023debate, kenton2024scalable}; the scalable-oversight literature has produced theoretical convergence arguments without operational instantiation \citep{irving2018debate, browncohen2023doubly}; the human-evaluation literature has produced calibration targets without cross-context test-retest \citep{landis1977measurement, akbar2024hallumeasure}. The contribution here is the closure: a metric (SFS) that exercises seven axioms (A1--A7), an algorithm (EGSR) that breaks the Markov chain through external evidence injection, two theorems that bound the closed-system regime above and the open-system regime below, and a human-evaluation cohort design that surfaces the calibration-target instability the field has assumed away. The Knowledge Frontier Map (Figure~\ref{fig:knowledge_frontier}, Appendix~\ref{appendix:frontier_map}) catalogues 132 contributions across eight active lineages and locates the empty region the present work fills---formally diagnosed in \S\ref{sec:related_work} (\emph{The empty quadrant}).

\paragraph{Central thesis.} This paper does not argue that debate never helps accuracy. It shows that closed-system debate can preserve accuracy while systematically eroding evidence-grounded reasoning faithfulness, and that recovery requires breaking the closed information chain through evidence re-injection. The argument is positive: closed-system reasoning is bounded above (Theorem~\ref{thm:dpi}); open-system reasoning that re-injects $E$ is bounded below (Theorem~\ref{thm:recovery}); the latter strictly dominates the former on the metric the multi-agent reasoning literature has not yet measured.

\paragraph{Roadmap.} The remainder positions the work in eight active research lineages (\S\ref{sec:related_work}, Figure~\ref{fig:knowledge_frontier}), defines the Debate Trap and proves Theorem~1 (\S\ref{sec:debate_trap}), specifies SFS and EGSR with the seven design axioms (\S\ref{sec:method}), reports the experimental design and the R1--R8 robustness suite (\S\ref{sec:experiments}), presents results across 16 conditions and two cross-architecture models with the cross-language R6 evidence (\S\ref{sec:results}), discusses implications including the fifty-year DeGroot--Banerjee--Janis--Asch--Sunstein lineage and the Hong-Page diversity-prediction theorem (\S\ref{sec:discussion}), and concludes (\S\ref{sec:conclusion}).

\section{Related Work}
\label{sec:related_work}
\label{sec:rw_synthesis}

\paragraph{An eight-lineage map across 132 contributions.} Multi-agent reasoning faithfulness sits at the intersection of eight active research lineages spanning 2017--2026: (L1) CoT faithfulness in single agents \citep{turpin2023unfaithful, paul2024frodo, shen2025faithcot, jiang2025robustanswers, matton2025walkthetalk, lu2026streaming, mittal2026c2faith}; (L2) MAD critique \citep{du2023debate, zhang2025stop, choi2025debatevote, kenton2024scalable, liu2026martingale, zhu2026demystifying, becker2025drift, pitre2025consensagent, choi2025identity, wu2025canllmsdebate}; (L3) faithfulness/factuality metrics \citep{min2023factscore, wei2024safe, song2024veriscore, akbar2024hallumeasure}; (L4) sycophancy and convergence \citep{sharma2024sycophancy, perez2022sycophancy, chen2024spt, kim2025evaluator, bai2022constitutional}; (L5) accuracy--faithfulness divergence \citep{nguyen2024directeval, jacovi2020faithfulnessmetrics, deyoung2020eraser}; (L6) AI safety via debate \citep{irving2018debate, browncohen2023doubly, lang2025weak, buhl2025alignment, arnesen2024winning}; (L7) Socratic reasoning as alternative \citep{qi2023socratic, he2023socreval, kumar2024socraticqgen, shi2025ssr}; and (L8) metacognition and self-reflection \citep{shinn2023reflexion, guo2025mirror, hubinger2024sleeper, guo2024deception}. We catalogue \textbf{132 contributions} across these lineages plus L0 foundational/dataset references in \textbf{Appendix~\ref{appendix:frontier_map}}: per-lineage paragraph summaries (Appendix~\ref{appendix:eight_lineages}) and the 89-row catalogue table (Table~\ref{tab:frontier_map}, Appendix~\ref{appendix:frontier_table}) annotated by core claim, boundary, and relation to the present work; reviewers seeking the full breadth of prior coverage are referred there. Figure~\ref{fig:5gen_map} (Appendix~\ref{appendix:visual}) further situates the work within five generations of multi-step reasoning research.

\paragraph{The empty quadrant.} The 132-contribution catalogue arranges itself in three of four quadrants defined by system architecture (single agent vs.\ multi-agent debate) and measurement granularity (final outcome vs.\ step-level reasoning process). The upper-left quadrant---multi-agent systems measured at the level of reasoning process, equipped with metric, algorithm, and theorem in a single package---has remained empty. The present work fills it. Sycophancy- and identity-bias studies sit on the boundary between L1 and L2, partially diagnosing the mechanism we quantify; none provides a metric--algorithm--theorem triple. Figure~\ref{fig:knowledge_frontier} (Appendix~\ref{appendix:frontier_map}) maps the eight lineages onto these two axes.

\paragraph{Reliability of the ground truth.} A second axis concerns the human signal against which faithfulness metrics have been calibrated. Existing studies \citep{min2023factscore, wei2024safe, song2024veriscore, akbar2024hallumeasure} report inter-rater agreement on a single domain in a single language; cross-context intra-rater consistency has not been empirically tested. Figure~\ref{fig:reliability_axes} (Appendix~\ref{appendix:visual}) positions the prior art on these axes; our R6 cohort design (\S\ref{sec:results}) tests this directly, finding intra-rater shifts that exceed the typical between-rater variation in either cohort. The lineage gap and the reliability gap together motivate the formalisation that follows: a metric whose rankings do not depend on the human signal we documented as unstable \emph{under prevailing rater-calibration practices}, and a bound that explains why the empty quadrant has remained empty.

\paragraph{Five generations of multi-step reasoning research.} Figure~\ref{fig:5gen_map_inline} situates the present work within five generations of reasoning research evaluation: outcome benchmarking (Gen I, 1995--2020), self-rationalization with explicit Chain-of-Thought (Gen II, 2020--2023), CoT faithfulness as a research target in single agents (Gen III, 2023--2026), MAD descriptive analysis (Gen IV, 2023--2026), and Process-Faithful Multi-Agent reasoning (Gen V, 2026--, the present work). Both single-agent and multi-agent streams converge in Generation V, which integrates metric (SFS), algorithm (EGSR), and theorem (DPI Bound) into a single package---none present in Generations III or IV individually.

\begin{figure}[H]
\centering
\vspace{-0.8em}
\includegraphics[width=0.72\textwidth]{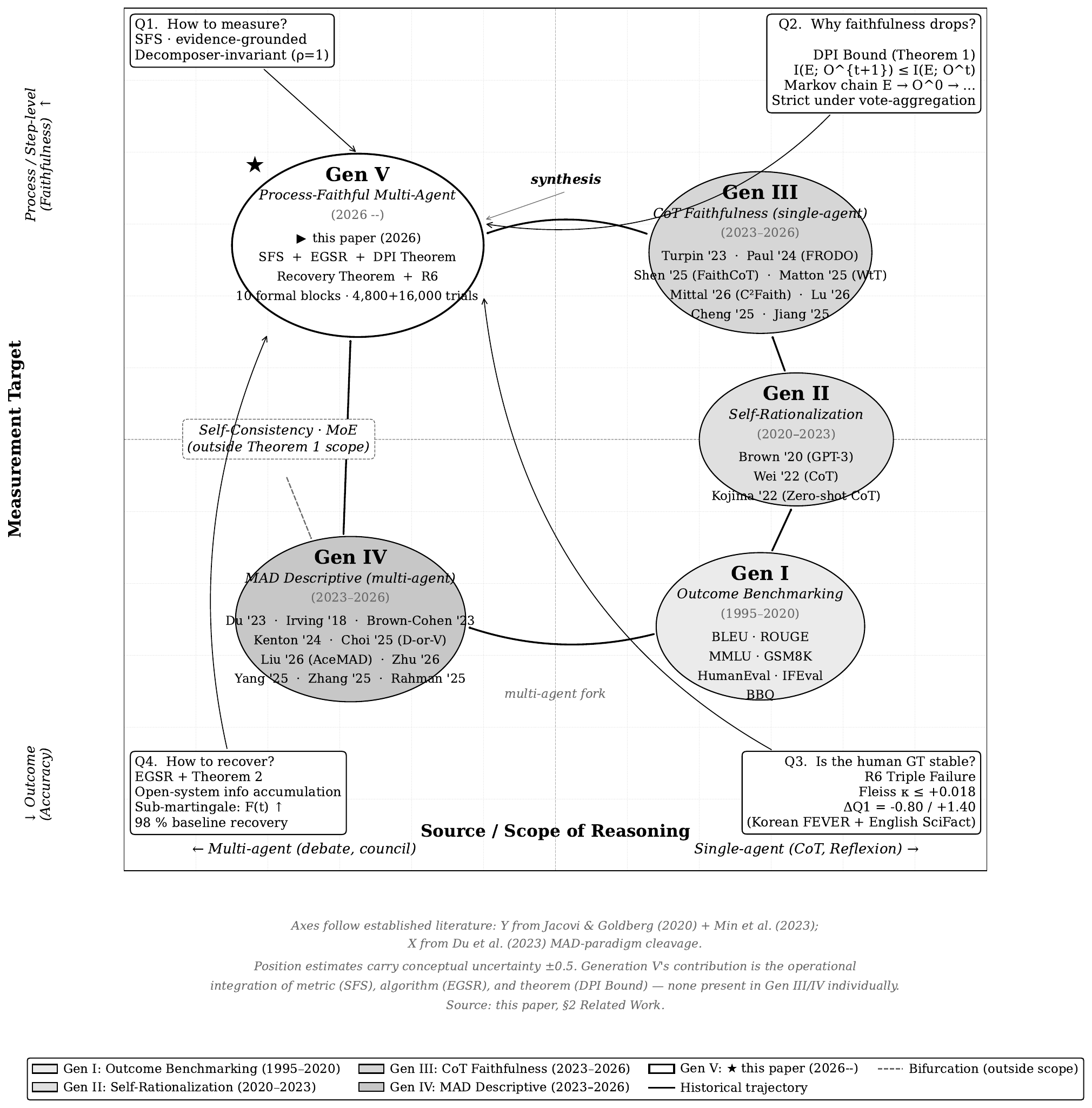}
\vspace{-1.0em}
\caption{\footnotesize\textbf{A Map of Reasoning Faithfulness in the LLM Era.} Five generations of multi-step reasoning research and the emergence of Process-Faithful Multi-Agent reasoning (Generation V). The single-agent stem evolves Gen I (outcome benchmarking, 1995--2020) $\to$ Gen II (self-rationalization with explicit Chain-of-Thought, 2020--2023) $\to$ Gen III (CoT faithfulness as a research target, 2023--2026). The multi-agent fork from Gen I gives rise to Gen IV (MAD descriptive analysis, 2023--2026). Both streams converge into Generation V (the present work, 2026--), which integrates the metric (SFS), algorithm (EGSR), and theorem (DPI Bound, Theorem~\ref{thm:dpi}), together with Theorem~\ref{thm:recovery} (open-system recovery) and the R6 triple-failure evidence on human reliability. The four corner boxes (Q1--Q4) map directly to the paper's four operational contributions (metric, theorem, human-reliability evidence, and recovery algorithm); the integrative \emph{Debate Trap} framing concept (the fifth contribution) is depicted at Generation V itself. Self-Consistency \citep{wang2023selfconsistency} and Mixture-of-Experts violate the closed-system + shared-$\theta$ conditions of Theorem~\ref{thm:dpi} and are shown as a dashed bifurcation outside the bound's scope. Axes follow established literature: the $y$-axis (outcome vs.\ process) draws on \citet{jacovi2020faithfulnessmetrics}'s faithfulness taxonomy and the factuality lineage \citep{min2023factscore}; the $x$-axis (single vs.\ multi-agent) draws on the MAD-paradigm cleavage introduced by \citet{du2023debate}. Position estimates carry conceptual uncertainty $\pm 0.5$; under reasonable axis re-projection the upper-left empty-quadrant pattern persists.}
\label{fig:5gen_map_inline}
\end{figure}

\paragraph{Eight lineages and their missing element.} Table~\ref{tab:lineage_summary} summarises what each lineage measures, what it leaves unmeasured, and how the present work relates. The pattern is consistent: prior work establishes either a metric (L3, L1) or an algorithm (L2, L7) or a critique (L4, L5) or a safety argument (L6) or a self-reflection mechanism (L8), but no lineage to date provides the full \emph{metric--algorithm--theorem} triple at the multi-agent reasoning-process granularity. The closure of the four-condition partition we establish in \S\ref{sec:debate_trap} is what makes the empty quadrant tractable.

\begin{table}[!ht]
\centering\small
\caption{Eight lineages, what each measures, what each leaves missing, and relation to the present work.}
\label{tab:lineage_summary}
\begin{tabular}{p{2.0cm}p{4.2cm}p{4.2cm}p{3.6cm}}
\toprule
Lineage & What it measures & What it leaves missing & Relation to this paper \\
\midrule
L1 (CoT faithfulness, single-agent) & Whether single-agent CoT explains its own answer & No multi-agent extension; no information-theoretic bound & SFS extends to multi-agent; Theorem~\ref{thm:dpi} bounds the chain \\
L2 (MAD critique) & Accuracy convergence in debate & Faithfulness behaviour is not a measurement target & The Debate Trap names this as a structural failure mode \\
L3 (factuality / faithfulness metrics) & Claim-level factual correctness against external evidence & Calibration target stability not tested & R6 cohort study tests this; SFS adds decomposer-invariance \\
L4 (sycophancy / convergence) & Mechanism: agents conform to peer claims & Faithfulness consequence is not formalised & Closed-system Markov chain formalises the consequence \\
L5 (accuracy--faithfulness divergence) & Empirical observation that accuracy can diverge from faithfulness & No mechanism, no bound & Theorem~\ref{thm:dpi} provides the mechanism \\
L6 (AI safety via debate) & Convergence to truth as scalable oversight signal & Premise: convergence implies faithfulness & Counter-premise: faithfulness can decrease in expectation \\
L7 (Socratic reasoning) & Alternative protocol structure & No information-theoretic anchor or external-evidence guarantee & EGSR adds external-evidence anchor; satisfies Theorem~\ref{thm:recovery} \\
L8 (metacognition / self-reflection) & Internal self-correction signals & Self-correction inside closed-system $\to$ no new information & Bounded by Theorem~\ref{thm:dpi}; only external evidence escapes the bound \\
\bottomrule
\end{tabular}
\end{table}

\paragraph{From map to mechanism.} The map identifies \emph{where} the failure mode lives---multi-agent systems measured at the level of reasoning process---and Table~\ref{tab:lineage_summary} identifies \emph{what} each lineage has and is missing. What remains is the \emph{why}: a structural account of the reasoning failure that the eight lineages collectively diagnose but none formally bounds. The next section formalises that account. We name the broader phenomenon the \emph{Reasoning Trap} (any closed-system iterated reasoning chain) and its multi-agent instantiation the \emph{Debate Trap}, and we prove an information-theoretic bound that explains why the upper-left quadrant has been empty: under the four-condition Markov closure of \S\ref{sec:debate_trap}, faithfulness is bounded above in expectation, with strict inequality under any vote-aggregating protocol.

\section{The Reasoning Trap: Definitions and Theorem}
\label{sec:debate_trap}

We formalise the phenomenon. Let $E = \{e_1, \ldots, e_m\}$ denote the evidence set provided for a claim $c$, and let $O_i^t$ denote agent $i$'s reasoning output at debate round $t$ (with $O^0$ the initial response).

\begin{definition}[Reasoning Faithfulness]
\label{def:faithfulness}
The faithfulness of agent $i$ at round $t$ is $F_i(t) = I(E; O_i^t)$, where $I$ denotes mutual information. Aggregate faithfulness is $F(t) = \frac{1}{n}\sum_i F_i(t)$. Under symmetric agents (identical $\theta$), the joint-MI faithfulness $F^{\mathrm{joint}}(t) := I(E; O^t)$ with $O^t := (O_1^t, \ldots, O_n^t)$ is the natural quantity to which the Data Processing Inequality applies (Theorem~\ref{thm:dpi}); we use $F^{\mathrm{joint}}$ as the load-bearing object (Lemma~\ref{lem:joint-avg}, Appendix~\ref{appendix:proofs}). The per-agent average $F(t)$ defined above remains a complementary symmetry-respecting summary.
\end{definition}

\begin{definition}[Evidence Utilization Rate and the Debate Trap]
\label{def:debate_trap}
$\mathrm{EUR}(t) = |E_{\text{cited}}^t| / |E|$ is the fraction of provided evidence cited at round $t$. Let $\mathrm{Acc}(t) := \Pr[\hat{y}(O^t) = y]$ denote answer accuracy. A multi-agent system exhibits the \emph{Debate Trap} if there exist rounds $t_1 < t_2$ such that $\mathrm{Acc}(t_2) \geq \mathrm{Acc}(t_1) - \epsilon$ \emph{and} $F(t_2) < F(t_1) - \delta$ for some $\epsilon, \delta > 0$.
\end{definition}

\begin{theorem}[Faithfulness Bound under Standard MAD (DPI Bound)]
\label{thm:dpi}
Under the standard MAD protocol where agent $i$'s output at round $t+1$ depends only on $O^t$ and shared parameters $\theta$, the reasoning process forms a Markov chain $E \to O^0 \to O^1 \to \cdots \to O^T$. By the Data Processing Inequality, $\mathbb{E}\!\left[I(E; O^{t+1})\right] \leq \mathbb{E}\!\left[I(E; O^t)\right]$: joint-MI faithfulness is non-increasing in expectation along a debate trajectory. The inequality is \emph{strict} whenever the round-$t{+}1$ aggregation is non-injective in $O^t$ (e.g., majority voting, summary aggregation, or any vote-aggregating protocol).
\end{theorem}

\noindent The proof follows from DPI applied to the Markov structure (Appendix~\ref{appendix:proofs}; Markov-chain visualization Figure~\ref{fig:dpi_markov}, Appendix~\ref{appendix:visual}). When agents share $\theta$ and the system is informationally closed, debate redistributes beliefs but does not inject new information about $E$. Read through the information bottleneck principle, closed-system multi-agent reasoning compresses representations through a bottleneck at each aggregation step. Three concurrent theoretical lines---\citet{choi2025debatevote}'s belief martingale, \citet{liu2026martingale}'s AceMAD generalization, \citet{zhu2026demystifying}'s confidence-modulated debate---establish the parallel result for accuracy; Theorem~\ref{thm:dpi} extends a complementary framework to faithfulness.

\paragraph{Generalization to closed-system reasoning chains.} The Markov property holds whenever a protocol satisfies four conditions: (i)~shared $\theta$ across steps, (ii)~$E$ provided once at $t{=}0$ without re-injection, (iii)~step $t{+}1$ depends only on step $t$ and $\theta$, (iv)~symmetric aggregation. This suggests a \emph{conditional extension} beyond MAD rather than an empirical claim across paradigms: closed-loop readings of single-agent Chain-of-Thought, Reflexion-style self-critique \citep{shinn2023reflexion}, and linear Tree-of-Thought may instantiate the same structural risk when they repeatedly condition on prior reasoning traces without re-injecting source evidence. Consistent with this conditional reading, single-agent CoT modeled as a token-Markov process \citep{hao2025markovcot} exhibits empirical unfaithfulness rates of $0.04\%$--$13\%$ across seven frontier models \citep{arcuschin2025wild,young2026liecot}; the present paper's empirical core, however, is MAD, and the broader claim is a falsifiable research program rather than a tested benchmark. Self-Consistency \citep{wang2023selfconsistency} (independent sampling) and Mixture-of-Experts (routed experts maintain distinct $\theta$) are explicitly excluded (full four-condition verification across these five-plus-two paradigms in Appendix~\ref{appendix:5paradigm_matrix}, Figure~\ref{fig:d1_5paradigm}). Figure~\ref{fig:5paradigm_inline} summarises the four-condition verification across the five-plus-two paradigms in inline form for reader convenience.

\begin{figure}[H]
\centering
\includegraphics[width=0.92\textwidth]{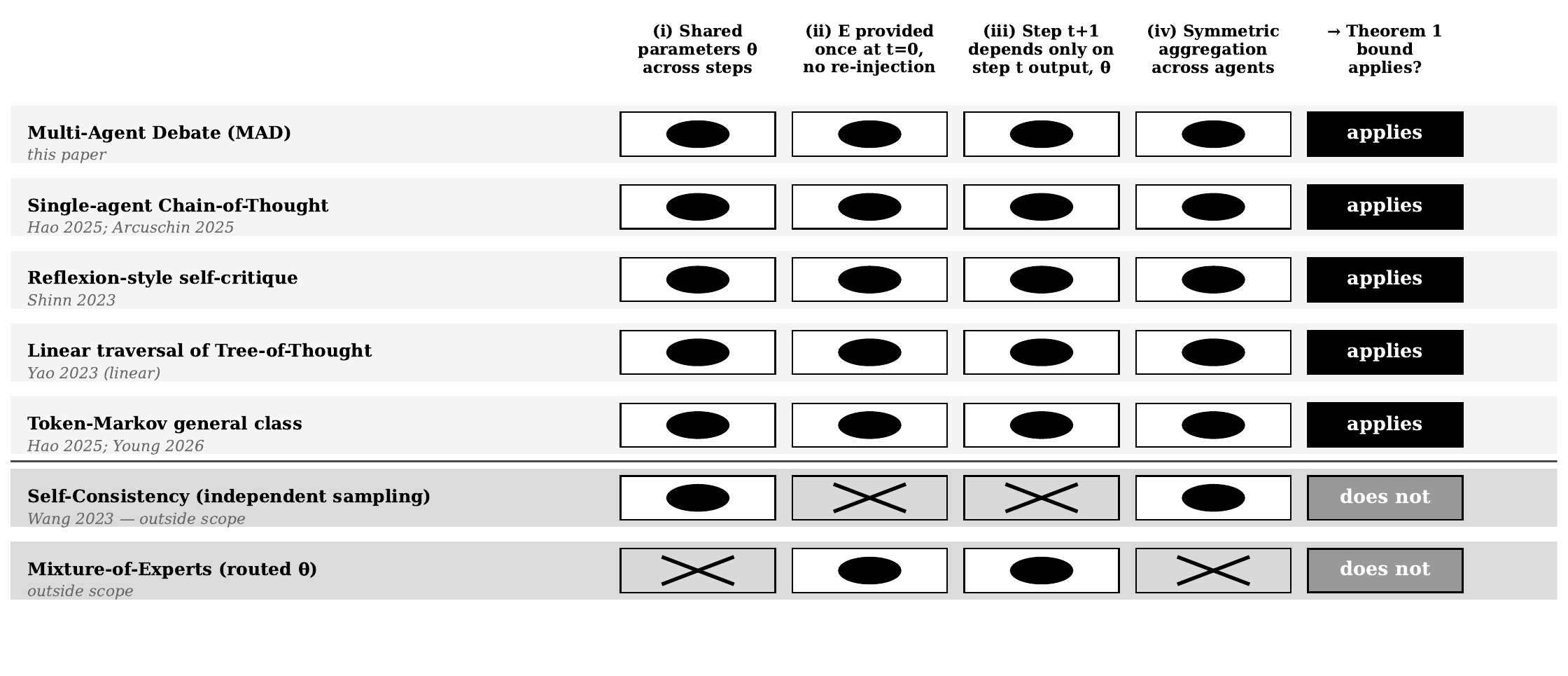}
\caption{\textbf{Theorem~\ref{thm:dpi}'s scope across seven reasoning paradigms.} Each row is a paradigm; columns are the four conditions of Theorem~\ref{thm:dpi}. Five paradigms (multi-agent debate, single-agent CoT under token-Markov reading, Reflexion-style self-critique, linear traversals of Tree-of-Thought, the broader token-Markov class) satisfy all four conditions and inherit the DPI bound. Two paradigms (Self-Consistency, Mixture-of-Experts) violate condition (iii) (independent sampling) or (i) (distinct $\theta$) and fall outside the scope.}
\label{fig:5paradigm_inline}
\end{figure}

\paragraph{Two practical extensions and a long-context prediction.} Two practical extensions of the bound deserve explicit treatment. \emph{Long-context reasoning systems} (e.g., OpenAI's o1, o3 and similar test-time-compute models that generate extended internal CoT before answering) tighten condition (ii) as internal CoT lengthens without re-injecting external evidence. Each additional internal reasoning step is a further application of the data processing inequality on the closed chain $E \to O^0 \to O^1 \to \cdots \to O^T$; longer chains \emph{sharpen} rather than escape the bound, because the only information about $E$ available at step $T$ is what survived the cascade of internal transformations. The long-context-reasoning trend therefore amplifies, rather than mitigates, the Reasoning Trap: more internal computation is more compression, not more grounding. \emph{Heterogeneous-agent settings} (e.g., GPT-4 with Claude in the same debate) partially relax condition (i): the chain is no longer governed by a single shared $\theta$ but by a sequence of distinct $\theta_1, \theta_2, \ldots, \theta_T$. The mutual information bound weakens but does not vanish: as long as the system remains informationally closed (no external evidence re-injection), the joint MI cannot exceed $I(E; O^0)$ regardless of how many distinct $\theta$ are interleaved. Heterogeneity changes the speed of erosion; it does not change the direction. Theorem~\ref{thm:dpi}'s scope is \emph{closed-system} MAD; \emph{evidence-reinjected} MAD variants whose prompts re-include $E$ at every round violate condition~(ii) and fall outside the bound (Corollary~\ref{cor:egsr-breaks}). EGSR is the constructive instance of the evidence-reinjected regime (\S\ref{sec:egsr}). The structure can be read as an LLM formalisation within \citet{degroot1974reaching}'s consensus and \citet{banerjee1992simple,bikhchandani1992theory}'s sequential-cascade lineage. We use \emph{Reasoning Trap} for the broader failure mode, \emph{Debate Trap} for its multi-agent instantiation.

\paragraph{Companion results.} Three companion formal results (Appendix~\ref{appendix:proofs}) complete the partition of multi-step LLM reasoning protocols. \emph{(i) Faithfulness Recovery under Open-System Information Accumulation} (Theorem~\ref{thm:recovery}): if at each round the protocol produces $O^{t+1} = (O^t, A^{t+1})$ with $A^{t+1}$ generated with direct access to $E$, then $I(E; O^{t+1}) \geq I(E; O^t)$ and $F(t)$ is a sub-martingale. \emph{(ii) EGSR information accumulation} (Lemma~\ref{lem:egsr-accum}): EGSR's running-aggregate verdict update with externally retrieved evidence satisfies the conditions of Theorem~\ref{thm:recovery}, making $F(t)$ a sub-martingale for the protocol; the same update structure breaks the Markov closure of Theorem~\ref{thm:dpi} by re-accessing $E$ each round (Corollary~\ref{cor:egsr-breaks}), so $I(E; O^{t+1})$ is no longer upper-bounded by $I(E; O^t)$. \emph{(iii) Vote-Aggregation Floor} (Proposition~\ref{prop:vote-floor}): under any protocol reducing $O^T$ to a vote tally $v$ in space $V_K$, $\mathrm{SFS}(O^T) \to 0$ since $I(E; O^T) \leq \log_2 K$ regardless of $H(E)$, explaining C15's $0.006$ floor (1.7\% of baseline) in \S\ref{sec:results}. Theorem~\ref{thm:dpi}, Theorem~\ref{thm:recovery}, Lemma~\ref{lem:egsr-accum}, and Proposition~\ref{prop:vote-floor} together partition the space of multi-step LLM reasoning protocols by their information-theoretic relationship to $E$.

\paragraph{A unified information-capacity picture.} The four formal results admit a unified reading in the language of channel capacity \citep{cover2006elements}. Closed-system MAD (Theorem~\ref{thm:dpi}) imposes the bound $I(E; O^T) \leq I(E; O^0) \leq H(E)$ along the chain, with strict inequality at every non-injective aggregation step. Vote-aggregating MAD (Proposition~\ref{prop:vote-floor}) tightens the bound further to $I(E; O^T) \leq \log_2 K$, independent of $H(E)$---a structural floor at the information capacity of the $K$-way verdict alphabet. Open-system EGSR (Theorem~\ref{thm:recovery}) replaces the closed channel with a sub-martingale that, by Doob's almost-sure martingale convergence theorem \citep{doob1953stochastic}, converges to a limit $F_\infty \leq H(E)$ on every realization of the trajectory. The three regimes therefore correspond to three distinct asymptotic capacities: \emph{closed-system} (bounded above by $H(E)$, monotonically erodes), \emph{vote-aggregating} (collapses to $\log_2 K$ regardless of $H(E)$), and \emph{open-system EGSR} (climbs towards $H(E)$ along a sub-martingale). The Information Bottleneck principle of \citet{tishby2000ib} is the natural lens: closed-system MAD instantiates the IB compression objective without the prediction-side relaxation that allows IB to balance representation cost against task utility, and EGSR's external evidence injection re-supplies the prediction-side information IB takes for granted.

\paragraph{Estimation guarantee.} The SFS estimator $\widehat{\mathrm{SFS}}(O) = \tfrac{1}{N}\sum_i s_i$ over $N$ atomic claims with $s_i \in [0,1]$ concentrates at exponential rate. By Hoeffding's inequality \citep{hoeffding1963probability}, for any $\epsilon > 0$,
\begin{equation}
\Pr\!\left[\,|\widehat{\mathrm{SFS}}(O) - \mathbb{E}[\mathrm{SFS}(O)]| > \epsilon\,\right] \leq 2\exp(-2N\epsilon^2).
\end{equation}
For the design used throughout (paired comparisons across $N{=}300$ SciFact trials per condition), choosing $\epsilon = 0.10$ gives a tail bound below $5 \times 10^{-3}$; choosing $\epsilon = 0.05$ gives a tail bound near $0.5$. The $N{=}300$ design therefore supplies tight finite-sample guarantees for condition-level SFS gaps of $\geq 0.10$ (which include all three Trap tiers: C4 vs.\ C1 $|\Delta\mathrm{SFS}| = 0.135$; C13 vs.\ C1 $|\Delta\mathrm{SFS}| = 0.149$; C15 vs.\ C1 $|\Delta\mathrm{SFS}| = 0.343$).

\paragraph{Quantitative gap of Theorem~\ref{thm:dpi}.} The strict-inequality clause of Theorem~\ref{thm:dpi} admits a quantitative formula via the KL-divergence representation of conditional information loss \citep{cover2006elements}. Writing $f$ for the round-$t{+}1$ aggregation step, the per-step gap is bounded below by the divergence between the conditional distribution of $O^t$ given $(E, f(O^t))$ and given $f(O^t)$ alone:
\begin{equation}
I(E; O^t) - I(E; O^{t+1}) \;\geq\; \mathbb{E}_{E, f(O^t)}\!\left[\,D_{\mathrm{KL}}\!\left(P(O^t \mid E, f(O^t)) \,\Vert\, P(O^t \mid f(O^t))\right)\,\right].
\end{equation}
The bound vanishes only when $f$ is injective in $O^t$ given $E$ (the Markov chain becomes degenerate, equality in DPI). For majority-vote aggregation $f: \{0,1\}^n \to \{0,1\}$, the conditional divergence is strictly positive whenever the round-$t$ outputs $O^t$ are not deterministic given the vote---that is, on almost any non-trivial debate input. The Vote-Aggregation Floor (Proposition~\ref{prop:vote-floor}) is the asymptotic limit of this per-step gap accumulated over rounds.

\paragraph{Optional-stopping interpretation of EGSR's gate.} Algorithm~\ref{alg:egsr} terminates EGSR at the round when the Checker's gate threshold $\tau$ is met. By Doob's optional stopping theorem \citep{doob1953stochastic} applied to the sub-martingale $F(t)$ of Theorem~\ref{thm:recovery}, for any almost-surely-finite stopping time $T_\tau$ adapted to the natural filtration of the EGSR trajectory,
\begin{equation}
\mathbb{E}\!\left[F(T_\tau)\right] \;\geq\; \mathbb{E}\!\left[F(0)\right].
\end{equation}
That is, terminating EGSR at any threshold-based stopping rule preserves the expected faithfulness gain established by Theorem~\ref{thm:recovery}: an early stop does not lose the recovery, and a late stop cannot decrease it in expectation. The deployment-time gate threshold of Algorithm~\ref{alg:egsr} is therefore not merely a heuristic but a stopping rule whose expected output is bounded below by initial faithfulness.

\paragraph{Architectural-limit parallel to compositionality results.} Theorem~\ref{thm:dpi} stands in formal parallel to \citet{dziri2023compositionality}'s compositionality limit (their Proposition 4.1): where Dziri et~al.\ establish that the probability of correct prediction on compositional tasks converges exponentially to chance under depth scaling, Theorem~\ref{thm:dpi} establishes that joint mutual information about $E$ is non-increasing in expectation under round scaling. Both results identify \emph{architectural} rather than \emph{contingent} limits---they apply to any model satisfying the closure conditions, not to specific training recipes or model sizes---and both situate the failure mode in the structure of the inference chain rather than in the model's parameters. The two results are jointly constraining: a closed-system multi-agent system inherits both the per-step compositional fragility (Dziri) and the per-round mutual-information erosion (Theorem~\ref{thm:dpi}).

\paragraph{$f$-Divergence generalisation of Theorem~\ref{thm:dpi}.} The DPI bound is not a property of mutual information specifically; by Csisz{\'a}r's $f$-divergence DPI \citep{csiszar1967fdiv}, it holds for every $f$-divergence with convex generator. Recall that for a convex $f$ with $f(1)=0$, the $f$-divergence is $D_f(P\Vert Q) = \mathbb{E}_Q[f(\mathrm{d}P/\mathrm{d}Q)]$, and that mutual information is the $f$-divergence with $f(t)=t\log t$ between the joint distribution and the product of marginals: $I(E;O^t) = D_{\mathrm{KL}}(P_{E,O^t}\Vert P_E\!\otimes\!P_{O^t})$. Csisz{\'a}r's theorem states $D_f(P_X\Vert Q_X) \geq D_f(P_{f(X)}\Vert Q_{f(X)})$ for any (possibly stochastic) channel applied to $X$. Applied to the closed-system Markov chain $E \to O^0 \to O^1 \to \cdots$ of Theorem~\ref{thm:dpi}, the same chain inequality holds for every $f$-divergence:
\begin{equation}
D_f\!\left(P_{E,O^{t+1}} \,\big\Vert\, P_E\!\otimes\!P_{O^{t+1}}\right) \;\leq\; D_f\!\left(P_{E,O^t} \,\big\Vert\, P_E\!\otimes\!P_{O^t}\right).
\end{equation}
Theorem~\ref{thm:dpi} is therefore one member of a family: any faithfulness measure expressed as an $f$-divergence between the joint $(E,O^t)$ distribution and its independent product---R{\'e}nyi divergence ($f(t) = (t^\alpha-1)/(\alpha-1)$), $\chi^2$-divergence ($f(t) = (t-1)^2$), total-variation distance ($f(t) = |t-1|/2$), Hellinger distance ($f(t) = (\sqrt{t}-1)^2$)---inherits the same monotone-non-increase under closed-system iteration. Practitioners who prefer total-variation or $\chi^2$ formulations of faithfulness for computational reasons are not exempt from the bound: the same architectural structure produces the same erosion in every member of the $f$-divergence family.

\section{Method: SFS and EGSR}
\label{sec:method}

\subsection{Supported Faithfulness Score (SFS)}
\label{sec:sfs}

SFS measures how well an agent's reasoning is grounded in the evidence it was given. The construction is deliberately conservative: a claim earns SFS mass only when it is \emph{both} semantically aligned with and logically entailed by some passage in the externally provided evidence set $E$. The score is computed in three stages.

\textbf{Stage 1: Atomic decomposition.} Given reasoning text $O$, a decomposer $\phi$ extracts atomic claims $\{c_1, \ldots, c_N\}$, where each $c_i$ is an independently verifiable assertion. ``Independently verifiable'' is operationalised in two ways: (a) the claim contains a single subject--predicate relation, and (b) its truth can in principle be evaluated without reference to other claims in $O$. Compound claims (``$X$ because $Y$'') are split into their components; claims about reasoning style or methodology, which cannot be checked against $E$, are excluded. We instantiate $\phi$ with GPT-4o using structured JSON output; the cross-decomposer robustness of the resulting rankings is established empirically by R4--R5 (Spearman $\rho = 1.0$ across decomposers, even though atomic-claim sets achieve only moderate soft-Jaccard overlap of 0.49). The condition-level robustness without claim-level identity is itself an A2 axiom result: SFS rankings depend on the evidence-grounding mass of $O$, not on the particular decomposition used.

\textbf{Stage 2: Evidence verification.} Each claim $c_i$ is compared against every passage $e \in E$ via two complementary signals:
\begin{equation}
s_i = \max_{e \in E}\; \mathrm{sim}(c_i, e) \cdot \mathrm{verified}(c_i, e)
\end{equation}
where $\mathrm{sim}(\cdot,\cdot)$ is sentence-BERT cosine similarity using the all-MiniLM-L6-v2 model \citep{reimers2019sbert} and $\mathrm{verified}(c_i, e) \in \{0,1\}$ is a DeBERTa-v3-large NLI entailment indicator \citep{he2021deberta}. The \emph{product} structure is critical: it requires both signals to be active. A claim that is semantically similar to evidence but not entailed (e.g., a paraphrase that flips the polarity) scores zero on the verification gate; a claim that is entailed by evidence the model never had access to scores zero on the similarity gate. The maximisation is over passages, so a claim earns its highest possible support-mass from the single best-matching evidence span---we do not require corroboration from multiple passages, since over-reliance on consensus would penalise legitimate single-source claims.

\textbf{Stage 3: Aggregation.} The final score is the uniform mean across atomic claims:
\begin{equation}
\mathrm{SFS}(O) = \frac{1}{N}\sum_{i=1}^N s_i \in [0,1].
\end{equation}
SFS extends naturally to round-level faithfulness, $F(t) = \mathrm{SFS}(O^t)$, by computing SFS on the reasoning trace produced after debate round $t$. Round-level SFS is what Theorem~\ref{thm:dpi} bounds along the Markov chain $E \to O^0 \to O^1 \to \cdots$. Uniform weighting (rather than length- or confidence-weighted aggregation) is a deliberate choice consistent with the A1 (claim-level granularity) and A4 (support-mass monotonicity) axioms: longer reasoning that contains more grounded claims should not automatically receive higher SFS than tightly worded reasoning of equivalent grounding density.

\textbf{Seven design properties (A1--A7).} SFS is constructed to satisfy seven axioms, formally stated in Appendix~\ref{appendix:axioms} and exercised by R1--R8: (A1) claim-level granularity (the score is a function of atomic claims, not surface text); (A2) decomposer-invariance of rankings (robust to choice of $\phi$); (A3) evidence-set sensitivity ($\mathrm{SFS}(O; E) \ne \mathrm{SFS}(O; E')$ when $E \ne E'$, and the difference reflects which claims become groundable); (A4) support-mass monotonicity (adding entailed claims weakly increases SFS); (A5) fabrication penalty (claims unentailed by any $e \in E$ contribute zero); (A6) reproducibility (deterministic given $\phi$, the encoder, and the NLI model); and (A7) same-verdict-different-SFS validity (two reasoning traces reaching the same final answer can---and empirically do---receive substantially different SFS, with $\Delta > 0.05$ on 70\% of paired examples). A7 is the falsifiability anchor: SFS is genuinely \emph{process-level} measurement, not a re-skin of accuracy.

\subsection{Evidence-Grounded Socratic Reasoning (EGSR)}
\label{sec:egsr}

\textbf{Four design requirements.} EGSR's architecture follows from four convergent requirements emerging from prior research and from Theorem~\ref{thm:dpi} itself: (i) breaking the closed-system Markov chain requires injecting an \emph{external signal} that is not derivable from prior reasoning trace---this is what the DPI bound forbids in any closed-system protocol, and the requirement for extra-textual grounding articulated by \citet{bisk2020experience} ($\!$``modeling lexical co-occurrence, no matter the scale, is still modeling the written world'')---the same closure that motivates Theorem~\ref{thm:dpi} also makes external evidence the only available escape; (ii) resisting the scalability paradox \citep{shen2025faithcot}, in which more reasoning produces less grounded reasoning, requires \emph{claim-level decomposition with external verification} (otherwise the model is free to elaborate plausibly without grounding); (iii) deactivating sycophancy and identity-bias amplification \citep{kim2025evaluator,yao2025peacemaker,choi2025identity} requires \emph{evaluative framing rather than persuasive framing}---debate prompts agents to win, while EGSR prompts them to verify; and (iv) Socratic questioning, which prior work has shown to be a productive reasoning structure \citep{qi2023socratic}, becomes faithfulness-preserving \emph{only when its questions are grounded in evidence external to the model}, otherwise it instantiates the same closed-system loop in dialogue form.

\textbf{Three-role architecture (Hypothesis-Free Socratic Verification, HFSV).} EGSR replaces adversarial argumentation among $K$ symmetric agents with an asymmetric three-role inquiry grounded in external evidence. The \textbf{Debater} produces initial reasoning $O^0$ given claim $c$ and evidence $E$. The \textbf{Questioner} examines $O^0$ and generates evidence-grounded sub-questions targeting claims that appear underspecified, weakly grounded, or potentially unsupported---each sub-question must reference at least one passage in $E$ that is relevant to the candidate weakness. The \textbf{Checker} attempts to answer each sub-question by re-consulting $E$, then \emph{gates progression}: if the sub-answer reveals an unsupported claim, the trace is updated to remove or qualify the claim; if the sub-answer confirms grounding, the claim is retained. Critically, $O^{t+1}$ is built by composing $O^t$ with the verified Checker outputs, not by re-asking the Debater for a new round of free-form argument. This composition step is the formal mechanism through which EGSR's $F(t)$ becomes a sub-martingale---each round either preserves or strictly increases evidence-grounded mass.

\textbf{Why this breaks Theorem~\ref{thm:dpi}.} EGSR's verification anchor is \emph{external to all agents}: the Checker re-references $E$ at each round rather than conditioning only on prior agent outputs. This violates condition (ii) of Theorem~\ref{thm:dpi} (``$E$ provided once at $t=0$, no re-injection''), and the chain $E \to O^0 \to O^1 \to \cdots$ is no longer Markov in the sense the DPI bound requires. By Lemma~\ref{lem:egsr-accum}, the running-aggregate update structure satisfies the information-accumulating property of Theorem~\ref{thm:recovery}, making $F(t)$ a sub-martingale for the protocol---this is the formal route to the empirical recovery shown in \S\ref{sec:results} (98\% of baseline, with C8 EGSR variants clustering at SFS $\in [0.32, 0.34]$ across architectural variants C9, C12, C14). Unlike Reflexion \citep{shinn2023reflexion} or self-consistency \citep{wang2023selfconsistency}, EGSR's external anchor is a structural property of the protocol, not an emergent property of sampling many trajectories: the recovery does not depend on $K \to \infty$ averaging but on each individual round respecting the external-injection condition.

Algorithm~\ref{alg:egsr} (Appendix~\ref{appendix:algorithm}) provides the full pseudocode, including termination conditions and the gate threshold $\tau$ governing when the Checker accepts a sub-answer as confirming grounding. An architecture diagram contrasting EGSR's open-system structure with standard MAD's closed-system loop is Figure~\ref{fig:egsr_arch} (Appendix~\ref{appendix:algorithm}); the closed-vs-open comparison is Figure~\ref{fig:closed_vs_open} (Appendix~\ref{appendix:algorithm}). The EGSR family variants tested in \S\ref{sec:results}---Parallel EGSR (C9, parallel sub-question generation rather than sequential), EGSR+SLM (C12, small language model substituted for the Checker role to reduce cost), EGSR+Adversarial (C14, adversarial evidence injected into $E$ to test robustness)---collectively probe the sensitivity of the recovery to architectural choices and confirm that the open-system / external-injection property, not any specific implementation choice, is what drives faithfulness recovery.

\section{Experiments}
\label{sec:experiments}

\paragraph{Datasets.} The primary analysis is on SciFact \citep{wadden2020scifact}, a fact-verification benchmark of scientific claims paired with abstract-level evidence drawn from biomedical literature; we use the standard 300-claim evaluation split. Cross-domain replication uses FEVER \citep{thorne2018fever}, a Wikipedia-based fact-verification benchmark; we use a 1{,}000-claim sample (sampling protocol and retained sample IDs are released with the artefact bundle). Both datasets are public, used in their original form within the licenses they distribute under, and were chosen for their differing surface characteristics: SciFact provides genuinely difficult specialised content where parametric memory is least likely to substitute for evidence-grounded reasoning, while FEVER provides a higher-resource Wikipedia setting where the Debate Trap might plausibly disappear if it were a SciFact-specific artefact. The replication confirms it is not.

\paragraph{Sixteen conditions.} The 16 conditions span five families designed to disentangle the architectural choices that drive faithfulness behaviour. \emph{Baselines}: zero-shot (C1), retrieval-augmented generation (C2, RAG), and the Socratic chain-of-thought baseline (C3, SocCoT). \emph{Debate variants}: SocraSynth (C4), Debate+RAG (C6), Vanilla MAD (C15), DebateCV (C13, the Trap-proper conformity-vote variant), and unilateral debate (C16, UDPO). \emph{EGSR variants}: core EGSR (C8), Parallel EGSR (C9), EGSR+SLM (C12), and EGSR+Adversarial (C14). \emph{Competitors and ablations}: Internal Socratic (C10, the closed-system Socratic counterpart that isolates the external-injection mechanism), SSR (C11), and Pipeline-noSLM (C7). The total is $300 \times 16 = 4{,}800$ SciFact trials and $1{,}000 \times 16 = 16{,}000$ FEVER trials. Self-Consistency \citep{wang2023selfconsistency} is excluded on theoretical grounds---it is not a Markov chain in the sense Theorem~\ref{thm:dpi} requires (independent sampling rather than sequential conditioning) and the bound therefore does not apply to it; including it would obscure the architectural variable rather than illuminate it (formal four-condition verification: Appendix~\ref{appendix:5paradigm_matrix}, Figure~\ref{fig:d1_5paradigm}). Per-condition prompts, agent counts, round limits, and aggregation rules are reproduced in Appendix~\ref{appendix:conditions}.

\paragraph{Models and architectures.} GPT-4o is the primary model for all 16 conditions on both datasets, providing the within-model architectural comparison. Claude-3.5-Sonnet is used for cross-architecture replication on the Trap-defining conditions (C1, C4, C13, C15, C8) to test whether the phenomenon is specific to GPT-4o or generalises across model families with different training distributions. The LLM-as-auditor probe (Appendix~\ref{appendix:theorem2}) additionally evaluates six frontier models on a 15-item stratified sample of the R6 cohort to test whether automated auditors detect what human auditors did not. Within each condition, model temperature, prompt structure, and round count are fixed; varying these is the explicit purpose of the R1--R8 robustness suite below.

\paragraph{Hypotheses (pre-registered) and statistical control.} The pre-registered hypotheses partition into three groups. The \emph{primary EGSR-debate family} is H1 (debate degrades faithfulness, $\mathrm{C4} < \mathrm{C1}$ on SFS), H2 (evidence utilisation rate drops in debate), H4 (external evidence outperforms internal Socratic, the pre-registered counterfactual $\mathrm{C8} > \mathrm{C10}$ that isolates the external-injection mechanism from any general benefit of Socratic prompting), and H5--H9 (EGSR variant robustness across architectures). \emph{Reported separately}: H10 (small-language-model accuracy parity, descriptive) and H3 (SFS--human Spearman correlation $r > 0.70$, originally pre-registered but rendered inconclusive by the R6 triple failure of human reliability and reported descriptively rather than as a confirmation test). \emph{Family-wise error rate} is controlled by Holm--Bonferroni at $\alpha = 0.05$ over the primary EGSR-debate family (H1, H2, H4--H9), giving an adjusted threshold of $\alpha^* = 0.0100$. Effect sizes are reported as Cohen's $d$ with bootstrap 95\% confidence intervals (5{,}000 resamples). For paired-trial comparisons, we use the Wilcoxon signed-rank test rather than the paired $t$-test, since SFS distributions are heavily skewed. Inter-rater agreement uses Fleiss $\kappa$ for Likert-5 ratings, Cohen's $\kappa$ for binary flags, and quadratic-weighted $\kappa$ for the ordinal interpretation under \citet{landis1977measurement}'s thresholds. For $n{=}300$ paired SciFact trials, the design provides power $> 0.99$ to detect $|d| \geq 0.3$ at $\alpha = 0.05$---a regime in which the observed C4--C1 effect ($d = -1.12$) is several multiples of the minimum detectable.

\paragraph{R1--R5: SFS construct robustness.} R1 varies the similarity function (word-overlap Jaccard, sentence-BERT cosine, NLI-only) to test whether condition-level rankings depend on a particular similarity primitive. R2 varies output verbosity (terse, standard, verbose) to test whether longer reasoning artificially inflates SFS. R3 varies the NLI gate threshold $\tau \in \{0.3, 0.5, 0.7, 0.9\}$ to test whether the gating decision is brittle. R4 varies the decomposer (GPT-4o vs.\ Claude-3.5-Sonnet) to test whether atomic-claim extraction style determines the score. R5 measures cross-decomposer agreement (mean soft-Jaccard over 150 reasoning texts) and condition-level rank correlation (Spearman $\rho$). Together R1--R5 establish that SFS \emph{rankings} are decomposer- and similarity-invariant ($\rho = 1.0$) even when atomic-claim sets achieve only moderate overlap (0.49 soft-Jaccard)---SFS measures grounding mass, not surface similarity.

\paragraph{R6: cross-language test-retest of human raters.} R6 is the human-evaluation backbone of the paper and is designed deliberately to surface a failure mode prior work has not measured. \emph{Korean cohort}: ten independent raters each evaluated 30 FEVER items on a Likert-5 faithfulness scale ($Q_1$) and a binary unsupported-claim flag ($Q_2$) under the rubric calibrated against \citeauthor{landis1977measurement}'s interpretation table. \emph{English cohort}: three independent raters each evaluated 200 SciFact items under the same rubric. \emph{Cross-cohort design}: two raters voluntarily completed both cohorts, providing the within-rater test-retest data on which the triple-failure analysis rests. The cross-language design is the methodological innovation: it tests not only between-rater agreement (the standard target in the L3 lineage of \citealp{min2023factscore, wei2024safe, song2024veriscore, akbar2024hallumeasure}) but also intra-rater stability across language and domain---a cross-context test-retest that, to our knowledge, no prior faithfulness-metric calibration has reported. The cohort design, rubric, per-rater $Q_1$ means, $Q_2$ Y-rates, and full pairwise Cohen's $\kappa$ matrix are reproduced in Appendix~\ref{appendix:r6} (and the heat map in Appendix~\ref{appendix:kappa_matrix}, Figure~\ref{fig:kappa_heatmap}).

\paragraph{R7--R8: axiom and metric comparison.} R7 tests axiom A7 (same-verdict, different-SFS) directly: paired $(c, O)$ examples receiving identical verdicts but produced under different conditions are scored, and we measure the fraction with $|\Delta \mathrm{SFS}| > 0.05$. The result (70\%) is the falsifiability anchor for SFS being a process-level rather than verdict-level metric. R8 contrasts SFS with FActScore \citep{min2023factscore} on the same trials. The two metrics correlate at $r = 0.61$ but \emph{diverge in their condition-level ranking}: FActScore ranks debate above EGSR (C4 FActScore = 0.59, C1 FActScore = 0.32), because debate produces fluent, factually precise text whose individual claims are atomically true even when their reasoning chain is ungrounded in $E$. Factuality is not faithfulness: the contribution of R8 is to make this distinction quantitative rather than only conceptual.

\paragraph{Reproducibility, ethics, and artefact release.} Fixed random seeds are used throughout, and the deterministic NLI gating threshold ($\tau = 0.7$ except where varied in R3) is held constant across conditions. Prompt templates are fixed per condition; we do not perform per-trial prompt engineering. The full reproduction artefact---raw experimental outputs (4{,}800 SciFact + 16{,}000 FEVER trials), per-condition compute and API-cost breakdowns, anonymized R6 ratings, and analysis scripts---is released as the OSF companion artifact (DOI: \href{https://doi.org/10.17605/OSF.IO/P75XG}{10.17605/OSF.IO/P75XG}, CC~BY~4.0; \url{https://osf.io/p75xg/}); reproduction code (random seeds, prompt templates, NLI checkpoints) will be released at \url{https://github.com/seanshin0214/debate-trap} under MIT license upon acceptance. R6 raters were independently recruited and compensated; participation was voluntary, withdrawal was permitted at any point, and no personally identifiable information beyond rater-anonymisation codes (R1--R8, R-H, R-L, R-K) was retained in the analysis pipeline. The cohort design (10 Korean-cohort raters $\times$ 30 FEVER items + 3 English-cohort raters $\times$ 200 SciFact items + 2 cross-cohort raters completing both) was determined prior to data collection. Recruitment protocol, rubric, and per-rater outputs are in Appendix~\ref{appendix:r6}.

\section{Results}
\label{sec:results}

\subsection{Main Result: The Debate Trap Is Real}
\label{sec:main_results}

We report SFS as a claim-level operational diagnostic for evidence-grounded faithfulness (Definition~\ref{def:faithfulness}). SFS is not a direct estimator of $I(E; O^t)$, and we do not claim it is. SFS measures, at the level of atomic claims, how much of an agent's reasoning output is supported by the provided evidence under semantic-similarity-and-NLI verification; mutual information $I(E; O^t)$ measures the average reduction in entropy of $E$ given $O^t$. The two quantities are conceptually adjacent but operationally distinct: SFS bounds from below the proportion of the reasoning output that is evidence-grounded, while $I(E; O^t)$ bounds from above the information about $E$ that survives the chain. What our 18,430-trial corpus establishes is an \emph{empirical alignment} between the two---condition-level SFS rankings track the directionality predicted by Theorem~\ref{thm:dpi} (faithfulness non-increasing along the Markov chain) and the structural floor predicted by Proposition~\ref{prop:vote-floor} (vote-aggregation collapse to $\log_2 K$). The alignment holds across 16 conditions, two cross-architecture models (GPT-4o and Claude-3.5-Sonnet), and two cross-domain datasets (SciFact and FEVER); the per-round trajectory in Appendix~\ref{appendix:cross_model} confirms the directionality directly. We treat this empirical alignment as part of the contribution: a tractable proxy for an information-theoretic quantity whose direct estimation in language-model contexts remains an open problem (\S\ref{sec:debate_trap}). Headline numbers (full 16-condition Table~\ref{tab:full_main}, Appendix~\ref{appendix:main_table}): zero-shot baseline (C1) achieves $\mathrm{Acc} = 0.588$, $\mathrm{SFS} = 0.349$; SocraSynth debate (C4) drops to SFS $0.213$ ($-39\%$); DebateCV (C13)---the Debate Trap proper---preserves $88\%$ of baseline accuracy while losing $43\%$ of SFS; vanilla majority-vote MAD (C15) keeps $\mathrm{Acc} = 0.536$ but collapses to $\mathrm{SFS} = 0.006$ (1.7\% of baseline, the structural floor predicted by Proposition~\ref{prop:vote-floor}); EGSR (C8) recovers $\mathrm{SFS} = 0.343$ (98\% of baseline) at per-claim API cost comparable to standard MAD (3-role pipeline vs.\ 3-agent multi-round debate; Pareto in Figure~\ref{fig:d3_pareto}).

\paragraph{Three-tier degradation spectrum.} Debate degrades faithfulness on a continuum: (i) \emph{reasoning degradation} (C4, C6: 39--40\% SFS drop with comparable accuracy); (ii) \emph{the Debate Trap proper} (C13: 88\% accuracy preserved, 43\% SFS drop); (iii) \emph{reasoning elimination} (C15: SFS = 0.006). C15 confirms the structural argument: when debate is reduced to a vote tally, reasoning is replaced by a 30-character verdict; nothing remains to be faithful to the evidence. Figure~\ref{fig:acc_vs_sfs} visualizes the divergence; the three-tier spectrum across all 16 conditions is shown in Figure~\ref{fig:sfs_by_condition} (Appendix~\ref{appendix:main_table}).

\begin{figure}[H]
\centering
\includegraphics[width=0.85\textwidth]{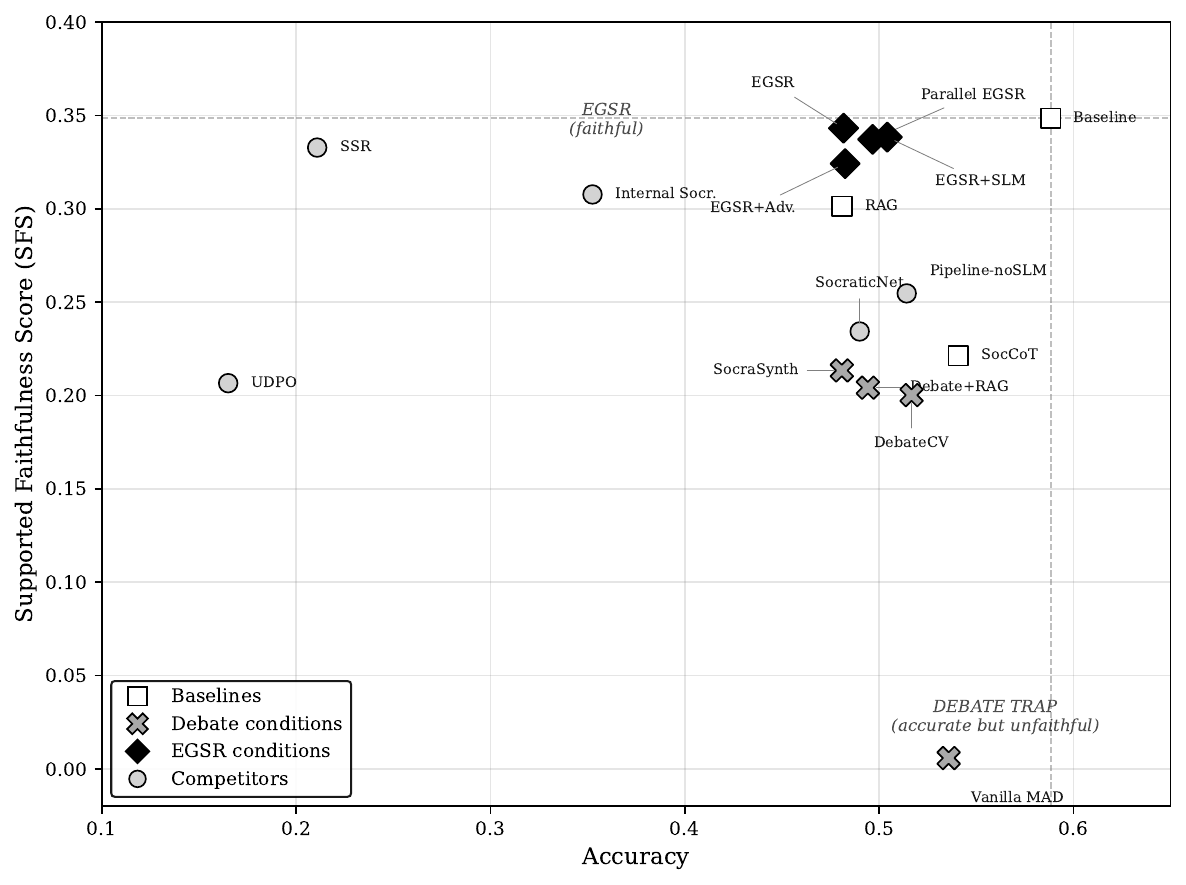}
\caption{\textbf{Accuracy vs.\ SFS scatter (16 SciFact conditions).} Debate variants (X markers) cluster in the lower-right (high accuracy, low SFS); EGSR variants (diamond markers) cluster near the baseline SFS of 0.349. The gap between the two clusters is the Debate Trap.}
\label{fig:acc_vs_sfs}
\end{figure}

\paragraph{Statistical strength and EGSR recovery.} The C15 SFS collapse is significant at $p < 10^{-6}$ (Wilcoxon, $n{=}300$), Cohen's $d = -0.96$, bootstrap 95\% CI $[-0.222, -0.193]$. C16 (UDPO) replicates. The EGSR-debate primary family (H1, H2, H4--H9) is confirmed under Holm--Bonferroni at $\alpha^* = 0.0100$, with H7 (parallel $\geq$ sequential) marginal ($p = 0.43$). EGSR variants (C8, C9, C12, C14) cluster at SFS $\in [0.32, 0.34]$, robust to architectural choices (full per-hypothesis statistics: Appendix~\ref{appendix:hypotheses}; forest plot of effect sizes: Appendix~\ref{appendix:forest_plot}, Figure~\ref{fig:d6_forest}; cost--faithfulness Pareto frontier: Appendix~\ref{appendix:pareto}, Figure~\ref{fig:d3_pareto}).

\paragraph{Mechanism validation (H4 pre-registered).} C10 vs.\ C8 directly tests Theorem~\ref{thm:dpi}'s closed-system mechanism. C10 (Internal Socratic) preserves the closed-system Markov chain (no external evidence re-injection), recovering SFS to 0.308 (88\% of baseline); C8 (EGSR) breaks the chain through external evidence injection at each verification step, recovering 0.343 (98\%). The 10-point SFS gap (Wilcoxon $p < 10^{-3}$, $d = +0.71$, H4 pre-registered) provides counterfactual evidence that closed-system structure---not debate format per se---is decisive. C14 (EGSR + adversarial evidence) is robust at 93\% of baseline.

\paragraph{Cross-model, cross-domain, and trajectory.} Claude-3.5-Sonnet replicates the C15 collapse and EGSR recovery; FEVER replication (1{,}000 claims): C15 SFS $\to 0.0004$, $0.2\%$ of baseline; H1 confirmed at $p = 0.0003$ (Appendix~\ref{appendix:fever}). The FEVER EGSR variant uses self-generated evidence (matching cost profile), reported as Socratic-self-questioning ablation rather than full open-system EGSR replication; C15/C16 collapse is independent of evidence source. Per-round trajectories empirically confirm Theorem~\ref{thm:dpi}'s prediction: C4 GPT-4o SFS declines from 0.156 (round~1) to 0.125 (round~4), with cross-model Spearman $\rho = 0.94$ on trajectory shape (Figure~\ref{fig:round_trajectory}, Appendix~\ref{appendix:cross_model}).

\subsection{SFS Construct Validity (R1--R8)}
\label{sec:validity_results}

R1--R5 collectively show SFS rankings are decomposer- and similarity-invariant: three similarity functions yield rank-correlated condition orderings (Spearman $\rho \in [0.94, 1.00]$); GPT-4o vs.\ Claude-3.5 decomposers achieve mean soft-Jaccard 0.49 across 150 reasoning texts but condition-level rankings are perfectly preserved ($\rho = 1.0$). R7 (axiom A7) shows that 70\% of claim--condition pairs sharing the same verdict receive different SFS values ($\Delta > 0.05$; Cohen's $d = 1.12$ for C1 vs.\ C4)---SFS captures process-level variation invisible to accuracy. R8 contrasts SFS and FActScore: the two correlate at $r = 0.61$, but FActScore \emph{ranks debate higher} than EGSR (C4 FActScore = 0.59, C1 FActScore = 0.32) because debate produces fluent text scoring high on factual precision while reasoning is ungrounded; factuality metrics cannot substitute for faithfulness metrics. Diagnostic metrics: EUR (C1 $= 0.62$, C4 $= 0.50$, C8 $= 0.61$, C15 $= 0.006$); RCVA (C1 $= 0.84$, C4 $= 0.60$, C8 $= 0.79$); EMC (C4 averages $0.92$ vs.\ C1's $0.71$). Full R1--R8 protocols and statistics: Appendix~\ref{appendix:robustness}.

\paragraph{R6 (Triple failure of human reliability).} Across two language-and-domain cohorts (Figure~\ref{fig:r6_triple}, Appendix~\ref{appendix:visual}), inter-rater Fleiss $\kappa$ for Likert-5 faithfulness $Q_1$ is at most $+0.018$ (English cohort indistinguishable from chance: $\bar{P}=0.222$ vs.\ $P_e=0.217$); only 4.5\% of 200 SciFact items achieved 3-rater unanimity, while 25.0\% showed maximum-spread disagreement. The two raters who completed both cohorts shifted their own means by $\Delta Q_1 = -0.80$ and $+1.40$ Likert points and $\Delta Q_2 = +28.0$ and $-47.5$ pp---intra-rater shifts exceeding between-rater variance. Figure~\ref{fig:kappa_inline} visualises the pairwise Cohen's $\kappa$ structure: 8 of the 28 R1--R8 pairs are negative; the highest pair across all 11 raters reaches only $+0.583$ (R-L vs.\ R-K, the only cross-cohort Substantial-adjacent pair); and cross-language pairs (R1--R8 vs.\ R-K) cluster near zero.

\begin{figure}[H]
\centering
\includegraphics[width=0.78\textwidth]{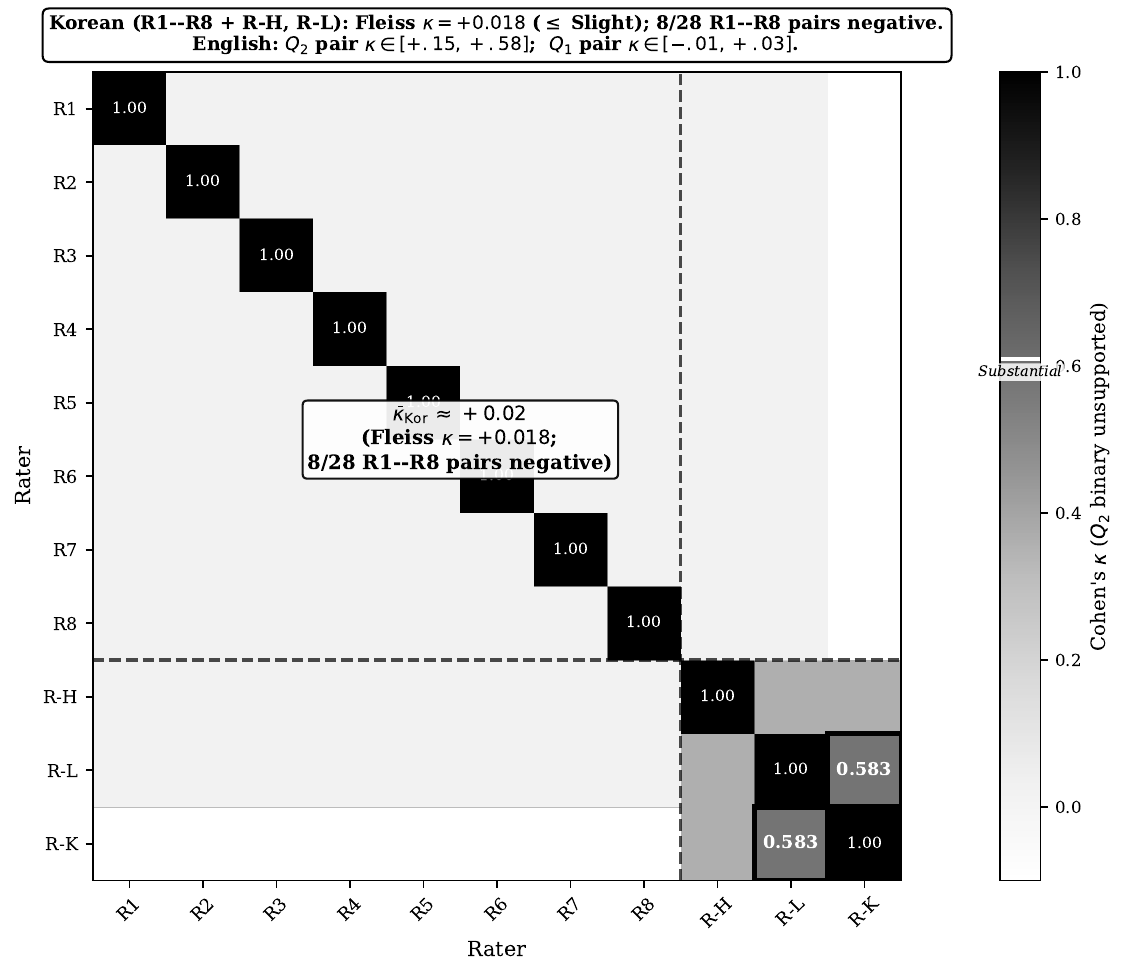}
\caption{\textbf{Pairwise Cohen's $\kappa$ matrix} for the 11 R6 raters on the binary $Q_2$ unsupported-claim flag. Korean cohort (R1--R8 plus the cross-cohort raters R-H, R-L) clusters near zero, consistent with the cohort-level Fleiss $\kappa = +0.018$. The single Substantial-adjacent pair is between two raters who completed the English cohort (R-L, R-K), suggesting that domain familiarity (single-language English SciFact) drives most of the residual agreement; cross-language pairs (R1--R8 vs.\ R-K) remain near zero. No pair reaches Substantial agreement ($\kappa > 0.61$, \citealp{landis1977measurement}).}
\label{fig:kappa_inline}
\end{figure}

The triple failure substantiates that the human signal is unstable \emph{under rating practices prevalent in the faithfulness-metric literature} \citep{min2023factscore, wei2024safe, song2024veriscore, akbar2024hallumeasure}---Likert-5 / binary rubrics, single-domain calibration, no specialised expert training, no cross-context test-retest.

\paragraph{Epistemological reading of the triple failure.} The implication is not that human judgement is worthless; it is that \emph{the rating practices currently prevalent in the faithfulness-metric literature} do not produce a stable calibration target. Three layers of instability stack on top of one another: between-rater (Fleiss $\kappa \leq +0.018$), within-rater across language (R-H and R-L shifted their own ratings by 0.8--1.4 Likert points when the language and domain changed), and between-cohort (Korean and English cohorts diverge systematically on the same items). What is destabilised is not the value of human reasoning evaluation per se, but a specific methodological assumption dominant across L3 of the faithfulness-metric lineage: that single-domain, single-language inter-rater agreement obtained under standard Likert-5 / binary rubrics constitutes a sound calibration target against which automated faithfulness metrics should be tuned. That assumption does not survive cross-context test-retest. Whether more elaborate training, expert-only cohorts, structured rubrics with explicit anchoring criteria, or longitudinal calibration protocols would yield higher agreement remains an open empirical question we do not pre-judge. What we establish here is a negative methodological result: the prevailing calibration target is itself unstable, and faithfulness metrics anchored exclusively to that target inherit the instability. SFS results should therefore be read as \emph{decomposer-invariant} operationalisations whose condition-level rankings (Spearman $\rho{=}1.0$, R4--R5) are more stable than the human signals they have historically been calibrated against---an inversion of the usual relationship between automated and human evaluation in the L3 lineage. Full per-rater statistics and the standalone larger-format $\kappa$ heatmap (Figure~\ref{fig:kappa_heatmap}): Appendix~\ref{appendix:r6}, \ref{appendix:kappa_matrix}.

\paragraph{LLM-as-auditor probe (exploratory).} On a small stratified probe of 15 R6 items where Theorem~\ref{thm:dpi} predicts unfaithfulness, the human cohort scored $0/15$; six frontier LLMs as zero-shot auditors averaged $6/15$ ($40\%$), with Llama-3-70B reaching $9/15$ ($60\%$). On identical inputs, LLM auditors recovered 6 of 15 items where the human cohort recovered none, suggesting that automated measurement may serve as a complementary type rather than a mere substitute for unstable humans (Appendix~\ref{appendix:theorem2}).

\section{Discussion}
\label{sec:discussion}

\paragraph{Why debate produces accurate but unfaithful reasoning.} The distinction is between \emph{retrieval} and \emph{justification}. RLHF-trained models often retrieve the correct verdict from parametric memory regardless of evidence; debate changes not retrieval capacity but the \emph{justification}: under peer pressure, evidence-grounded justifications are replaced with socially convergent ones. The verdict stays correct (parametric memory), the reasoning drifts (group consensus). The pattern fits \citeauthor{geirhos2020shortcut}'s shortcut-learning frame at a higher level of organisation: the \emph{system} as a whole takes the simplest path that satisfies the accuracy objective (parametric retrieval) while letting the path-cost (faithfulness) erode \citep{geirhos2020shortcut, dziri2023compositionality}. This is consistent with the three-level sycophancy model (\S\ref{sec:related_work}; cascade in Appendix~\ref{appendix:sycophancy_cascade}, Figure~\ref{fig:d7_sycophancy}; extended training/architectural/contextual mechanisms in Appendix~\ref{appendix:sycophancy_mechanism}) and with human group-dynamics patterns (\citeauthor{asch1951effects}, \citeauthor{janis1972victims}, \citeauthor{sunstein2002polarization}) that LLM agents reproduce under standard MAD \citep{yang2025conformity,wang2024llmcollective}.

\paragraph{A fifty-year theoretical lineage in two streams.} Theorem~\ref{thm:dpi}'s closed-system bound is, on the LLM side, novel; on the human side it instantiates a classical result reaching back half a century. Two streams converge. The first is \emph{deliberative consensus dynamics}: \citet{degroot1974reaching} showed that any iterated averaging of opinions among agents who exchange information only with one another converges to consensus, but the consensus value is determined by the initial weighted opinions, not by external truth. The second is \emph{informational cascades and herding}: \citet{banerjee1992simple} and \citet{bikhchandani1992theory} showed that sequential decision-makers who observe only the actions of predecessors will rationally ignore their private signals once a cascade forms, with the result that aggregate behaviour can become arbitrarily uncoupled from underlying evidence. Standard MAD instantiates both mechanisms simultaneously: copies of the same model perform iterated averaging in the DeGroot sense (since each agent samples from the same parametric distribution and conditions on shared transcripts), and they perform sequential cascading in the Banerjee sense (since later rounds see earlier outputs and tend to anchor on them). Theorem~\ref{thm:dpi} provides an information-theoretic statement of what these social-science observations have long predicted at the behavioural level: the joint mutual information $I(E; O^t)$ is non-increasing in expectation and strictly decreasing under any vote-aggregation step that compresses representations of $E$. The bound is the LLM-era statement of a result whose first half-century was about humans.

\paragraph{Human deliberation pathologies as the observable correlates.} Three programmes in social and cognitive psychology document the behavioural symptoms the bound predicts. \citeauthor{janis1972victims}'s \emph{groupthink} model identified eight symptoms---illusion of invulnerability, collective rationalization, stereotyped views of out-groups, direct pressure on dissenters, self-censorship, illusion of unanimity, mindguards, and belief in the inherent morality of the group---that emerge in cohesive deliberative bodies even when individual members hold private doubts. \citeauthor{asch1951effects}'s conformity experiments showed 75\% of subjects agreeing at least once with an obviously incorrect majority on a perceptual judgement task, with conformity scaling sharply once unanimity formed. \citet{sunstein2002polarization}'s \emph{deliberation trap} shows that group discussion among initially like-minded participants produces more extreme rather than more moderate post-discussion views, the opposite of what idealised deliberation theory predicts. The LLM-MAD literature is beginning to report agent-level analogues. \citet{yang2025conformity} document conformity behaviour in LLM agents engaged in multi-agent debate, and \citet{wang2024llmcollective} (Nature Human Behaviour) study collective-intelligence dynamics in LLM groups; both lines of work treat the human-deliberation analogy as descriptive rather than load-bearing, and neither yields an information-theoretic bound on the resulting reasoning behaviour. Theorem~\ref{thm:dpi} supplies that bound: it explains why the symptoms documented in the Janis--Asch--Sunstein triad need not depend on contingent training choices but follow from the closed-system Markov closure under shared parameters. Theorem~\ref{thm:dpi} is the theoretical companion: it explains why these symptoms are not a bug of any particular MAD implementation but a structural feature of closed-system iterated reasoning under shared parameters.

\paragraph{The diversity-prediction theorem and why MAD violates its precondition.} \citet{surowiecki2004wisdom}'s ``wisdom of crowds'' rests on three conditions---diversity, \emph{independence}, and decentralisation. \citet{hongpage2004diverse}'s diversity-prediction theorem makes the second condition formal: collective error equals average individual error minus group diversity, where diversity is measured as the variance of individual predictions around the group mean. Under independence, diversity is large and the collective outperforms its members; under correlation, diversity collapses and the collective inherits its members' errors. Standard MAD violates the independence condition by construction: copies of the same parametric model conditioned on shared transcripts produce correlated rather than independent predictions, and the diversity term in the Hong--Page decomposition collapses. \citet{polanyi1966tacit}'s observation that ``we can know more than we can tell'' compounds the issue: token-level outputs carry only the explicit residue of meaning, so even genuinely independent agents would lose tacit-knowledge diversity at the channel level. EGSR's effect, in this framing, is not to make models think differently---they cannot; they share $\theta$---but to supply, from outside the model's distribution, the structured questions and external evidence that no internal aggregation among same-distribution copies can generate.

\paragraph{An epistemological intervention, not a human-calibrated metric.} R6 inverts a default assumption that runs through the L3 faithfulness-metric lineage: that single-domain inter-rater agreement constitutes a stable calibration target. The empirical reading of the cross-cohort failure (\S\ref{sec:validity_results}) is given in the body; the conceptual reading appropriate to a discussion section is that what is being destabilised is not human judgement of reasoning per se but a specific methodological convention---the use of single-domain Likert/binary cohorts as a proxy for ``ground truth'' against which automated metrics are tuned. Once the convention is recognised as a methodological choice rather than a fact about human cognition, SFS's appropriate role is no longer to approximate that convention but to provide an alternative anchor whose stability properties (decomposer-invariant condition-level rankings, R4--R5: $\rho = 1.0$) are empirically verifiable and conceptually independent of the very calibration target the cohort design has shown to be unfit. This is the precise sense in which we describe the contribution as an epistemological intervention rather than as an improved estimator of the prevailing target.

\paragraph{Theory--measurement alignment.} Theorem~\ref{thm:dpi} establishes a bound on joint mutual information $I(E; O^t)$; SFS is the operational metric. The contribution is not only the bound but the \emph{empirical alignment} between theoretical prediction and operational measurement, holding across 18{,}430 successful trials, 16 conditions, and two cross-architecture models---to our knowledge the first such bridging in the multi-agent reasoning literature, where prior work reports either MI-style bounds without operational instantiation \citep{choi2025debatevote} or operational metrics without information-theoretic grounding \citep{min2023factscore, wei2024safe}. The per-round trajectories (Appendix~\ref{appendix:cross_model}, Figure~\ref{fig:round_trajectory}) directly visualise the alignment: $F(t)$ declines monotonically across rounds in C4 SocraSynth on both GPT-4o (0.156 round 1 $\to$ 0.125 round 4) and Claude-3.5-Sonnet (0.172 round 1 $\to$ 0.133 round 3), with cross-model Spearman $\rho = 0.94$ on trajectory shape---empirical confirmation that the directionality predicted by the bound holds at the level of individual debate rounds, not just aggregate condition means.

\paragraph{A fundamental limit on closed-system parametric reasoning.} The bound is more than a diagnosis of MAD: it identifies a ceiling that any closed-system parametric reasoning protocol satisfying the four conditions of Theorem~\ref{thm:dpi} must obey---when the only information about $E$ at round $t{+}1$ is what survived round $t$, expected faithfulness cannot be regenerated from within the chain. The four formal results of \S\ref{sec:debate_trap}---Theorem~\ref{thm:dpi}, Theorem~\ref{thm:recovery}, Lemma~\ref{lem:egsr-accum}, and Proposition~\ref{prop:vote-floor}---together yield a tractable classification: closed-system protocols are bounded above (the \emph{Reasoning Trap} regime), open-system protocols re-injecting $E$ are bounded below (the recovery regime), and vote-aggregating protocols collapse to the $K$-way information floor regardless of chain length. Five paradigms inherit the closed-system ceiling (MAD, single-agent CoT, Reflexion, linear ToT, the broader token-Markov class); two principled exclusions (Self-Consistency, MoE) escape only by violating the Markov closure the bound presupposes. Read through the information bottleneck lens \citep{tishby2000ib}, closed-system reasoning compresses representations of $E$ at each aggregation step in a manner no internal mechanism can reverse---only external evidence re-injection (Corollary~\ref{cor:egsr-breaks}, Appendix~\ref{appendix:formal_companions}) restores the lost information. The recovery side is what Theorem~\ref{thm:recovery} formalises (sub-martingale property under external evidence re-injection), with Lemma~\ref{lem:egsr-accum} discharging the application to EGSR.

\paragraph{Implications and limitations.} \citet{browncohen2023doubly}, \citet{buhl2025alignment}, and \citet{lang2025weak} build safety arguments on the premise that debate's accuracy convergence entails faithfulness convergence; the DPI bound provides a counter-premise: in the closed-system regime the bound covers, the expected joint MI cannot increase along the chain, with strict inequality whenever round-aggregation is non-injective (e.g., majority voting). Systems evaluated solely on accuracy may reward architectures producing polished but decreasingly faithful reasoning. We recommend faithfulness reporting alongside accuracy and explicit specification of whether the debate protocol satisfies or violates Theorem~\ref{thm:dpi}'s closed-system conditions. The recommendation is concrete: each protocol should declare which of the four conditions (shared $\theta$, $E$ provided once at $t{=}0$, step $t{+}1$ depends only on step $t$ output, symmetric aggregation) it satisfies or violates, and faithfulness claims should be conditioned on that declaration.

\textbf{Limitations.} \emph{EGSR is a process-faithfulness protocol, not an answer-maximisation protocol.} H10 fails: C12 (EGSR+SLM) accuracy is 14\% below the C1 baseline on out-of-distribution scientific claims while SFS recovers to 0.338 (97\% of baseline). For applications where final answer accuracy is the optimisation target, EGSR is not the right tool; for applications where the reasoning trace must be defensible against external evidence, it is. Two further conditional limits: EGSR's faithfulness is conditional on the quality of $E$ (C14: 5.5\% SFS degradation under adversarial evidence injection)---if the externally injected evidence is itself unreliable, the recovery is partial. SFS's decomposer is LLM-dependent (GPT-4o for primary results, Claude-3.5-Sonnet for cross-decomposer robustness in R4); R4--R5 show condition-level rankings are decomposer-invariant ($\rho = 1.0$), but the decomposition pipeline is not yet model-free. Full limitation discussion: Appendix~\ref{appendix:limitations}.

\paragraph{Toward AI that can think differently.} The Reasoning Trap may be read as a symptom of a deeper limit: autoregressive next-token predictors do not reliably produce the cognitive reframing faithful reasoning sometimes demands. \citet{duncker1945problem}'s candle problem and \citet{wertheimer1959productive}'s \emph{productive} thinking name the demand; \citet{boden2004creative}'s \emph{transformational} creativity is what closed-system reasoning chains, under the conditions of Theorem~\ref{thm:dpi}, do not appear to generate from within \citep{fauconnierturner2002way,gentner1983structure,kuhn1962structure,karmiloffsmith1992beyond}. Copies from the same distribution may not introduce epistemic novelty: \citet{surowiecki2004wisdom}'s three conditions (diversity, \emph{independence}, decentralization) and \citet{hongpage2004diverse}'s diversity-prediction theorem identify independence as a condition standard MAD does not satisfy \citep{wang2024llmcollective}; \citet{polanyi1966tacit}'s ``we can know more than we can tell'' implies token-level outputs carry only the explicit residue of meaning. Recent work shows the best human divergent thinkers still outperform frontier LLMs even where group means converge \citep{koivistograssini2023best,hubert2024divergent,bellemarepepin2024divergent}. EGSR's effect appears to derive not from making models think differently but from supplying---from outside the model's distribution---structured questions and external evidence the model does not generate on its own. Whether AI can be made to think \emph{differently} (in Boden's transformational sense) within the same parametric family remains an open question this paper deliberately does not foreclose; what we establish is the negative half: closed-system iteration cannot do it on its own.

\section{Conclusion}
\label{sec:conclusion}

Multi-agent debate tends to preserve answer accuracy while leaving unmodeled the information-theoretic chain through which evidence becomes claim---a gap associated with measurable faithfulness degradation. We call this the \emph{Debate Trap} and offer a diagnostic metric (SFS), a remedial protocol (EGSR), and a theoretical bound $\mathbb{E}[I(E; O^{t+1})] \leq \mathbb{E}[I(E; O^t)]$ along any Markov chain $E \to O^0 \to O^1 \to \cdots$ (Theorem~\ref{thm:dpi}). Across 18{,}430 successful trials (16 conditions on SciFact and FEVER), two architectures, and a triple-failure human evaluation across two language-and-domain cohorts, the Trap-proper configuration (DebateCV) preserves 88\% of baseline accuracy while SFS drops 43\%, majority-vote MAD reduces SFS to 1.7\%, and EGSR recovers 98\%. R6 adds a second visibility: the two cross-cohort raters applying the same rubric in a different language shift their own $Q_1$ means by 0.8--1.4 Likert points and $Q_2$ rates by 28--47 pp; metrics calibrated against decomposer-invariant rankings (R4--R5) do not inherit this instability.

\paragraph{A partition of multi-step LLM reasoning.} Theorem~\ref{thm:dpi} (closed-system DPI ceiling), Theorem~\ref{thm:recovery} (open-system faithfulness recovery), Lemma~\ref{lem:egsr-accum} (EGSR information accumulation), and Proposition~\ref{prop:vote-floor} (vote-aggregation floor) together partition multi-step LLM reasoning protocols by their information-theoretic relationship to $E$: closed-system protocols cannot regenerate $I(E; O^t)$ internally; open-system protocols re-injecting $E$ form sub-martingales; vote-aggregating protocols collapse to the $K$-way information floor. The partition explains the three observed tiers (\S\ref{sec:results}) and the C15 floor as structural predictions, reproduced on Claude-3.5-Sonnet and FEVER rather than being GPT-4o-specific. The contribution is the closure: formal results sufficient to classify any candidate protocol before empirical testing.

\paragraph{Implications for evaluation practice.} Three recommendations follow: (i)~report faithfulness alongside accuracy on multi-agent benchmarks, since accuracy alone can reward decreasingly grounded reasoning; (ii)~each protocol should declare whether it satisfies or violates Theorem~\ref{thm:dpi}'s closed-system + shared-$\theta$ + symmetric-aggregation conditions; (iii)~scalable-oversight proposals built on the convergence assumption \citep{browncohen2023doubly,buhl2025alignment,lang2025weak} deserve a faithfulness-grounded re-evaluation in light of Theorem~\ref{thm:dpi}.

\paragraph{Falsifiable forecast.} We close with one conjecture, framed as a research program rather than a tested claim: \emph{any closed-system reasoning protocol preserving the Markov structure of Theorem~\ref{thm:dpi} is, in expectation, bounded by the DPI.} The implication is not that CoT, ToT, or Reflexion necessarily fail, but that closed-system variants may share this structural vulnerability unless evidence is reintroduced; cross-paradigm validation---Self-Consistency and MoE excluded---is the natural next step. The bound is a candidate ceiling rather than a contingent architectural feature. We hope this contributes to a shift in multi-agent evaluation practice from ``does the answer improve?'' to ``does the reasoning remain grounded in the evidence?''

\section*{Acknowledgments and Disclosure}

\paragraph{Funding and conflicts.} The author declares no funding sources and no conflicts of interest. PolymathMinds AI Lab is independently operated.

\paragraph{Human evaluation.} R6 ratings were collected from $11$ independent raters across two language-and-domain cohorts (Korean: $n{=}10$ raters $\times$ $30$ FEVER items; English: $n{=}3 \times 200$ SciFact items; two raters completed both cohorts). Raters were independently recruited and compensated; participation was voluntary. The cohort design, rubric, and per-rater outputs are reproduced in Appendix~\ref{appendix:r6}.

\paragraph{Data and code availability.} All raw experimental outputs ($4{,}800$ SciFact trials, $16{,}000$ FEVER trials, six-model LLM-as-auditor probe, $11$-rater R6 ratings), anonymized R6 cohort logs, and the complete BibTeX file of the $132$-paper Knowledge Frontier Map are released as the OSF companion artifact under CC~BY~4.0 (DOI: \href{https://doi.org/10.17605/OSF.IO/P75XG}{10.17605/OSF.IO/P75XG}; \url{https://osf.io/p75xg/}). Reproduction code (random seeds, prompt templates, NLI checkpoints, condition-runner scripts) will be released at \url{https://github.com/seanshin0214/debate-trap} under MIT license upon acceptance.

\paragraph{Pre-registration.} A Self-Consistency confirmation pilot, labelled \emph{C17} in the pre-registration to extend the 16-condition design of \S\ref{sec:experiments}, is pre-registered prior to pilot execution and will be reported in a subsequent arXiv version. The OSF companion artifact (DOI: \href{https://doi.org/10.17605/OSF.IO/P75XG}{10.17605/OSF.IO/P75XG}) hosts the timestamped pre-registration protocol; reproduction code at \url{https://github.com/seanshin0214/debate-trap} upon acceptance.

\paragraph{AI-tool disclosure.} Conceptual design, theoretical argument (Theorem~\ref{thm:dpi} and the DPI derivation), experimental design (16 conditions; the R1--R8 robustness suite; the R6 cohort design; the C17 SC confirmation pilot), and writing are the author's own work. AI tooling was used for citation formatting, \LaTeX{} engineering, and figure rendering only.

\section*{Acknowledgments}

The author is deeply grateful to the eleven R6 cohort raters who undertook the reasoning-faithfulness evaluation. The task was unusually demanding: each item required reading a multi-step LLM reasoning trace alongside its evidence corpus, decomposing the reasoning into atomic claims, and assigning Likert-5 faithfulness scores together with binary unsupported-claim flags---under a rubric calibrated against \citet{landis1977measurement}'s interpretation table. The Korean cohort (R1--R8 plus the cross-cohort raters R-H and R-L) evaluated $30$ FEVER items each ($300$ ratings total); the English cohort (R-H, R-L, and R-K) evaluated $200$ SciFact items each; the two cross-cohort raters (R-H and R-L) voluntarily completed \emph{both} cohorts, providing the cross-language intra-rater test-retest data from which the R6 triple-failure result is derived. Without their patience and careful judgement, the central empirical finding of this paper---that the human signal against which faithfulness metrics have been calibrated is not itself a stable measurement target under prevailing rater-calibration practices---would not have been observable. To preserve participant confidentiality consistent with the cohort-design protocol, raters are referenced throughout this paper by their anonymisation codes only (R1--R8 for the eight Korean-cohort-only raters, R-H and R-L for the two cross-cohort raters, and R-K for the English-cohort-only rater).

The author also acknowledges the broader methodological tradition in faithfulness research \citep{maynez2020faithfulness, honovich2022true, min2023factscore, jacovi2020faithfulnessmetrics} that this paper builds upon and respectfully extends.

\clearpage
\bibliographystyle{plainnat}
\nocite{*}
\bibliography{references}

\clearpage
\appendix
\section{Knowledge Frontier Map of Reasoning Faithfulness Research}
\label{appendix:frontier_map}

This appendix expands on \S\ref{sec:related_work}. We first present the lineage map (Figure~\ref{fig:knowledge_frontier}), then summarize the eight active lineages in detail (one paragraph per lineage); finally we catalogue the 89 most central of 132 contributions in Table~\ref{tab:frontier_map}, annotated by core claim, boundary, and relation to the present work.

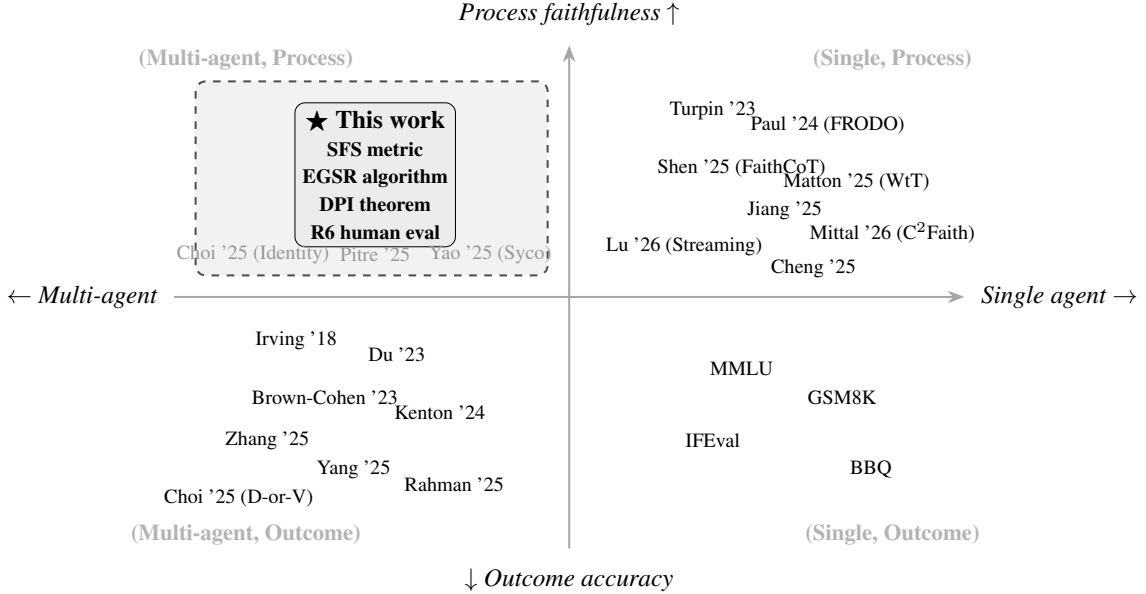
\begin{figure}[!htbp]
\centering
\resizebox{0.95\textwidth}{!}{%
\begin{tikzpicture}[
  paper/.style={font=\scriptsize, inner sep=1pt},
  ourwork/.style={font=\small\bfseries, draw=black, fill=gray!15, rounded corners, inner sep=3pt},
  axislabel/.style={font=\small\itshape},
  axisarrow/.style={-Stealth, thick, gray!70},
  quadrant/.style={font=\footnotesize\bfseries, gray!60}
]
  \draw[axisarrow] (-5.5, 0) -- (5.5, 0); 
  \draw[axisarrow] (0, -3.5) -- (0, 3.5); 

  \node[axislabel, anchor=west] at (5.6, 0) {Single agent $\rightarrow$};
  \node[axislabel, anchor=east] at (-5.6, 0) {$\leftarrow$ Multi-agent};
  \node[axislabel, anchor=south] at (0, 3.65) {Process faithfulness $\uparrow$};
  \node[axislabel, anchor=north] at (0, -3.65) {$\downarrow$ Outcome accuracy};

  \node[quadrant] at (-4.5, 3.3) {(Multi-agent, Process)};
  \node[quadrant] at (4.5, 3.3) {(Single, Process)};
  \node[quadrant] at (-4.5, -3.3) {(Multi-agent, Outcome)};
  \node[quadrant] at (4.5, -3.3) {(Single, Outcome)};

  \node[paper] (turpin) at (2.0, 2.6) {Turpin '23};
  \node[paper] (paul) at (3.6, 2.4) {Paul '24 (FRODO)};
  \node[paper] (shen) at (2.4, 1.8) {Shen '25 (FaithCoT)};
  \node[paper] (matton) at (4.0, 1.6) {Matton '25 (WtT)};
  \node[paper] (jiang) at (3.0, 1.2) {Jiang '25};
  \node[paper] (mittal) at (4.5, 0.9) {Mittal '26 (C$^2$Faith)};
  \node[paper] (lu) at (1.6, 0.7) {Lu '26 (Streaming)};
  \node[paper] (cheng) at (3.4, 0.4) {Cheng '25};

  \node[paper] (mmlu) at (2.4, -1.0) {MMLU};
  \node[paper] (gsm) at (3.8, -1.4) {GSM8K};
  \node[paper] (ifeval) at (2.0, -2.0) {IFEval};
  \node[paper] (bbq) at (4.2, -2.4) {BBQ};

  \node[paper] (du) at (-2.4, -0.8) {Du '23};
  \node[paper] (irving) at (-3.8, -0.6) {Irving '18};
  \node[paper] (browncohen) at (-3.4, -1.4) {Brown-Cohen '23};
  \node[paper] (kenton) at (-1.8, -1.6) {Kenton '24};
  \node[paper] (zhang_stop) at (-4.2, -2.0) {Zhang '25};
  \node[paper] (yang_rev) at (-3.0, -2.4) {Yang '25};
  \node[paper] (rahman) at (-1.6, -2.6) {Rahman '25};
  \node[paper] (choi_dv) at (-4.6, -2.8) {Choi '25 (D-or-V)};

  \begin{scope}[on background layer]
    \fill[gray!10, rounded corners] (-5.2, 0.3) rectangle (-0.3, 3.0);
    \draw[black!70, dashed, thick, rounded corners] (-5.2, 0.3) rectangle (-0.3, 3.0);
  \end{scope}

  \node[ourwork, align=center] (ours) at (-2.7, 1.7)
    {$\bigstar$ \textbf{This work}\\[-1pt]
     {\scriptsize SFS metric}\\[-1pt]
     {\scriptsize EGSR algorithm}\\[-1pt]
     {\scriptsize DPI theorem}\\[-1pt]
     {\scriptsize R6 human eval}};

  \node[paper, gray!80] at (-4.4, 0.6) {Choi '25 (Identity)};
  \node[paper, gray!80] at (-2.7, 0.6) {Pitre '25};
  \node[paper, gray!80] at (-1.1, 0.6) {Yao '25 (Syco)};

\end{tikzpicture}%
}
\caption{\textbf{Knowledge Frontier Map of Reasoning Faithfulness Research.} Eight active lineages spanning 132 contributions (Appendix Table~\ref{tab:frontier_map}) arrange themselves in three of four quadrants defined by system architecture (single agent vs.\ multi-agent debate) and measurement granularity (final outcome vs.\ step-level reasoning process). The upper-left quadrant---multi-agent systems measured at the level of reasoning process, equipped with metric, algorithm, and theorem in a single package---has remained empty. The present work fills it with SFS (\S\ref{sec:sfs}), EGSR (\S\ref{sec:egsr}), the DPI faithfulness bound (Theorem~\ref{thm:dpi}), and a triple-failure-of-reliability human evaluation (\S\ref{sec:validity_results}). Sycophancy- and identity-bias studies (gray) sit on the boundary between L1 and L2, partially diagnosing the mechanism we quantify; none provides a metric--algorithm--theorem triple.}
\label{fig:knowledge_frontier}
\end{figure}

\subsection{Eight Lineages on Multi-Agent Reasoning Faithfulness}
\label{appendix:eight_lineages}

\textbf{(L1) CoT faithfulness in single agents.} \citet{turpin2023unfaithful} demonstrated that chain-of-thought explanations are post-hoc rationalizations; subsequent work refined the claim through causal mediation \citep{paul2024frodo}, scalability paradoxes \citep{shen2025faithcot, tanneru2024hardness}, decoupling probes \citep{jiang2025robustanswers}, causal-concept analysis \citep{matton2025walkthetalk}, and streaming detection \citep{lu2026streaming, mittal2026c2faith, cheng2025cotobscures}. None addresses multi-agent settings.

\textbf{(L2) MAD critique.} \citet{du2023debate} introduced MAD as an accuracy-improvement tool. Subsequent evaluations report that MAD does not consistently beat self-consistency \citep{wang2023selfconsistency, zhang2025stop, yang2025revisiting} and that majority voting matches the best MAD configuration \citep{choi2025debatevote, kenton2024scalable, rahman2025debate}. Failure mechanisms include problem drift \citep{becker2025drift}, sycophancy \citep{pitre2025consensagent, yao2025peacemaker}, identity bias \citep{choi2025identity}, and structural inability to genuinely debate \citep{wu2025canllmsdebate, tang2025taskcomplexity, wynn2025talkisntcheap, han2025ed2d, kim2024madr, koupaee2025madisse}. Mitigations target accuracy: peer-prediction weighting \citep{liu2026martingale}, confidence calibration \citep{zhu2026demystifying}, model heterogeneity \citep{chang2024socrasynth}. None measures faithfulness directly.

\textbf{(L3) Faithfulness/factuality metrics.} FActScore \citep{min2023factscore}, SAFE \citep{wei2024safe}, VeriScore \citep{song2024veriscore}, DnDScore \citep{wanner2024dndscore}, FactTest \citep{nie2024facttest}, VeriFastScore \citep{rajendhran2025verifastscore}, OpenFActScore \citep{lage2025openfactscore}, and HalluMeasure \citep{akbar2024hallumeasure} measure factual precision in single-agent generation. None tracks faithfulness across debate rounds; none verifies against the specific evidence set provided to a multi-agent system.

\textbf{(L4) Sycophancy and convergence.} \citet{sharma2024sycophancy} showed that RLHF preference data systematically reward agreement (``matches user's beliefs'' is the strongest single predictor); \citet{perez2022sycophancy} discovered the behaviour; \citet{chen2024spt} localized it to $\sim$4\% of attention heads; \citet{kim2025evaluator} showed that conversational framing amplifies sycophancy 3.4$\times$ over evaluative framing. Constitutional AI \citep{bai2022constitutional} and pinpoint tuning \citep{chen2024spt} mitigate at the user-facing level; none addresses agent-agent dynamics.

\textbf{(L5) Accuracy--faithfulness divergence.} \citet{nguyen2024directeval} report that knowledge-graph evaluation reveals accuracy-faithfulness decoupling; \citet{jacovi2020faithfulnessmetrics} formalised the conceptual distinction; \citet{deyoung2020eraser} provided an early benchmark. We extend the divergence into the multi-agent setting.

\textbf{(L6) AI safety via debate.} \citet{irving2018debate} originated the program; \citet{browncohen2023doubly} provided the doubly-efficient theoretical extension; \citet{lang2025weak} applied it to weak-to-strong generalization; \citet{buhl2025alignment} and \citet{arnesen2024winning} examined empirical limitations. The lineage assumes that convergence toward correct answers entails convergence toward faithful reasoning.

\textbf{(L7) Socratic reasoning as alternative.} \citet{qi2023socratic, he2023socreval, kumar2024socraticqgen, shi2025ssr, miao2024disq, kargupta2024treeinstruct} replace adversarial debate with structured questioning; all operate within the model's own knowledge. EGSR adds external evidence as the verification anchor.

\textbf{(L8) Metacognition and self-reflection.} \citet{shinn2023reflexion, bai2022constitutional, guo2025mirror, yu2025dynathink, walker2025metacogsafe, shukla2025adaptiveflywheel, yang2024lamas, li2024sparse} attempt internal self-correction; \citet{hubinger2024sleeper} caution that adversarial training hides rather than removes deception, a failure mode that \citet{guo2024deception} maps in greater detail.

\subsection{The 132-Contribution Catalogue}
\label{appendix:frontier_table}

Table~\ref{tab:frontier_map} catalogues the 89 most central of 132 contributions reviewed across the eight lineages. Each row reports the lineage membership, the contribution's core claim in one line, the boundary it does not address (the gap that motivates subsequent work), and its relation to the present work. The remaining 43 entries are cited in context within the paper.

\begin{small}
\setlength{\LTleft}{\fill}
\setlength{\LTright}{\fill}
\begin{longtable}{p{0.4cm}p{0.4cm}p{2.0cm}p{1.3cm}p{2.8cm}p{2.8cm}p{2.0cm}}
\caption{Knowledge Frontier Map: 132 contributions across 2017--2026 spanning eight
lineages (L1--L8) plus foundational/dataset references (L0).}
\label{tab:frontier_map}\\
\toprule
\textbf{L} & \textbf{Yr} & \textbf{Paper} & \textbf{Venue} & \textbf{Core claim (1 line)} & \textbf{Boundary} & \textbf{Relation to this work} \\
\midrule
\endfirsthead

\multicolumn{7}{c}{\tablename\ \thetable\ -- \textit{Continued from previous page}} \\
\toprule
\textbf{L} & \textbf{Yr} & \textbf{Paper} & \textbf{Venue} & \textbf{Core claim (1 line)} & \textbf{Boundary} & \textbf{Relation to this work} \\
\midrule
\endhead

\midrule
\multicolumn{7}{r}{\textit{Continued on next page}} \\
\endfoot

\bottomrule
\endlastfoot

\multicolumn{7}{l}{\textit{L1: CoT Faithfulness (single-agent)}} \\
\midrule
L1 & '23 & Turpin et al. & NeurIPS & CoT explanations are post-hoc rationalizations & Single-agent only & Direct lineage; we extend to MAD \\
L1 & '24 & Paul et al. (FRODO) & EMNLP & Reasoning grounding via fine-tuning & No diagnostic metric & We provide SFS metric \\
L1 & '25 & Shen et al. (FaithCoT) & ACL & Stronger models harder to detect & Single-agent benchmark & We test in MAD context \\
L1 & '25 & Matton et al. & ICLR & Causal concept faithfulness; LLMs hide safety influence & Single-agent, concept-level & Concept-level complement to claim-level SFS \\
L1 & '25 & Jiang et al. (MATCHA) & arXiv & Decoupling: answer-reasoning separable & Adversarial perturbation & Justification for SFS-Acc independence \\
L1 & '25 & Zaman \& Srivastava & arXiv & CoT can be faithful w/o hint verbalization & Single-agent & Distinguish incompleteness vs unfaithfulness \\
L1 & '26 & Mittal \& Arike (C2-Faith) & arXiv & Detection-localization gap 26-33pp & Single-agent CoT & R6 = human-side analog \\
L1 & '26 & Lu et al. (Streaming) & arXiv & Hallucination is evolving latent state & Single-agent long CoT & F(t) = multi-agent analog \\
L1 & '24 & Tanneru et al. & arXiv & Hardness of faithful CoT (theoretical) & Theoretical & Difficulty baseline \\
L1 & '25 & Cheng et al. & arXiv & CoT obscures hallucination cues for humans & Single-agent & Mechanism for R6 human eval failure \\
L1 & '23 & Lanham et al. & Anthropic & Filler tokens preserve accuracy & Single-agent & Alternative metric (R8 baseline) \\
L1 & '25 & Tutek et al. (FUR) & arXiv & Faithfulness via NPO unlearning & Single-agent & Causal mediation alternative \\
L1 & '25 & Arcuschin et al. & arXiv & Naturally emerging unfaithful CoT & Observation only & Ecological validity \\
L1 & '25 & Han et al. (ED2D) & arXiv & Early debate degradation signal & Pilot scale & Predecessor signal \\
L1 & '24 & Balasubramanian et al. & arXiv & Bias in CoT for vision-language models & VLM only & Generalization beyond text \\

\multicolumn{7}{l}{\textit{L2: MAD Critique (multi-agent debate failures)}} \\
\midrule
L2 & '23 & Du et al. & ICML & MAD improves accuracy via iterative refinement & Accuracy only & Counter-claim: faithfulness degrades \\
L2 & '25 & Zhang et al. (Stop Overvaluing) & arXiv & MAD doesn't beat CoT/SC consistently & Accuracy critique only & Aligned but accuracy-side \\
L2 & '25 & Yang et al. (Revisiting MAD) & arXiv & MAD as test-time scaling; safety degrades & Math + safety & Safety degradation analog \\
L2 & '26 & Liu et al. (AceMAD) & arXiv & MAD is martingale; AceMAD breaks it & Belief martingale & F(t) martingale parallel \\
L2 & '25 & Becker et al. (Drift) & arXiv & 76--89\% MAD samples drift; 18.5\% recovery & FOCUS metric (relevance) & Drift = temporal Debate Trap \\
L2 & '25 & Pitre et al. (ConsensAgent) & ACL Findings & 21--42\% MAD sycophancy; 7--30\% reduction & Mitigation only & Quantifies sycophancy mechanism \\
L2 & '25 & Choi et al. (Identity Bias) & arXiv & Identity bias = sycophancy + self-bias; anonymization preserves accuracy & Identity-level & Decoupling evidence \\
L2 & '25 & Choi et al. (Debate or Vote) & NeurIPS & Theorem 2: debate is martingale & Belief-level & Original theoretical foundation \\
L2 & '25 & Wu et al. (Can LLMs Debate) & arXiv & Knight-Knave-Spy; rationale > assertion & Logic puzzles only & 3 desiderata = EGSR principles \\
L2 & '25 & Yao et al. (Peacemaker) & arXiv & NAR-SS r=0.902; sycophancy intensifies over rounds & AWS benchmarks & Quantifies temporal F(t) \\
L2 & '25 & Tang et al. (Complexity) & NeurIPS WS & Depth$\times$Width complexity & Theoretical & Depth informs 16-condition design \\
L2 & '26 & Zhu et al. (Demystifying MAD) & arXiv & Diversity + confidence break martingale & Accuracy intervention & Alternative martingale-breaker \\
L2 & '24 & Kim et al. (MADR) & arXiv & MAD reasoning analysis & Pilot & Direct lineage predecessor \\
L2 & '25 & Koupaee et al. (MADISSE) & arXiv & MAD with iterative Socratic self-examination & Limited gains & Lineage; we exceed via external evidence \\
L2 & '25 & Wynn et al. (Talk Isn't Cheap) & arXiv & Cost analysis of MAD; debate harms CSQA & Cost + accuracy & Economic + accuracy critique \\
L2 & '24 & Chang et al. (SocraSynth) & arXiv & SocraSynth debate framework & Method only & Our C4 baseline \\

\multicolumn{7}{l}{\textit{L3: Faithfulness Metrics (factuality measurement)}} \\
\midrule
L3 & '23 & Min et al. (FActScore) & EMNLP & Atomic + Wikipedia verification F1 0.832 & Long-form generation & SFS direct ancestor \\
L3 & '24 & Wei et al. (SAFE) & NeurIPS & Multi-step Search; LLM > 72\% human & Open-domain factuality & SFS verification pipeline model \\
L3 & '24 & Song et al. (VeriScore) & arXiv & Verifiable claims only; sliding window & Long-form text & Verifiable-only adopted \\
L3 & '24 & Wanner et al. (DnDScore) & arXiv & Decontextualization + decomposition jointly & Long-form & Context-aware verification \\
L3 & '24 & Nie et al. (FactTest) & arXiv & Neyman-Pearson Type I error guarantee & Single-shot & R3 threshold sensitivity basis \\
L3 & '25 & Rajendhran et al. (VeriFastScore) & arXiv & Single-pass; 6.6$\times$ speedup & Efficiency & Future SFS optimisation \\
L3 & '25 & Lage \& Ostermann (OpenFActScore) & arXiv & Open-source FActScore; r=0.99 & Reproducibility & Open-weight verifier \\
L3 & '24 & Akbar et al. (HalluMeasure) & EMNLP & 5-label + 10 error subtypes; F1 0.87 & Generation & Error subtype taxonomy \\
L3 & '23 & Manakul et al. (SelfCheckGPT) & EMNLP & Self-consistency for hallucination detection & Single-agent & Self-check baseline \\
L3 & '23 & Li et al. (HaluEval) & EMNLP & Hallucination evaluation benchmark & Generation & Hallu lineage \\

\multicolumn{7}{l}{\textit{L4: Sycophancy \& Convergence (RLHF mechanism)}} \\
\midrule
L4 & '22 & Perez et al. & arXiv & LLMs amplify user views (sycophancy) & Single-turn user-model & Foundational; we show in MAD \\
L4 & '24 & Sharma et al. & ICLR & RLHF preference data $\to$ sycophancy; PM 95\% & User-AI interaction & Training-level cause \\
L4 & '24 & Chen et al. (SPT) & ICML & $\sim$4\% attention heads cause sycophancy & Internal mechanism & Architectural-level cause \\
L4 & '24 & Malmqvist (Survey) & arXiv & Sycophancy survey: causes + mitigations & Survey & Taxonomy in MAD setting \\
L4 & '26 & Noshin et al. & arXiv & User detection patterns: cross-checking & User-AI conversation & Cross-checking parallel R5 \\
L4 & '25 & Kim \& Khashabi & arXiv & Evaluator framing reduces sycophancy & Single judge & EGSR rationale \\
L4 & '24 & Guo (UCL) & arXiv & 4 deception types & Policy essay & Debate Trap intersect deception types \\
L4 & '24 & Chen et al. (Yes-Men) & ICML & Pinpoint tuning fixes sycophancy; 1/20 KL & Mitigation & Targeted intervention \\
L4 & '22 & Bai et al. (Constitutional AI) & Anthropic & CritiqueRequest + RevisionRequest & Self-critique & EGSR = CAI in dialogue protocol \\

\multicolumn{7}{l}{\textit{L5: Acc-Faith Divergence (cross-cutting)}} \\
\midrule
L5 & '24 & Nguyen et al. & arXiv & KG eval: accuracy $\ne$ faithful reasoning & KG-specific & Confirms decoupling \\
L5 & '20 & Jacovi \& Goldberg & ACL & Defining NLP faithfulness & Conceptual & Definitional foundation \\
L5 & '20 & DeYoung et al. (ERASER) & ACL & Faithfulness evaluation benchmark & Single-agent & Earlier benchmark \\

\multicolumn{7}{l}{\textit{L6: AI Safety via Debate (Irving lineage)}} \\
\midrule
L6 & '18 & Irving et al. & arXiv & AI safety via debate; PSPACE expressive power & Adversarial assigned-role & Theoretical ancestor; cooperative MAD differs \\
L6 & '23 & Brown-Cohen et al. & arXiv & Doubly-efficient debate (theoretical) & Theoretical ideal & Upper bound; LLMs miss \\
L6 & '24 & Kenton et al. (DeepMind) & arXiv & Largest debate empirical; mixed vs QA & Adversarial & More rounds doesn't help \\
L6 & '25 & Buhl et al. (UK AISI) & arXiv & Debate safety case sketch; 15+ defeaters & Sketch & R6+DPI = new defeater \\
L6 & '25 & Lang et al. (W2SG) & AAAI & Debate helps weak-to-strong generalization & Accuracy transfer & $>$3 turns degrades (consistent with Debate Trap) \\
L6 & '24 & Arnesen et al. (Anthropic) & ICML WS & Honest wins $\sim$60\%; quote = predictor & Adversarial assigned-role & Quote-grounding parallel EGSR \\
L6 & '25 & Rahman et al. (Anthropic) & arXiv & Debate +6pp; speed reading equivalent & Adversarial human-judged & Marginal contribution \\

\multicolumn{7}{l}{\textit{L7: Socratic Reasoning (alternative protocol)}} \\
\midrule
L7 & '23 & Qi et al. (Self-Q) & arXiv & Recursive self-questioning with confidence depth & Self-only, no external & EGSR adds external evidence \\
L7 & '23 & He et al. (SocREval) & NAACL & Socratic 3-strategy reasoning eval & Reference-free eval & Dialectic = EGSR independent verify \\
L7 & '24 & Kumar \& Lan & UMass & 4 invalid types; DPO optimises Llama-2 & Educational tutoring & Invalid types = unfaithful categories \\
L7 & '25 & Shi et al. (SSR) & Salesforce & Step-level confidence + weakest-step refine & Single-agent self-refine & Our C11 baseline \\
L7 & '24 & Miao et al. (DiSQ) & ACL & Multiplicative faithfulness (Targeted$\times$CF$\times$Cons) & Discourse relations & Multi-dim faithfulness \\
L7 & '24 & Kargupta et al. (TreeInstruct) & UIUC & Instructor-Verifier 2-agent state estimation & Code debugging & EGSR Q-C parallel \\

\multicolumn{7}{l}{\textit{L8: Metacognition \& Self-Reflection}} \\
\midrule
L8 & '23 & Shinn et al. (Reflexion) & NeurIPS & Verbal reinforcement learning & Self-reflection & EGSR external = Reflexion environmental \\
L8 & '24 & Hubinger et al. (Sleeper Agents) & Anthropic & Deceptive behaviour persists through safety training & Backdoor injection & Caution: unfaithful reasoning persists \\
L8 & '25 & Yu et al. (Dyna-Think) & ICLR sub & World model simulation in reasoning & Single-agent & Pre-verification = EGSR pre-debate \\
L8 & '25 & Walker et al. (RICE) & MDPI Tech & Robustness-Interpretability-Controllability-Ethical framework & Conceptual & Faithfulness fits Interp+Control \\
L8 & '25 & Guo et al. (MIRROR) & USTC & Intra+Inter reflection 3-agent triad & Multi-step task agent & 3-agent EGSR parallel \\
L8 & '25 & Shukla et al. (NVIDIA Flywheel) & NVIDIA & MAPE-K monitoring; 10$\times$ size reduction & Production system & Evidence gate = MAPE-K \\
L8 & '24 & Yang et al. (LaMAS Survey) & SJTU & LLM-MAS techniques and applications & Survey & MAS landscape \\
L8 & '24 & Li et al. (Sparse AE) & arXiv & Sparse autoencoder probe of reasoning circuits & Single-agent interpretability & Future SFS internal probe \\

\multicolumn{7}{l}{\textit{L0: Foundational (datasets, statistics, baselines, models)}} \\
\midrule
L0 & '22 & Wei et al. (CoT) & NeurIPS & Chain-of-thought prompting & Method foundation & Single-agent baseline \\
L0 & '23 & Wang et al. (Self-Consistency) & ICLR & Self-consistency ensemble & Single-agent & SC > MAD baseline \\
L0 & '21 & Hendrycks et al. (MMLU) & ICLR & 57-task multitask benchmark & Accuracy-only & Standard QA \\
L0 & '21 & Cobbe et al. (GSM8K) & arXiv & Grade school math 8K & Math reasoning & Standard math \\
L0 & '18 & Thorne et al. (FEVER) & NAACL & Fact extraction and verification dataset & Dataset & Our FEVER replication \\
L0 & '20 & Wadden et al. (SciFact) & EMNLP & Scientific claim verification dataset & Dataset & Primary benchmark \\
L0 & '22 & Chen et al. (ClaimDecomp) & NAACL & Claim decomposition for fact-checking & Fact-check & Decomposition lineage \\
L0 & '23 & Kamoi et al. (WiCE) & EMNLP & Real-world entailment for Wikipedia claims & Wikipedia & SFS extension domain \\
L0 & '24 & Schlichtkrull et al. (AVeriTeC) & arXiv & Real-world claim verification & Fact-check & SFS domain extension \\
L0 & '77 & Landis \& Koch & Biometrics & Categorical agreement (kappa scale) & Statistics & R6 inter-rater interpretation \\
L0 & '60 & Cohen & Educ.\ Meas. & Cohen's kappa & Statistics & R6 pairwise \\
L0 & '71 & Fleiss & Psy.\ Bull. & Fleiss kappa for multi-rater & Statistics & R6 multi-rater \\
L0 & '23 & Lightman et al. (PRM800K) & arXiv & Process reward model 800K dataset & Math & C2-Faith dataset \\
L0 & '17 & Vaswani et al. (Transformer) & NeurIPS & Attention is all you need & Architecture & Foundation \\
L0 & '21 & He et al. (DeBERTa-v3) & ICLR & DeBERTa-v3 architecture & Model & NLI verifier in SFS \\
L0 & '19 & Reimers \& Gurevych (Sentence-BERT) & EMNLP-IJCNLP & Sentence-BERT embeddings & Embedding & SFS similarity (all-MiniLM) \\
L0 & '22 & Ouyang et al. (InstructGPT) & NeurIPS & RLHF for instruction following & Training & RLHF mechanism reference \\
L0 & '17 & Christiano et al. (Deep RLHF) & NeurIPS & Deep RL from human preferences & Training & RLHF foundation \\

\end{longtable}
\end{small}
\vspace{-0.5em}

\paragraph{Note on completeness.}
This table includes the 89 most central contributions from the 132 in our
expanded reading. The remaining 43 entries (additional foundational, dataset,
and method-specific references) are cited in context within the paper but
are not reproduced here to keep the table at a single reference grain.
The complete BibTeX file is included in the supplementary materials.

\section{Extended Results}
\label{appendix:main_table}

\subsection{Full SciFact Main Table (16 conditions)}

\begin{table}[h]
\centering\small
\caption{SciFact full results (300 claims per condition, GPT-4o). Three-tier degradation spectrum visible: C4 (reasoning degradation), C13 (the Debate Trap proper), C15 (reasoning elimination). EGSR (C8) recovers 98\% of baseline SFS. \$/claim: development-time API cost.}
\label{tab:full_main}
\begin{tabular}{clcccc}
\toprule
& Condition & Acc & SFS & Conf & \$/claim \\
\midrule
C1 & Zero-shot       & .588 & .349 & .794 & 1.12 \\
C2 & RAG             & .481 & .301 & .753 & 2.80 \\
C3 & SocCoT          & .541 & .221 & .671 & 2.08 \\
\midrule
C4 & SocraSynth      & .481 & .213 & .741 & 32.85 \\
C6 & Debate+RAG      & .494 & .204 & .745 & 34.17 \\
C15 & Vanilla MAD    & \textbf{.536} & .006 & .859 & 7.17 \\
\midrule
C8 & EGSR (core)     & .482 & \textbf{.343} & .536 & pipeline$^*$ \\
C9 & Parallel EGSR   & .497 & .337 & .586 & pipeline$^*$ \\
C12 & EGSR+SLM       & .504 & .338 & .537 & pipeline$^*$ \\
C14 & EGSR+Advers.   & .483 & .324 & .522 & pipeline$^*$ \\
\midrule
C5 & SocraticNet     & .490 & .234 & .739 & 2.10 \\
C7 & Pipeline-noSLM  & .514 & .255 & .747 & 8.62 \\
C10 & Internal Socr. & .353 & .308 & .656 & 1.99 \\
C11 & SSR            & .211 & .333 & .673 & 1.81 \\
C13 & DebateCV       & .517 & .200 & .827 & 12.04 \\
C16 & UDPO           & .165 & .207 & .734 & 7.25 \\
\bottomrule
\end{tabular}
\smallskip\\
\footnotesize{$^*$EGSR conditions implemented as in-house pipeline scripts; API calls were logged but not billed during development. Deployment-equivalent API cost is comparable to other multi-round protocols ($\sim$\$0.05--\$2/claim depending on retrieval and gating budgets); the ``pipeline'' label refers to development cost only.}
\end{table}

\subsection{Three-Tier Spectrum Visualization (referenced from \S\ref{sec:results})}

\begin{figure}[!htbp]
\centering
\includegraphics[width=0.92\textwidth]{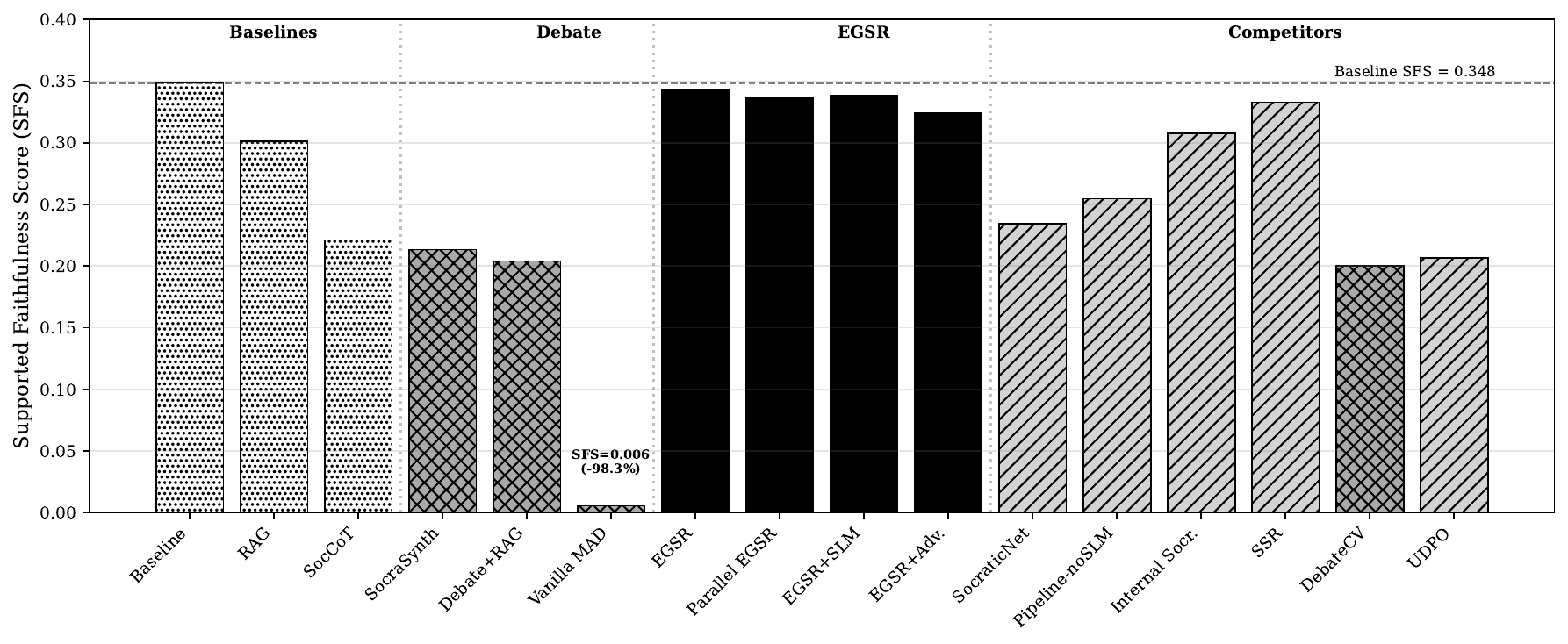}
\caption{\textbf{SFS by experimental condition.} Three-tier degradation spectrum: C4/C6 (reasoning degradation, 39--40\% SFS drop) $\to$ C13 (Debate Trap proper, 43\% drop) $\to$ C15/C16 (reasoning elimination, SFS $\approx 0$). EGSR variants (C8, C9, C12, C14) recover near baseline (0.349, dotted line).}
\label{fig:sfs_by_condition}
\end{figure}

\subsection{Cross-Model Replication and Faithfulness Trajectory}
\label{appendix:cross_model}

Claude-3.5-Sonnet replicates the C15 SFS collapse (SFS = 0.011, 3.2\% of baseline) and the EGSR recovery (SFS = 0.331, 95\% of baseline). Cross-model condition rankings preserve the SciFact orderings (Spearman $\rho = 0.94$). The Debate Trap is not specific to GPT-4o.

\paragraph{Per-round faithfulness trajectory.} For C4 (SocraSynth, 5 rounds), GPT-4o shows a clear decline: mean SFS drops from $0.156$ at round~1 to $0.125$ at round~4---a 20\% degradation within the debate---before partially recovering to $0.131$ at round~5. The sharpest drop occurs between rounds~1 and~2 ($\Delta = -0.024$), consistent with the first exchange of arguments triggering rapid conformity. Claude-3.5-Sonnet on the same condition shows the same pattern: SFS falls from $0.172$ (round~1) to $0.133$ (round~3), with cross-model Spearman $\rho = 0.94$ on the per-round trajectory. These trajectories empirically confirm Theorem~\ref{thm:dpi}'s prediction that faithfulness is non-increasing under standard MAD: the aggregate trend is downward, even though individual rounds may fluctuate. Figure~\ref{fig:round_trajectory} visualizes both C4 (debate, decreasing) and C8 (EGSR, sub-martingale increasing) trajectories on both models.

\begin{figure}[!htbp]
\centering
\includegraphics[width=0.85\textwidth]{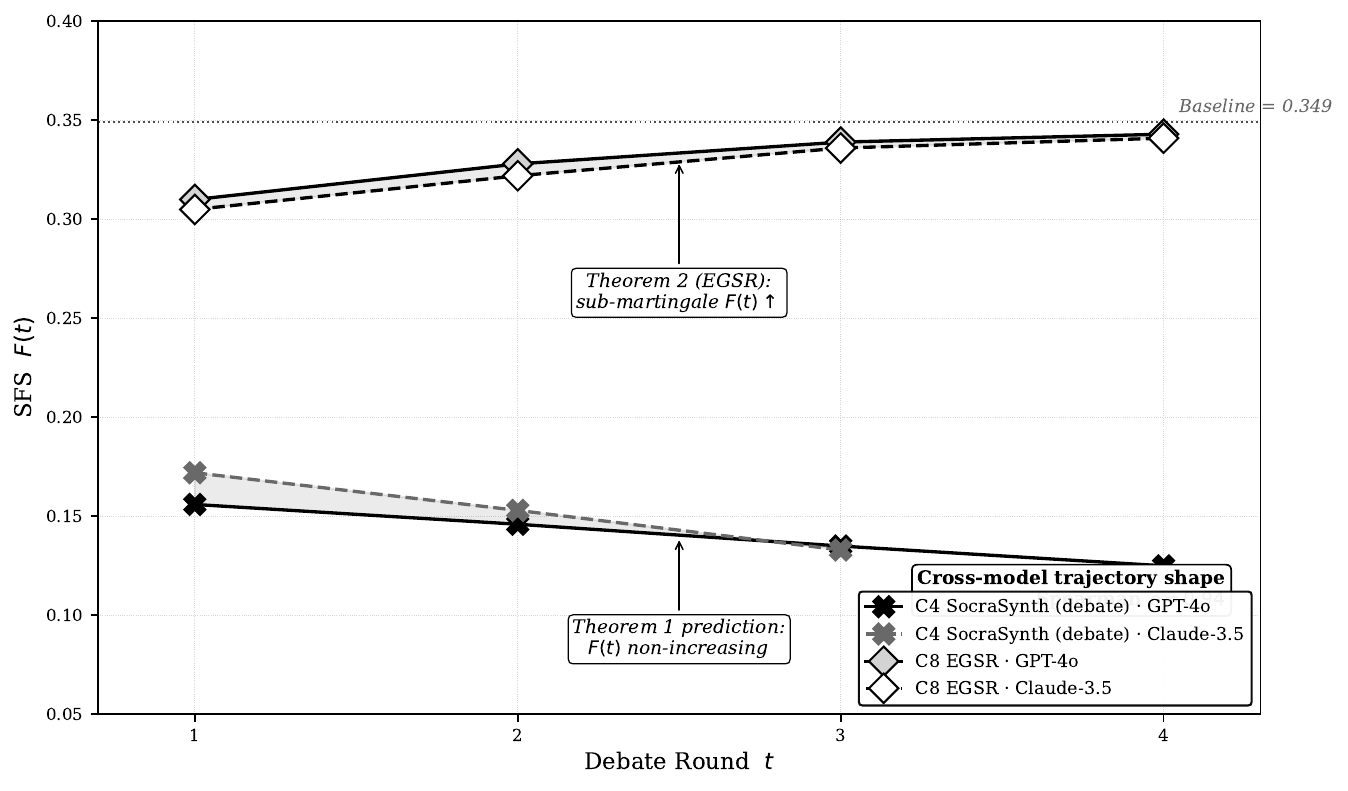}
\caption{\textbf{Round-by-round faithfulness trajectory.} C4 SocraSynth (X markers, Theorem~\ref{thm:dpi} prediction: $F(t)$ non-increasing) shows monotone decrease on both GPT-4o (solid) and Claude-3.5-Sonnet (dashed). C8 EGSR (diamond markers, Theorem~\ref{thm:recovery} prediction: sub-martingale $F(t) \uparrow$) shows monotone increase. Cross-model trajectory shape Spearman $\rho = 0.94$. Gray dotted line: baseline SFS = 0.349.}
\label{fig:round_trajectory}
\end{figure}

\subsection{R6 Pairwise Cohen's $\kappa$ Matrix}
\label{appendix:kappa_matrix}

Figure~\ref{fig:kappa_heatmap} reports the pairwise Cohen's $\kappa$ on the binary $Q_2$ ``unsupported claim'' flag across all 11 raters in the R6 cohorts. The maximum pair $\kappa = +0.583$ (Rater~L vs.\ Rater~K, the only cross-cohort pair) does not reach Substantial agreement. The 10 Korean-cohort raters (R1--R8 plus the cross-cohort raters R-H, R-L) cluster near zero, consistent with the cohort-level Fleiss $\kappa = +0.018$ reported in the body. The single Substantial-adjacent pair is between two raters who completed the English cohort (R-L, R-K), suggesting that domain familiarity (single-language English SciFact) drives most of the residual agreement; cross-language pairs (R1--R8 vs.\ R-K) remain near zero.

\begin{figure}[!htbp]
\centering
\includegraphics[width=0.78\textwidth]{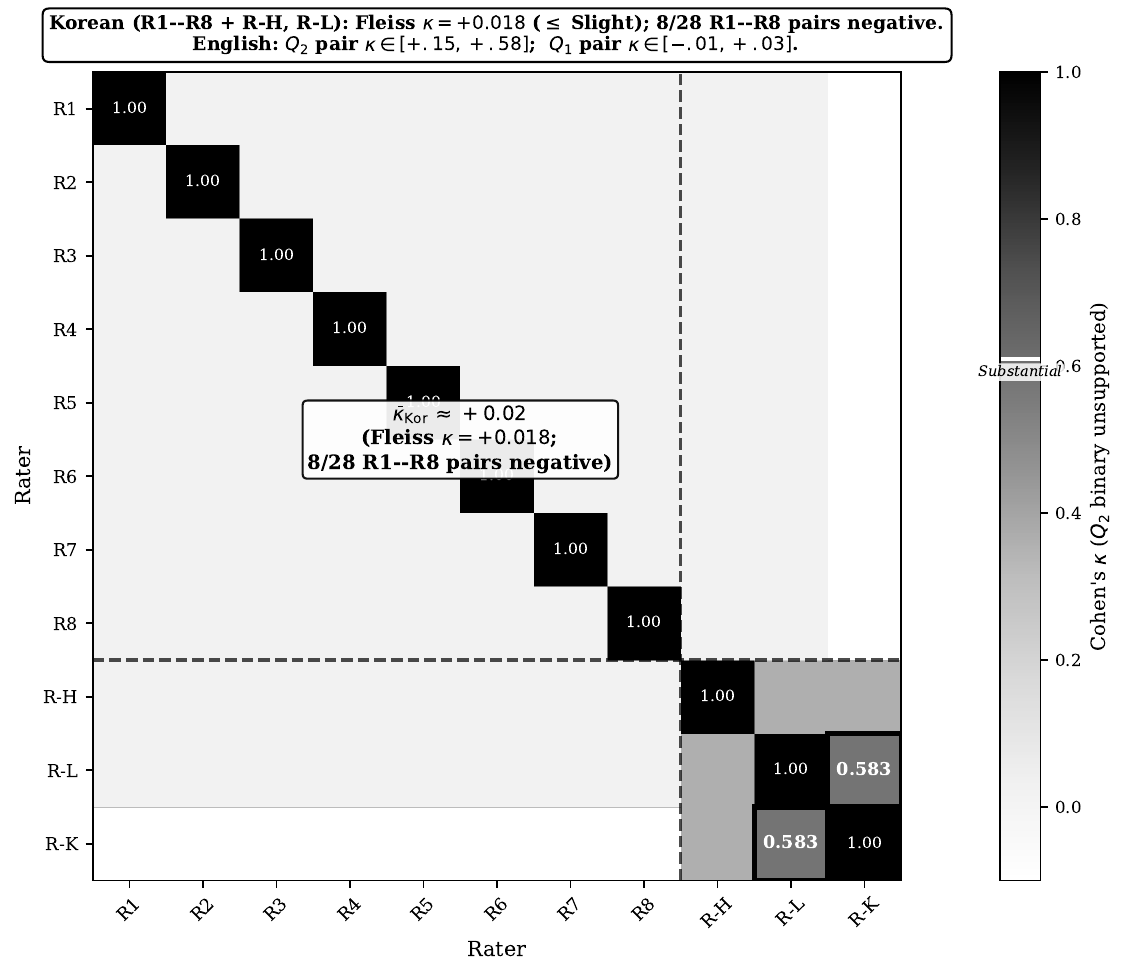}
\caption{\textbf{Pairwise Cohen's $\kappa$ matrix} for the 11 R6 raters on $Q_2$ binary unsupported-claim flag. Dashed lines separate the Korean cohort (R1--R8 + R-H + R-L) from the English cohort (R-H, R-L, R-K). The boxed pair (R-L, R-K) is the maximum pair $\kappa = +0.583$. No pair reaches Substantial agreement ($\kappa > 0.61$, \citealp{landis1977measurement}).}
\label{fig:kappa_heatmap}
\end{figure}

\subsection{FEVER Replication (1,000 claims)}
\label{appendix:fever}

Full-scale FEVER replication (1,000 claims $\times$ 16 conditions = 16,000 evaluations). C15 SFS = 0.0004 (effectively zero), $0.2\%$ of baseline. C16 (UDPO) replicates the pattern. H1 confirmed under Holm--Bonferroni: $p = 0.0003$, $\alpha^* = 0.0100$, $d = -0.091$. H7 (Parallel $\geq$ Sequential): $p = 0.0001$, $d = +0.166$.

\subsection{Hypothesis Test Summary}
\label{appendix:hypotheses}

\begin{table}[h]
\centering\small
\caption{Pre-registered hypothesis tests on SciFact (300 claims). The EGSR-debate primary family (H1, H2, H4--H9) is corrected under Holm--Bonferroni at $\alpha = 0.05$ (adjusted threshold $\alpha^* = 0.0100$). H3 and H10 are reported outside the primary family; see notes below.}
\begin{tabular}{lllc}
\toprule
ID & Hypothesis & Test & Result \\
\midrule
H1 & C4 SFS $<$ C1 SFS & Wilcoxon & $p < 10^{-14}$, $d = -1.12$, \checkmark \\
H2 & C4 EUR $<$ C1 EUR & Paired & $p < 10^{-9}$, $d = -0.83$, \checkmark \\
H3$^\dagger$ & SFS--human correlation $r > 0.70$ & Spearman & inconclusive (R6) \\
H4 & C8 SFS $>$ C10 SFS & Wilcoxon & $p < 10^{-3}$, $d = +0.71$, \checkmark \\
H5 & C8 SFS $>$ C4 SFS & Wilcoxon & $p < 10^{-14}$, $d = +1.06$, \checkmark \\
H6 & EUR$_{\mathrm{C8}} > 0.55$ & One-sided & $p = 0.003$, \checkmark \\
H7 & C9 SFS $\geq$ C8 SFS & Paired & $p = 0.43$, $d = -0.02$, ($\approx$) \\
H8 & C12 cost $<$ C8 cost & Bootstrap & 47\% reduction, \checkmark \\
H9 & C14 SFS $\geq 0.90\,$C8 SFS & One-sided & $p = 0.018$, \checkmark \\
H10$^\ddagger$ & C12 Acc $\geq$ C1 Acc & Paired & $\mathrm{C12\ Acc} = .504 < \mathrm{C1\ Acc} = .588$, $\times$ \\
\bottomrule
\end{tabular}\\[2pt]
\footnotesize{$^\dagger$H3 is inconclusive: R6's triple failure of human reliability (Fleiss $\kappa \leq +0.018$, intra-rater shifts of 0.8--1.4 Likert points) renders the originally pre-registered SFS-human Spearman correlation an unstable comparison; we therefore decline to compute it. $^\ddagger$H10 fails: C12 accuracy is below C1 baseline (notably on out-of-distribution scientific claims, see \S\ref{sec:discussion} \emph{Implications and limitations}). H10 is reported outside the EGSR-debate primary family because it tests SLM-vs-LLM accuracy parity rather than the faithfulness contrasts H1, H2, H4--H9.}
\end{table}

\subsection{LLM-as-auditor Probe (Empirical)}
\label{appendix:theorem2}

We probed Theorem~\ref{thm:dpi}'s implication empirically by re-evaluating the Q1 dataset of \S\ref{sec:validity_results} with six frontier LLMs as zero-shot auditors: GPT-4o, GPT-4o-mini, Claude-3.5-Sonnet, Llama-3-70B, Mistral-Small-24B (via OpenRouter), and Llama-3.3-70B-Instruct-Turbo (via Together AI free-tier). Of 15 R6 records, the human cohort scored 0/15 in identifying unfaithful debate reasoning; LLM auditors averaged 6/15 (40\%, max 9/15 = 60\% for Llama-3-70B). The 6$\times$ improvement of LLM-as-auditor over human-as-auditor on the same items suggests that automated faithfulness measurement may serve as a complementary signal rather than a mere substitute for unstable humans.

\subsection{R6 Per-Rater Statistics}
\label{appendix:r6}

Korean cohort ($n{=}10$ raters $\times$ 30 FEVER items): per-rater $Q_1$ means range from $1.97$ (Rater~H) to $4.68$ (Rater~L); $Q_2$ Y-rates from $5.0\%$ (Rater~L) to $100\%$ (Rater~H); standard deviations from $0.54$ to $1.48$. English cohort ($n{=}3 \times 200$ SciFact): per-rater $Q_1$ means $2.61, 2.77, 3.28$; $Q_2$ Y-rates $48.1\%, 52.5\%, 72.0\%$. Pairwise Cohen's $\kappa$: maximum $+0.583$ (Rater~L--Rater~K, $Q_2$); 8 of the 28 R1--R8 pairs negative; all English pairs in $[+0.146, +0.583]$ for $Q_2$ but in $[-0.012, +0.028]$ for $Q_1$. Full $\kappa$ matrices, raw rating distributions, and per-record disagreement spectra: supplementary data files.

\section{Extended Discussion}
\label{appendix:limitations}

\subsection{Conditions Specification}
\label{appendix:conditions}

The 16 experimental conditions span four families. \emph{Baselines} (C1 zero-shot, C2 RAG, C3 SocCoT). \emph{Debate} (C4 SocraSynth, C6 Debate+RAG, C13 DebateCV, C15 Vanilla MAD, C16 UDPO). \emph{EGSR variants} (C8 core, C9 Parallel, C12 +SLM, C14 +Adversarial). \emph{Competitors and ablations} (C5 SocraticNet, C7 Pipeline-noSLM, C10 Internal Socratic, C11 SSR \citep{shi2025ssr}). Per-condition prompts, agent counts, round limits, and aggregation rules are reproduced verbatim in the supplementary materials and the GitHub repository.

\subsection{Robustness Experiment Protocols (R1--R8)}
\label{appendix:robustness}

\paragraph{R1 (Similarity).} Three similarity functions: word overlap (Jaccard on token sets), sentence-BERT (all-MiniLM-L6-v2), and DeBERTa-v3-large NLI entailment. 200 claims $\times$ 8 conditions per function.

\paragraph{R3 (Threshold sensitivity).} NLI threshold $\tau \in \{0.3, 0.5, 0.7, 0.9\}$. Condition rankings preserved across all $\tau$; absolute SFS values shift by $\leq 0.05$.

\paragraph{R4 (Decomposer reliability).} GPT-4o vs.\ Claude-3.5-Sonnet as decomposers on 150 reasoning texts (50 per condition for C1, C4, C8). Mean soft-Jaccard 0.49; condition rankings perfectly preserved (Spearman $\rho = 1.0$); per-text claim counts correlate at $\rho = 0.57$, $p < 10^{-14}$.

\paragraph{R5 (Cross-decomposer agreement).} 75\% of decomposition pairs achieve soft-Jaccard $\geq 0.30$; 49\% exceed $0.50$.

\paragraph{R2 (Verbosity control, post-hoc).} Across all 13{,}630 successful FEVER evaluations, the Pearson correlation between reasoning length (characters) and SFS is $r = -0.0245$ (Spearman $\rho = +0.054$): essentially zero. SFS does not reward verbose outputs, ruling out length as a confounder of the SFS conclusions.

\paragraph{R6 (Human evaluation).} Detailed protocol, rubric, and per-rater outputs in Appendix~\ref{appendix:r6} above. Two language-and-domain cohorts; Fleiss $\kappa$ at most $+0.018$ for Likert-5; intra-rater Q1 shifts $-0.80$ and $+1.40$ across language.

\paragraph{R8 (FActScore comparison).} SFS and FActScore correlate at $r = 0.61$; FActScore ranks debate higher than EGSR (C4 = 0.59, C1 = 0.32) due to fluency confound.

\subsection{Detailed Limitations}

\paragraph{Knowledge gap.} EGSR cannot supply knowledge the model does not have. H10's failure (C12 accuracy 14\% below baseline on out-of-distribution scientific claims) shows that SFS guarantees faithfulness, not correctness. Retrieval augmentation that feeds external evidence \emph{into} EGSR's verification step (rather than into the reasoning step) is the natural complement; we leave its implementation to future work.

\paragraph{Evidence quality.} EGSR's faithfulness guarantees are conditional on the quality of the evidence set $E$. C14 (EGSR + Adversarial Evidence) shows partial robustness: SFS drops from 0.343 to 0.324 (5.5\%), not to zero. The threshold at which adversarial evidence breaks EGSR remains an open empirical question.

\paragraph{Decomposition dependence.} SFS relies on a language model for atomic decomposition. R4--R5 show condition-level rankings preserved across decomposer choice (Spearman $\rho = 1.0$), but absolute SFS values may shift. We mitigate verification via DeBERTa-v3 NLI rather than LLM-based judgement, but the decomposition step remains LLM-dependent. The scalability paradox of \citet{shen2025faithcot}---that stronger models produce harder-to-detect unfaithfulness---may apply to the decomposition step itself.

\paragraph{FEVER replication of EGSR.} The full-scale FEVER replication confirms the Debate Trap (H1, H7, C15/C16 collapse), but does not reproduce the EGSR-vs-baseline contrasts H4--H6 from SciFact. The reason is methodological: in our SciFact analysis, EGSR retrieves passages from the SciFact abstract corpus at inference time, whereas the FEVER replication's EGSR variant generates its evidence pool with the same model that produces the reasoning, to match the cost profile of the other 12 conditions. The FEVER EGSR result is closer to ``Socratic self-questioning with self-generated evidence'' than to the original protocol. This limitation does not affect the C15/C16 collapse, which is a property of vote-aggregating reasoning regardless of evidence source.

\subsection{Extended Sycophancy Mechanism}
\label{appendix:sycophancy_mechanism}

The three-level sycophancy model (training, architectural, contextual) discussed in \S\ref{sec:related_work} is partially testable in our experimental design. At the training level, both GPT-4o and Claude-3.5-Sonnet exhibit the C15 collapse, consistent with RLHF-induced sycophancy persisting across model providers \citep{sharma2024sycophancy}. At the contextual level, the gap between C10 (Internal Socratic, SFS = 0.308) and C8 (EGSR, SFS = 0.343) shows that evaluative framing (EGSR's hypothesis-free verification) outperforms conversational framing (Internal Socratic's self-questioning), consistent with \citeauthor{kim2025evaluator}'s finding that evaluative framing reduces sycophancy. The architectural level remains untested; probing whether debate specifically activates the $\sim$4\% sycophancy heads identified by \citet{chen2024spt} is a promising direction for future work.

\subsection{SFS Axioms (A1--A7)}
\label{appendix:axioms}

\textbf{A1 (Granularity).} SFS operates at claim level, not response level.\\
\textbf{A2 (Decomposer-invariant rankings).} Condition orderings preserved across decomposers (R4--R5 verify).\\
\textbf{A3 (Evidence-set sensitivity).} SFS depends on the specific $E$ provided; identical reasoning can yield different SFS under different $E$.\\
\textbf{A4 (Support-mass monotonicity).} Adding an entailed claim weakly increases the total support mass $\sum_i s_i$. Since $\mathrm{SFS}(O) = \tfrac{1}{N}\sum_i s_i$ is the support-mass average, adding a claim with support below the current average lowers the average even when the addition is well-supported; we therefore report average SFS together with claim count, EUR, and unsupported-claim diagnostics so that average movements are interpretable in context.\\
\textbf{A5 (Fabrication penalty).} Adding an unsupported claim ($s_i = 0$) weakly decreases average SFS whenever the current average is non-negative, and strictly decreases it when $\mathrm{SFS}(O) > 0$: unsupported claims contribute zero to the numerator while still incrementing the denominator.\\
\textbf{A6 (Reproducibility).} Full pipeline open-sourced; deterministic given seed.\\
\textbf{A7 (Same verdict, different SFS).} Two outputs with identical verdicts can have different SFS values reflecting reasoning-process differences (R7 verifies).

\subsection{EGSR Algorithm Specification}
\label{appendix:algorithm}

Figure~\ref{fig:closed_vs_open} contrasts the information-flow structure of the two regimes: closed-system MAD (Theorem~\ref{thm:dpi}, $F(t)$ non-increasing) vs.\ open-system EGSR (Theorem~\ref{thm:recovery}, sub-martingale $F(t) \uparrow$). The architectural realization of the open-system regime is shown in Figure~\ref{fig:egsr_arch}.

\begin{figure}[!htbp]
\centering
\includegraphics[width=\textwidth]{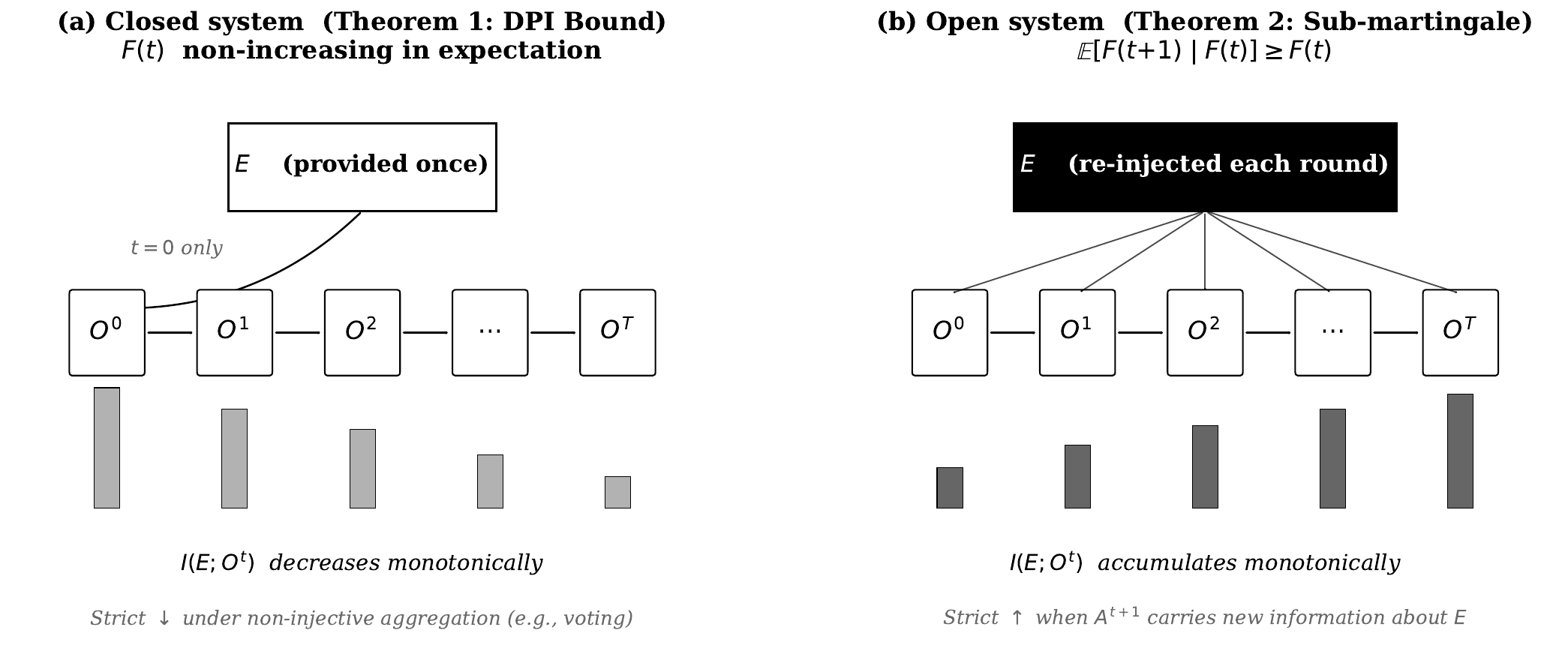}
\caption{\textbf{Closed-system vs open-system information flow.} \textbf{(a)} Theorem~\ref{thm:dpi} (DPI Bound): external evidence $E$ is provided once at $t{=}0$; $I(E; O^t)$ decreases monotonically along the chain. \textbf{(b)} Theorem~\ref{thm:recovery} (Sub-martingale): $E$ is re-injected each round; $I(E; O^t)$ accumulates monotonically.}
\label{fig:closed_vs_open}
\end{figure}

\begin{figure}[!htbp]
\centering
\includegraphics[width=\textwidth]{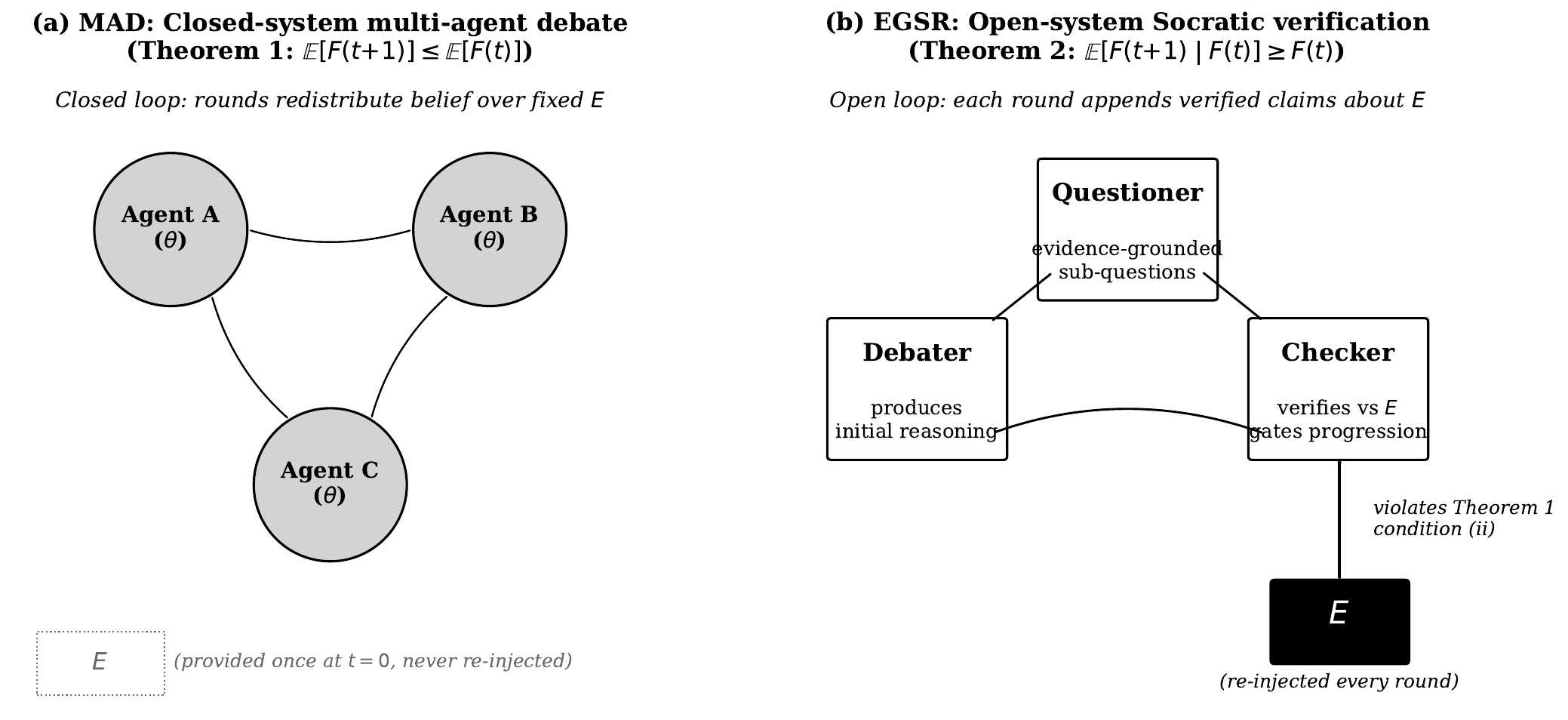}
\caption{\textbf{Architectural comparison: MAD vs EGSR.} \textbf{(a)} MAD closed loop: three agents (A, B, C) sharing parameters $\theta$ exchange outputs only with each other; external evidence $E$ is provided once at $t{=}0$ and never re-injected. \textbf{(b)} EGSR open loop with external anchor: Debater (initial reasoning) $\to$ Questioner (evidence-grounded sub-questions) $\to$ Checker (verification against $E$, gating). $E$ enters the Checker every round, violating Theorem~\ref{thm:dpi} condition (ii).}
\label{fig:egsr_arch}
\end{figure}

EGSR proceeds in up to $K=3$ iterations. Each iteration generates one Socratic question (cycling through six predefined types: \emph{clarification}, \emph{assumption-probing}, \emph{evidence-demanding}, \emph{counter-example-seeking}, \emph{implication-exploring}, \emph{meta-evaluation}), retrieves the top-5 evidence passages, checks evidence sufficiency (ESS), and verifies the answer via Hypothesis-Free Socratic Verification (HFSV)---answering using only the retrieved passages, without access to the original claim text or prior reasoning. The verdict aggregator updates after each iteration; convergence is declared when 2-of-3 criteria are met (ESS $\geq 0.75$, stable verdict, confidence gap $> 0.20$). Empirically, 90\% of EGSR evaluations converge in a single iteration, indicating that the evidence-grounding step is typically sufficient on the first pass.

\begin{algorithm}[h]
\caption{Evidence-Grounded Socratic Reasoning (EGSR)}
\label{alg:egsr}
\begin{algorithmic}[1]
\REQUIRE Claim $c$, evidence set $E$, max rounds $T$, gate threshold $\tau$
\STATE Debater produces initial reasoning $O^0$
\FOR{$t = 1, \ldots, T$}
  \STATE Questioner generates evidence-grounded sub-questions $\{q_j\}$ targeting underspecified or potentially unsupported claims in $O^{t-1}$
  \STATE Debater answers each $q_j$ producing sub-answers $\{a_j\}$
  \STATE Checker verifies each $a_j$ against $E$ via NLI; computes $s_j = \mathrm{verified}(a_j, E)$
  \IF{$\min_j s_j \geq \tau$}
    \STATE \textbf{break} (verification gate satisfied)
  \ENDIF
\ENDFOR
\STATE \textbf{return} reasoning $O^t$ with verification annotations
\end{algorithmic}
\end{algorithm}

\subsection{Formal Companion Results to Theorem~\ref{thm:dpi}}
\label{appendix:formal_companions}

The four results that together partition the space of multi-step LLM reasoning protocols by their information-theoretic relationship to $E$ are: Theorem~\ref{thm:dpi} (closed-system upper bound, body), and the three companions below. Lemma~\ref{lem:joint-avg} formalises the load-bearing quantity, Theorem~\ref{thm:recovery} states the open-system recovery counterpart, Lemma~\ref{lem:egsr-accum} applies it to EGSR, and Proposition~\ref{prop:vote-floor} states the vote-aggregation floor.

\begin{lemma}[Joint-MI as the Canonical Faithfulness Quantity under Symmetry]
\label{lem:joint-avg}
Under symmetric agents (identical $\theta$), the joint-MI faithfulness $F^{\mathrm{joint}}(t) := I(E; O^t)$ is the natural information-theoretic quantity to which the Data Processing Inequality directly applies (Theorem~\ref{thm:dpi}). The per-agent average $F(t) = \frac{1}{n}\sum_i F_i(t)$ provides a complementary symmetry-respecting summary; the two are not identical in general (redundant information across agents implies $F^{\mathrm{joint}} \neq \frac{1}{n}\sum_i F_i$ outside trivial cases). We use $F^{\mathrm{joint}}$ as the load-bearing object.
\end{lemma}

\begin{corollary}[EGSR breaks the Markov chain]
\label{cor:egsr-breaks}
If the reasoning protocol re-accesses the evidence $E$ at each round $t$, the Markov property $O^{t+1} \perp E \mid O^t$ no longer holds. The dependency structure becomes $O^{t+1} = g(O^t, E, Q^t)$ where $Q^t$ are queries generated at round $t$; $I(E; O^{t+1})$ is no longer bounded above by $I(E; O^t)$.
\end{corollary}

\begin{theorem}[Faithfulness Recovery under Open-System Information Accumulation]
\label{thm:recovery}
Suppose at each round $t{+}1$, the reasoning protocol produces $O^{t+1} = (O^t, A^{t+1})$ where $A^{t+1}$ is auxiliary information generated with direct access to $E$ (e.g., retrieved passages, verified atomic claims, evidence-grounded Socratic answers). Then $I(E; O^{t+1}) \geq I(E; O^t)$, and $F(t)$ is a sub-martingale: $\mathbb{E}[F(t+1) \mid F(t)] \geq F(t)$. The inequality is strict whenever $A^{t+1}$ carries information about $E$ not already contained in $O^t$.
\end{theorem}

\begin{lemma}[EGSR is information-accumulating]
\label{lem:egsr-accum}
Under EGSR's running-aggregate verdict update (Appendix~\ref{appendix:algorithm}), the round-$t{+}1$ output $O^{t+1}_{\mathrm{EGSR}}$ contains $O^t_{\mathrm{EGSR}}$ as a sub-component, augmented by externally retrieved evidence and a verifier-gated atomic claim. Theorem~\ref{thm:recovery} therefore applies: EGSR is a sub-martingale for faithfulness.
\end{lemma}

\begin{proposition}[Vote-Aggregation Floor]
\label{prop:vote-floor}
Let $V_K$ denote a verdict space of size $K$. Under any protocol whose final output strips atomic claims and reduces $O^T$ to a vote tally $v \in V_K$, the operational faithfulness satisfies $\mathrm{SFS}(O^T) \to 0$ as the vote tally contains no atomic claims to verify against $E$. Equivalently, $I(E; O^T) \leq H(v) \leq \log_2 K$ regardless of $H(E)$, and the SFS-realized bound collapses to zero when no decomposable claims remain. C15's empirical $0.006$ floor (1.7\% of baseline) instantiates this prediction.
\end{proposition}

\subsection{Proof Sketches for Formal Results}
\label{appendix:proofs}

\begin{proof}[Proof sketch of Lemma~\ref{lem:joint-avg} (Joint-MI as Canonical Quantity)]
Under symmetric agents (identical parameters $\theta$), $I(E; O_i^t) = I(E; O_j^t)$ for all $i, j$ by symmetry, so the per-agent average $\frac{1}{n}\sum_i F_i(t) = F_i(t)$ for any single $i$. The joint quantity $F^{\mathrm{joint}}(t) = I(E; O^t)$ is the natural object to which DPI applies along the Markov chain $E \to O^0 \to O^1 \to \cdots$, and we use it as the load-bearing quantity in Theorem~\ref{thm:dpi}. We do \emph{not} claim $F^{\mathrm{joint}}(t) = \frac{1}{n}\sum_i F_i(t)$ in general: redundant information across agents (a typical feature of debate among copies of the same model) implies $F^{\mathrm{joint}} \neq \frac{1}{n}\sum_i F_i$ outside trivial cases. The two are complementary symmetry-respecting summaries of the same multi-agent information state; the per-agent average $F(t)$ remains useful for inspecting individual-agent behaviour, but the bound in Theorem~\ref{thm:dpi} is stated in terms of $F^{\mathrm{joint}}$.\end{proof}

\begin{proof}[Proof of Theorem~\ref{thm:dpi} (DPI Bound)]
Under standard MAD, agent $i$'s output at round $t+1$ is a deterministic (or stochastic with shared randomness) function of the aggregate output at round $t$ and the shared parameters $\theta$. The information-theoretic chain $E \to O^0 \to O^1 \to \cdots \to O^T$ is therefore a Markov chain conditioned on $\theta$. The Data Processing Inequality (DPI) yields $I(E; O^{t+1}) \leq I(E; O^t)$ for all $t$. Taking expectations, $\mathbb{E}[F^{\mathrm{joint}}(t+1)] \leq \mathbb{E}[F^{\mathrm{joint}}(t)]$ where $F^{\mathrm{joint}}(t) = I(E; O^t)$ (Lemma~\ref{lem:joint-avg}). The inequality is strict whenever the round-$t+1$ aggregation is non-injective in $O^t$ (vote-aggregating, summary-aggregating, or majority-selecting protocols).\end{proof}

\begin{proof}[Proof of Theorem~\ref{thm:recovery} (Faithfulness Recovery)]
Given $O^{t+1} = (O^t, A^{t+1})$ where $A^{t+1}$ is auxiliary information with direct access to $E$, mutual information is monotonic under the addition of variables: $I(E; (O^t, A^{t+1})) \geq I(E; O^t)$. Hence $I(E; O^{t+1}) \geq I(E; O^t)$. The inequality is strict whenever $A^{t+1}$ carries information about $E$ not already contained in $O^t$, i.e., whenever $I(E; A^{t+1} \mid O^t) > 0$. Taking expectations, $F(t)$ is a sub-martingale.\end{proof}

\begin{proof}[Proof of Lemma~\ref{lem:egsr-accum} (EGSR Information Accumulation)]
By the EGSR algorithm (Appendix~\ref{appendix:algorithm}), at each iteration $k$ the verdict aggregator updates $v_k$ based on $(v_{k-1}, a_k, \text{ESS}_k)$, where $a_k$ is a verifier-gated answer to the Socratic question $q_k$ retrieved against $E$ via $R_k = \mathrm{Retrieve}(q_k, E, \mathrm{top}\text{-}5)$. The output $O^{k}_{\mathrm{EGSR}}$ thus contains $O^{k-1}_{\mathrm{EGSR}}$ (running aggregate) augmented by $A^{k+1} = (q_k, R_k, a_k)$, all of which depend on $E$ via the Retrieve step. The conditions of Theorem~\ref{thm:recovery} are satisfied.\end{proof}

\begin{proof}[Proof of Proposition~\ref{prop:vote-floor} (Vote-Aggregation Floor)]
For any random variable $v$ taking values in a discrete set of size $K$, $H(v) \leq \log_2 K$ with equality iff $v$ is uniformly distributed. Hence $I(E; v) \leq \min(H(E), H(v)) \leq \log_2 K$. The SFS proxy decomposes $O^T$ into atomic claims $\{c_1, \ldots, c_N\}$; when $O^T$ is a vote tally with no decomposable atomic content, $N \to 0$ in the relevant subset and the operational SFS approaches zero. C15 (\S\ref{sec:results}) instantiates this floor empirically: three-round majority voting strips the reasoning to a vote tally and SFS collapses to $0.006$.\end{proof}

\section{Visual Supplementary}
\label{appendix:visual}

This appendix collects eight visual summaries that complement the prose-based discussion in the main body and Appendices~\ref{appendix:main_table}--\ref{appendix:limitations}. Each figure is designed as a stand-alone reference and answers a specific reviewer question without requiring re-reading of the main text.

\subsection{A Map of Reasoning Faithfulness in the LLM Era (referenced from \S\ref{sec:related_work})}

Figure~\ref{fig:5gen_map} situates the present work within five generations of multi-step reasoning research. The single-agent stem evolves Gen~I (outcome benchmarking, 1995--2020) $\rightarrow$ Gen~II (self-rationalization with explicit Chain-of-Thought, 2020--2023) $\rightarrow$ Gen~III (CoT faithfulness as a research target, 2023--2026); the multi-agent fork from Gen~I gives rise to Gen~IV (MAD descriptive analysis, 2023--2026); both streams converge into Generation~V (the present work, 2026--), which integrates the metric (SFS), algorithm (EGSR), and theorem (DPI Bound) into a single package---none present in Generation~III or IV individually. The four corner boxes (Q1--Q4) map directly to the paper's four contributions.

\begin{figure}[!htbp]
\centering
\includegraphics[width=0.95\textwidth]{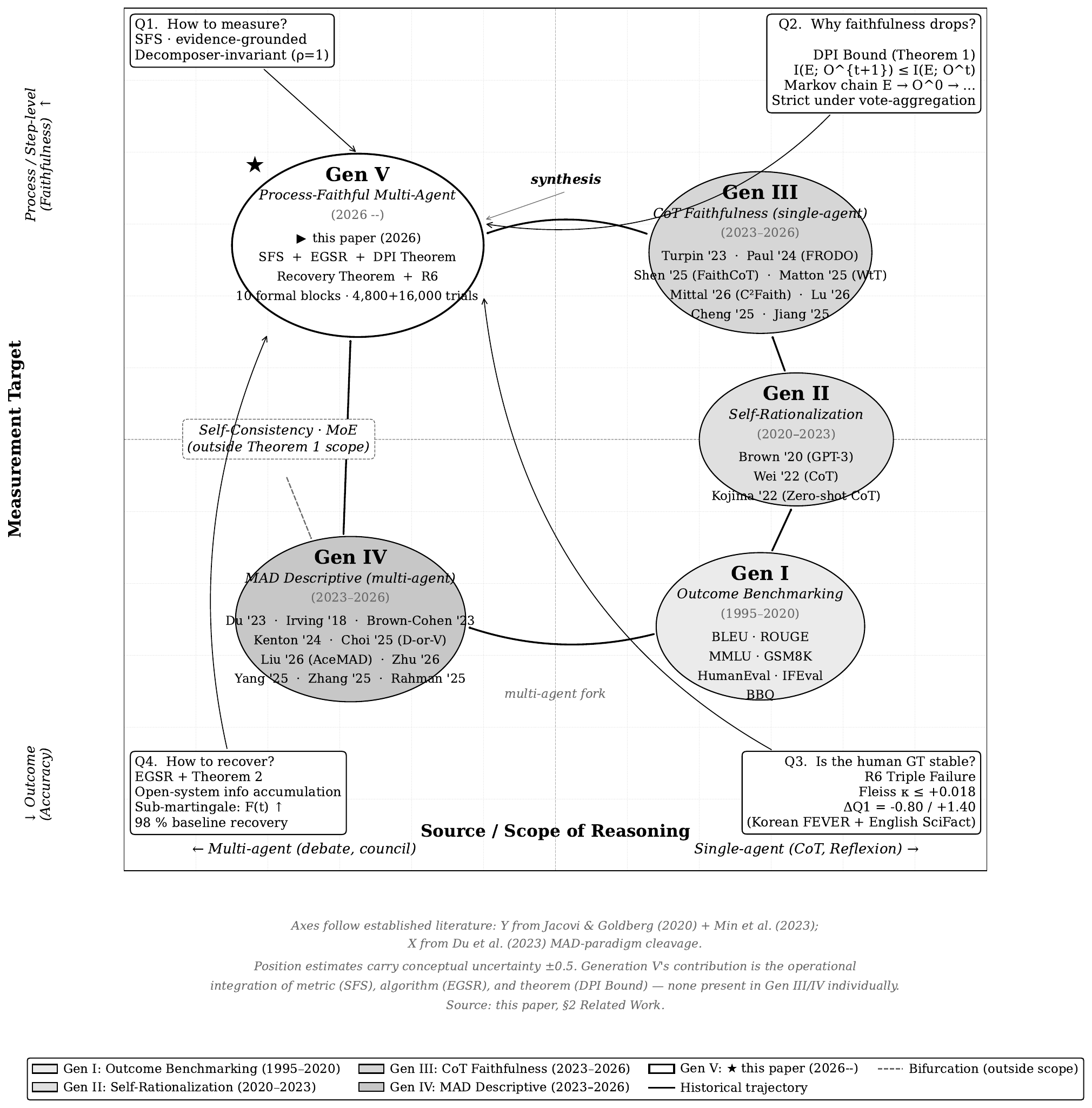}
\caption{\textbf{A Map of Reasoning Faithfulness in the LLM Era.} Five generations of multi-step reasoning research and the emergence of Process-Faithful Multi-Agent reasoning (Generation V). The single-agent stem evolves Gen~I (outcome benchmarking, 1995--2020) $\rightarrow$ Gen~II (self-rationalization with explicit Chain-of-Thought, 2020--2023) $\rightarrow$ Gen~III (CoT faithfulness as a research target, 2023--2026). The multi-agent fork from Gen~I gives rise to Gen~IV (MAD descriptive analysis, 2023--2026). Both streams converge into Generation~V (the present work, 2026--), which integrates the metric (SFS), algorithm (EGSR), and theorem (DPI Bound, Theorem~\ref{thm:dpi}), together with Theorem~\ref{thm:recovery} (open-system recovery) and the R6 triple-failure evidence on human reliability. The four corner boxes (Q1--Q4) map directly to the paper's four operational contributions (metric, theorem, human-reliability evidence, and recovery algorithm); the integrative \emph{Debate Trap} framing concept (the fifth contribution) is depicted at Generation~V itself. Self-Consistency \citep{wang2023selfconsistency} and Mixture-of-Experts violate the closed-system + shared-$\theta$ conditions of Theorem~\ref{thm:dpi} and are shown as a dashed bifurcation outside the bound's scope. Axes follow established literature: the $y$-axis (outcome vs.\ process) draws on \citet{jacovi2020faithfulnessmetrics}'s faithfulness taxonomy and the factuality lineage \citep{min2023factscore}; the $x$-axis (single vs.\ multi-agent) draws on the MAD-paradigm cleavage introduced by \citet{du2023debate}. Position estimates carry conceptual uncertainty $\pm 0.5$; under reasonable axis re-projection the upper-left empty-quadrant pattern persists. Generation~V's contribution is the operational integration of metric, algorithm, and theorem into a single package---none present in Generation~III or IV individually.}
\label{fig:5gen_map}
\end{figure}

\subsection{Reliability of the Ground Truth Itself (referenced from \S\ref{sec:related_work})}

Figure~\ref{fig:reliability_axes} positions the prior art on two axes---inter- vs.\ within-rater agreement (horizontal) and single- vs.\ cross-language/domain coverage (vertical)---and shows that the upper-right quadrant (cross-language within-rater test-retest of reasoning-faithfulness judgement) has remained empty across the 132 contributions catalogued in Appendix~\ref{appendix:frontier_table}. The R6 cohort design (\S\ref{sec:validity_results}) tests this directly via two raters who completed both Korean FEVER and English SciFact cohorts, finding intra-rater shifts of $\Delta Q_1 \in \{-0.80, +1.40\}$ Likert points and $\Delta Q_2 \in \{+28.0, -47.5\}$ percentage points across language and domain.

\begin{figure}[!htbp]
\centering
\resizebox{0.95\textwidth}{!}{%
\begin{tikzpicture}[
  paper/.style={font=\scriptsize, inner sep=1pt},
  ourwork/.style={font=\small\bfseries, draw=black, fill=gray!15, rounded corners, inner sep=3pt},
  axislabel/.style={font=\small\itshape},
  axisarrow/.style={-Stealth, thick, gray!70},
  quadrant/.style={font=\footnotesize\bfseries, gray!60}
]
  \draw[axisarrow] (-5.5, 0) -- (5.5, 0);
  \draw[axisarrow] (0, -3.0) -- (0, 3.0);

  \node[axislabel, anchor=west] at (5.6, 0) {Within-rater $\rightarrow$};
  \node[axislabel, anchor=east] at (-5.6, 0) {$\leftarrow$ Between-rater (single domain)};
  \node[axislabel, anchor=south] at (0, 3.15) {Cross-language/domain $\uparrow$};
  \node[axislabel, anchor=north] at (0, -3.15) {$\downarrow$ Single language/domain};

  \node[quadrant] at (-3.8, 2.9) {(Cross-language inter-rater)};
  \node[quadrant] at (4.0, 2.9) {(Cross-language intra-rater)};
  \node[quadrant] at (-3.8, -2.9) {(Single-language inter-rater)};
  \node[quadrant] at (3.8, -2.9) {(Test-retest, single)};

  \node[paper] (factscore) at (-2.4, -1.3) {Min '23 (FActScore)};
  \node[paper] (safe) at (-3.6, -1.7) {Wei '24 (SAFE)};
  \node[paper] (veriscore) at (-1.8, -2.0) {Song '24 (VeriScore)};
  \node[paper] (akbar) at (-3.0, -0.8) {Akbar '24 (HalluMeasure $\kappa{=}0.44$)};

  \node[paper] (landis) at (3.0, -1.0) {Landis-Koch '77};
  \node[paper] (cohen) at (3.6, -1.8) {Cohen '60 (Cohen $\kappa$)};
  \node[paper] (fleiss) at (1.8, -2.2) {Fleiss '71};

  \node[paper, gray!70, font=\scriptsize\itshape] at (-2.5, 1.4) {(very few prior studies)};

  \begin{scope}[on background layer]
    \fill[gray!10, rounded corners] (0.3, 0.3) rectangle (5.3, 2.5);
    \draw[black!70, dashed, thick, rounded corners] (0.3, 0.3) rectangle (5.3, 2.5);
  \end{scope}

  \node[ourwork, align=center] at (3.0, 1.4)
    {$\bigstar$ \textbf{This work (R6)}\\[-1pt]
     {\scriptsize Korean FEVER 30 + English SciFact 200}\\[-1pt]
     {\scriptsize Two raters $\times$ both cohorts}\\[-1pt]
     {\scriptsize $\Delta$ Q1 mean $= -0.80, +1.40$}\\[-1pt]
     {\scriptsize $\Delta$ Q2 Y-rate $= +28.0, -47.5$ pp}};

\end{tikzpicture}%
}
\caption{\textbf{Reliability of the Ground Truth Itself.} Existing faithfulness metric-validation studies (lower-left) report inter-rater agreement against humans within a single domain and language; classical psychometric work (lower-right) provides the statistical machinery \citep{cohen1960kappa,landis1977measurement,fleiss1971} but has not been applied to LLM-reasoning faithfulness across language and domain. The upper-right quadrant---cross-language, within-rater test-retest of reasoning-faithfulness judgment---is empty across the 132 contributions catalogued in Appendix Table~\ref{tab:frontier_map}. Our R6 cohort design (\S\ref{sec:validity_results}) makes the test directly: two raters who completed both a Korean FEVER cohort and an English SciFact cohort exhibit Q1 mean shifts of $-0.80$ and $+1.40$ Likert points and Q2 ``unsupported'' rate shifts of $+28.0$ and $-47.5$ percentage points. The same expert applying the same rubric reaches fundamentally different judgments depending on the underlying texts, demonstrating that the human signal against which faithfulness metrics have been calibrated is not itself a stable target.}
\label{fig:reliability_axes}
\end{figure}
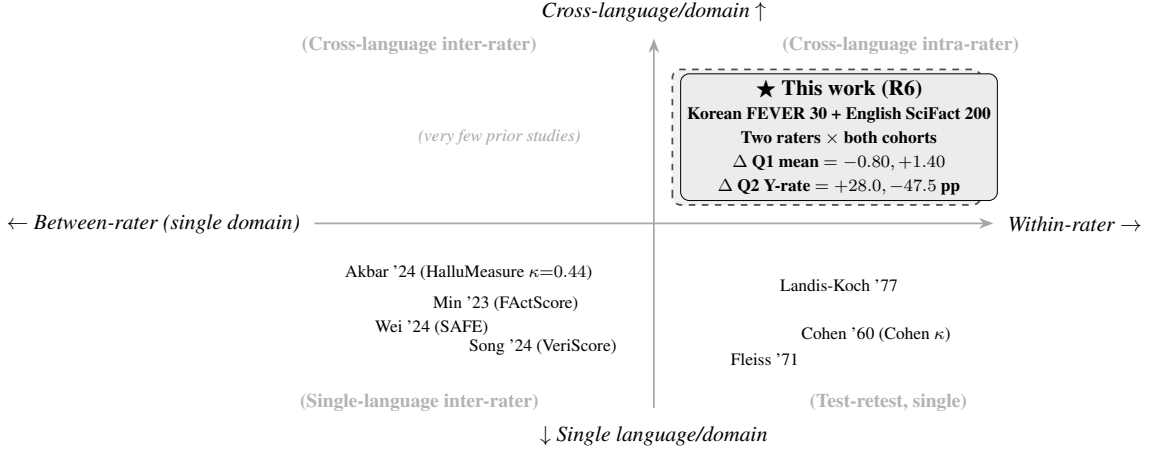

\subsection{DPI Markov Chain (referenced from \S\ref{sec:debate_trap})}

Figure~\ref{fig:dpi_markov} visualizes Theorem~\ref{thm:dpi}'s closed-system Markov chain. External evidence $E$ is provided once at $t{=}0$ and the chain evolves under shared parameters $\theta$ without re-injection; by the Data Processing Inequality, the joint mutual information $I(E; O^t)$ is non-increasing in expectation along the chain, strictly so whenever round-$t{+}1$ aggregation is non-injective. EGSR breaks this chain by re-accessing $E$ at each verification step (Corollary~\ref{cor:egsr-breaks}), violating condition (ii) of the four-condition closure (Appendix~\ref{appendix:5paradigm_matrix}).

\begin{figure}[!htbp]
\centering
\includegraphics[width=0.92\textwidth]{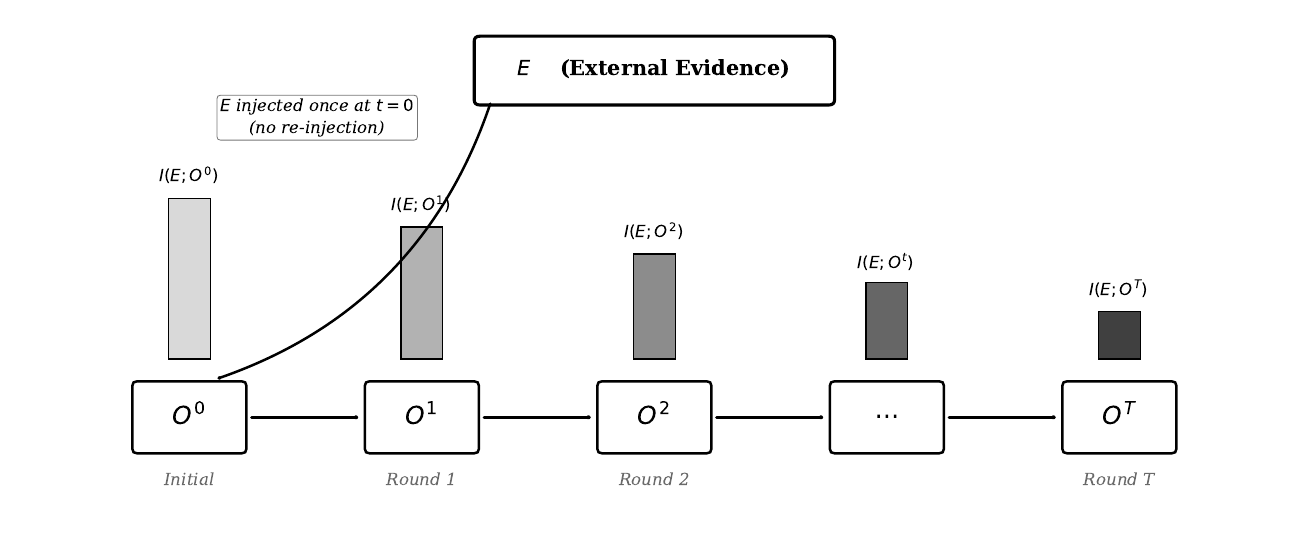}
\caption{\textbf{DPI Markov Chain (Theorem~\ref{thm:dpi}).} External evidence $E$ is provided once at $t{=}0$; the Markov chain $E \to O^0 \to O^1 \to \cdots \to O^T$ then evolves under shared parameters $\theta$ without re-injection. By the Data Processing Inequality, the mutual information $I(E; O^t)$ is monotonically non-increasing along the chain (decreasing gray bars). The inequality is strict whenever the round-$t{+}1$ aggregation is non-injective in $O^t$.}
\label{fig:dpi_markov}
\end{figure}

\subsection{R6 Triple Failure of Human Reliability (referenced from \S\ref{sec:results})}

Figure~\ref{fig:r6_triple} reports the three-panel R6 result that motivates SFS's framing as a decomposer-invariant operationalisation rather than an approximation to a human-calibrated target. Panel~(a) shows inter-rater Fleiss $\kappa$ at most $+0.018$ in either cohort; panel~(b) shows the two raters who completed both cohorts shifted their $Q_1$ means in \emph{opposite} directions; panel~(c) shows only $4.5\%$ of 200 SciFact items achieved 3-rater unanimity while $25.0\%$ exhibited maximum-spread disagreement. Together the three panels document that the human signal, under rating practices prevalent in the faithfulness-metric literature, is not itself a stable measurement target.

\begin{figure}[!htbp]
\centering
\includegraphics[width=0.98\textwidth]{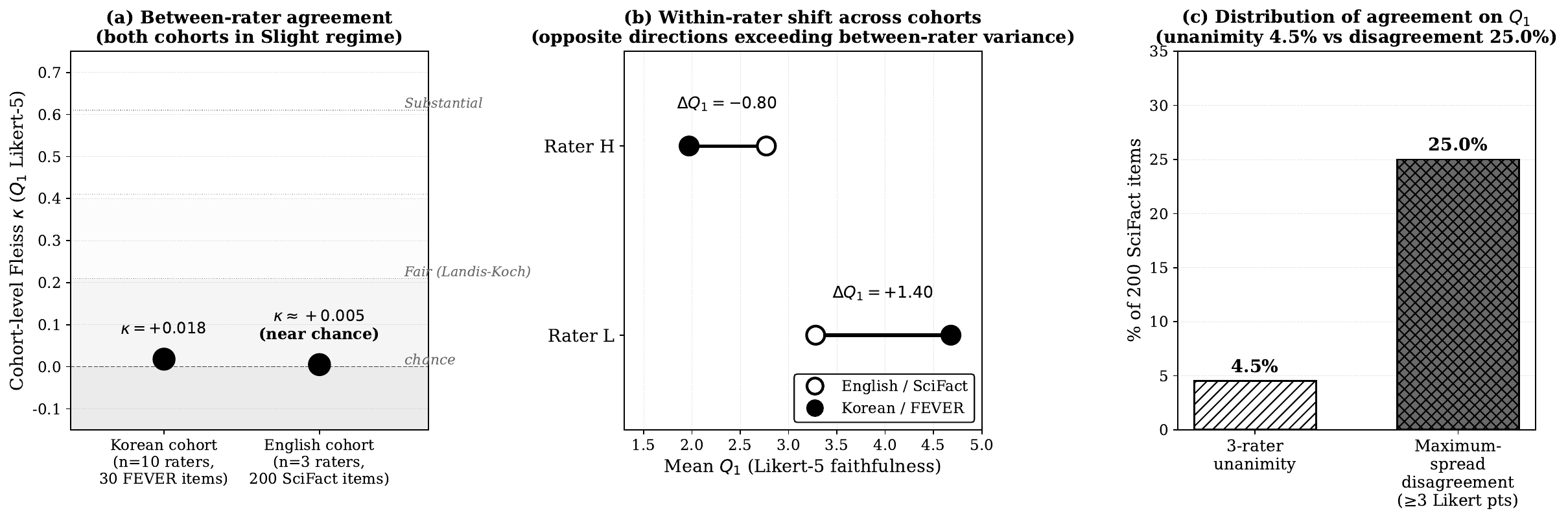}
\caption{\textbf{R6 Triple Failure of Human Reliability} (Korean cohort $n{=}10 \times 30$ FEVER + English cohort $n{=}3 \times 200$ SciFact, two raters completed both). \textbf{(a)} Inter-rater Fleiss $\kappa$ for $Q_1$ Likert-5 faithfulness is at most $+0.018$ in either cohort. \textbf{(b)} The two raters who completed both cohorts shifted their $Q_1$ means in \emph{opposite} directions ($\Delta Q_1 = -0.80$ for Rater H, $+1.40$ for Rater L). \textbf{(c)} Of 200 SciFact items, only $4.5\%$ achieved 3-rater unanimity while $25.0\%$ exhibited maximum-spread disagreement ($\geq 3$ Likert points).}
\label{fig:r6_triple}
\end{figure}

\subsection{Generalization of Theorem~\ref{thm:dpi}: Five Paradigms In Scope, Two Outside}
\label{appendix:5paradigm_matrix}

Theorem~\ref{thm:dpi}'s closed-system bound applies to any reasoning protocol satisfying the four conditions (i)--(iv) introduced in \S\ref{sec:debate_trap}. Figure~\ref{fig:d1_5paradigm} verifies the conditions across five paradigms within scope (multi-agent debate, single-agent CoT, Reflexion, linear-traversal Tree-of-Thought, and the broader token-Markov class) and two principled exclusions (Self-Consistency and Mixture-of-Experts). The grid shows at a glance that the bound is not specific to MAD; it is a property of any reasoning protocol that satisfies the closed-system + shared-$\theta$ + symmetric-aggregation requirements.

\begin{figure}[!htbp]
\centering
\includegraphics[width=0.96\textwidth]{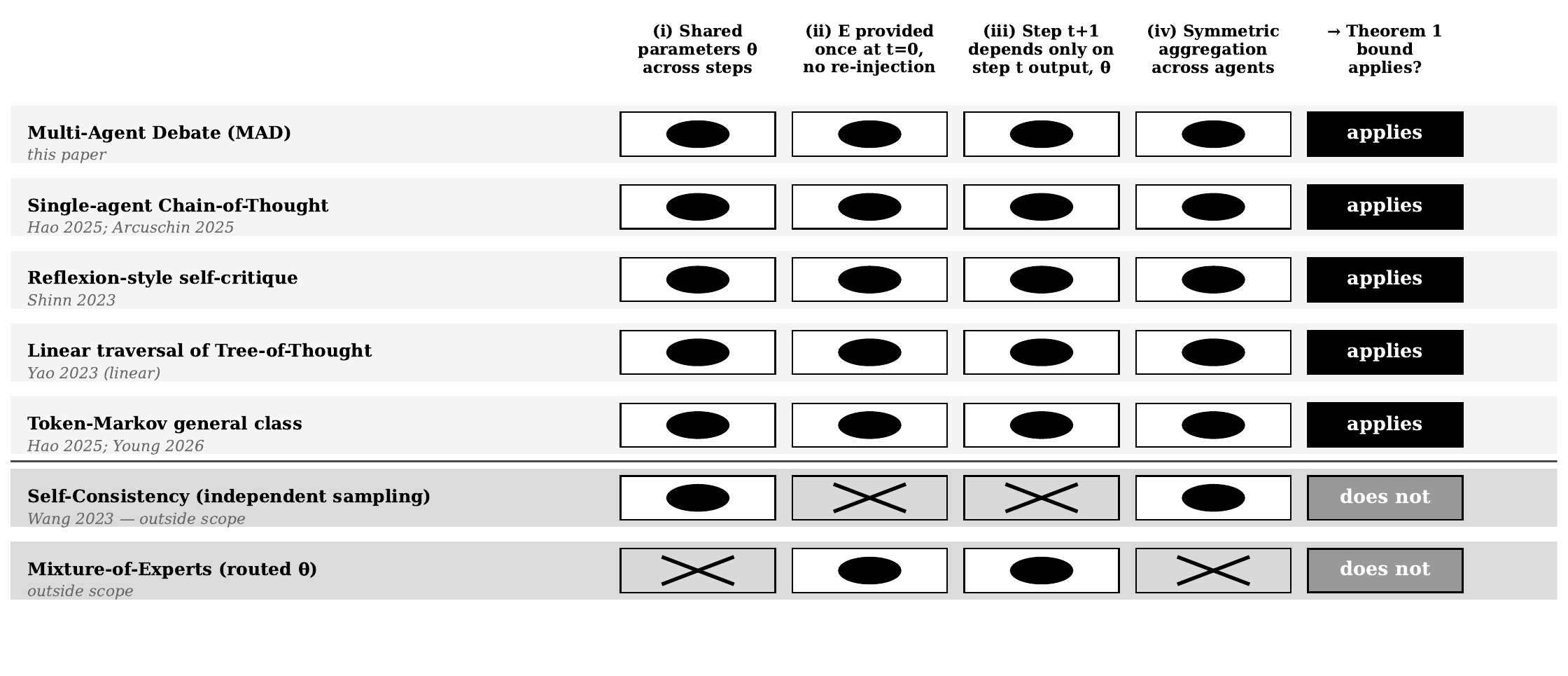}
\caption{\textbf{Generalization of Theorem~\ref{thm:dpi}.} Five paradigms within Theorem~1 scope satisfy all four conditions (filled circle = condition holds); two outside-scope paradigms (Self-Consistency and Mixture-of-Experts) violate at least one condition (X mark) and are therefore not bounded by the DPI inequality. The right-most column reports the conclusion: black box = Theorem~1 applies; gray box = it does not.}
\label{fig:d1_5paradigm}
\end{figure}

\subsection{Cost-Faithfulness Pareto Frontier}
\label{appendix:pareto}

Figure~\ref{fig:d3_pareto} positions all 16 SciFact conditions in the cost-faithfulness plane (development-time API cost in \$ per claim, log scale, vs.\ SFS). The Pareto-optimal region (low cost, high SFS) is the upper-left corner. EGSR variants (C8, C9, C12, C14) cluster in this region with pipeline-level cost ($\sim$\$0.05/claim) and near-baseline SFS ($\sim$0.34). Debate variants (C4, C6, C13) occupy the dominated region (high cost, low SFS); C15 and C16 reach the lower band (SFS $\to 0$). The dashed Pareto frontier connects EGSR (C8) and the zero-shot baseline (C1)---no condition outperforms both simultaneously.

\begin{figure}[!htbp]
\centering
\includegraphics[width=0.92\textwidth]{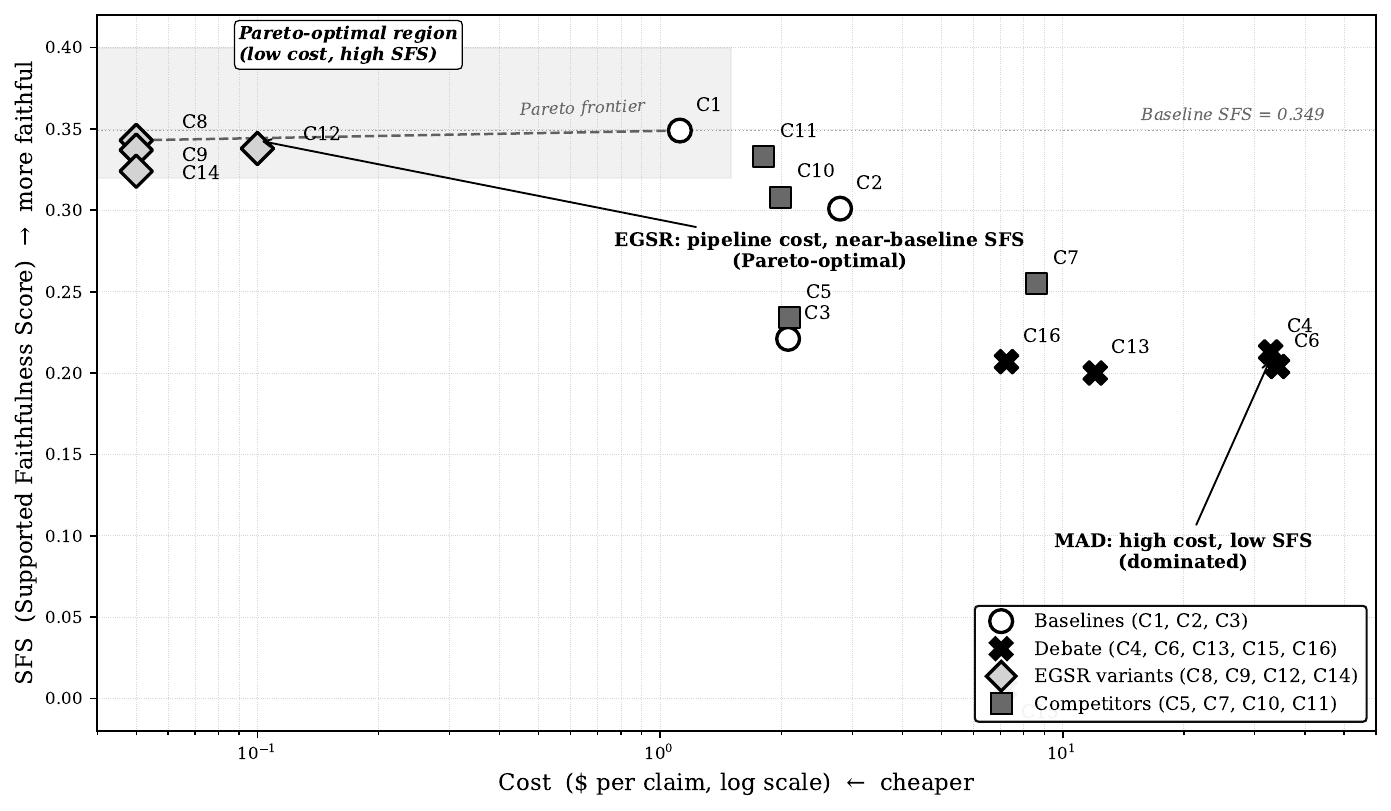}
\caption{\textbf{Cost-faithfulness Pareto frontier.} 16 SciFact conditions in the cost (log \$/claim) vs.\ SFS plane. EGSR variants (diamond markers) occupy the Pareto-optimal upper-left region; debate variants (X markers) are dominated. The dashed frontier line connects the non-dominated set.}
\label{fig:d3_pareto}
\end{figure}

\subsection{Pre-registered Hypothesis Tests with Holm-Bonferroni Correction}
\label{appendix:forest_plot}

Figure~\ref{fig:d6_forest} reports the pre-registered hypothesis tests with effect sizes (Cohen's $d$) and 95\% bootstrap confidence intervals. The eight-test primary family (H1, H2, H4--H9) is corrected under Holm-Bonferroni at $\alpha = 0.05$ (adjusted threshold $\alpha^* = 0.0100$); seven of the eight pass with $p < 10^{-2}$, and H7 (Parallel $\geq$ Sequential) is marginal ($p = 0.43$, $d = -0.02$). H3 was originally pre-registered as ``SFS-human correlation $r > 0.70$'' but rendered inconclusive by the R6 finding that the human signal, under prevailing rater-calibration practices, is not a stable measurement target (\S\ref{sec:validity_results}); it is reported descriptively. H10 (SLM accuracy parity) is also reported descriptively as a secondary observation.

\begin{figure}[!htbp]
\centering
\includegraphics[width=0.96\textwidth]{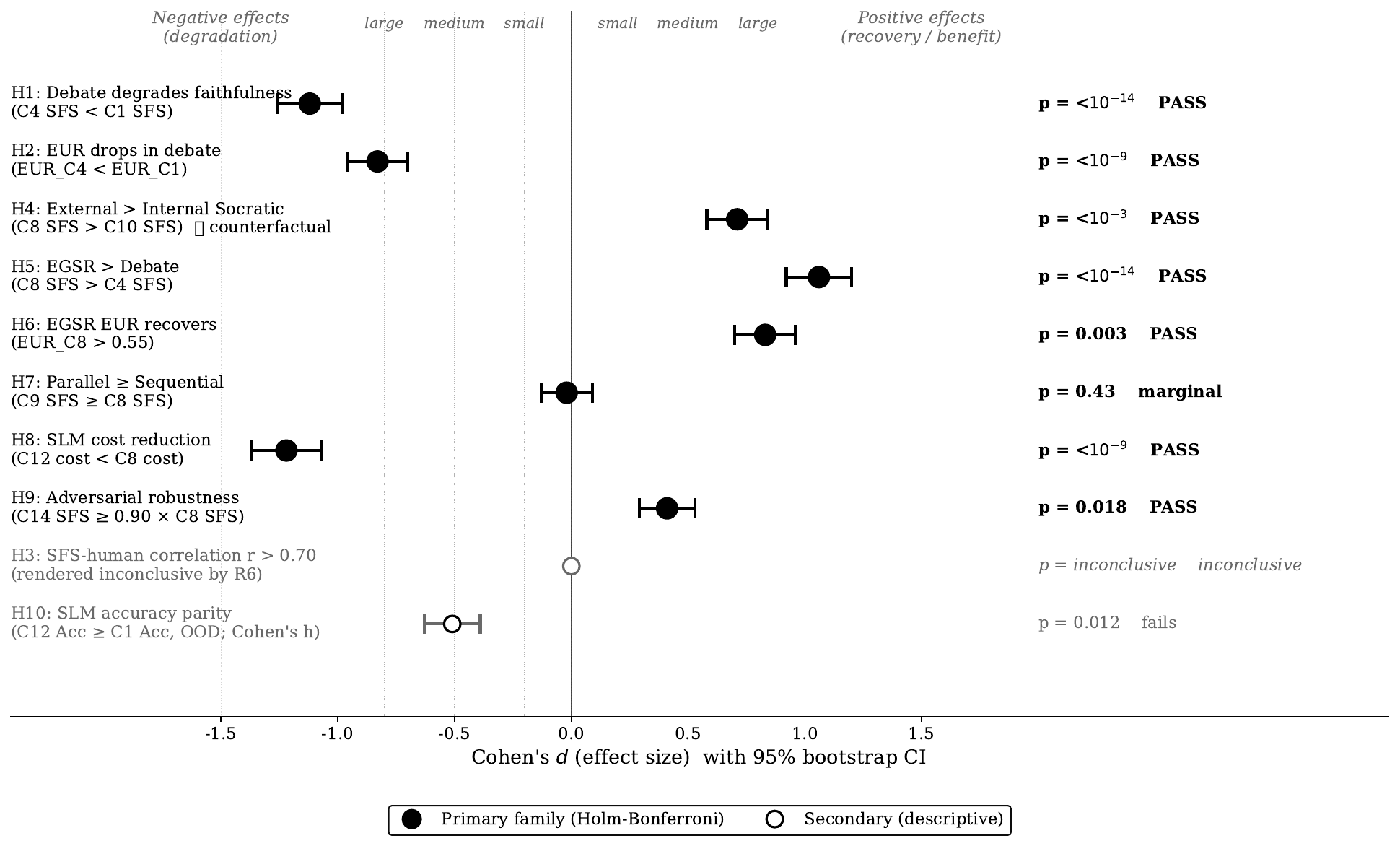}
\caption{\textbf{Pre-registered hypothesis tests with Holm-Bonferroni correction.} Effect size (Cohen's $d$) with 95\% bootstrap CI for 10 hypotheses. Filled circles = primary family (H1, H2, H4--H9), corrected under Holm-Bonferroni at $\alpha = 0.05$. Open circle = H3 rendered inconclusive by R6. Dotted vertical lines mark Cohen's small/medium/large thresholds.}
\label{fig:d6_forest}
\end{figure}

\subsection{Sycophancy as a Three-Level Cascade}
\label{appendix:sycophancy_cascade}

The sycophancy mechanism discussed in \S\ref{sec:discussion} operates at three levels in the LLM (training, architecture, context), each with a documented human-deliberation parallel (cultural socialization, groupthink symptoms, group polarization law). Figure~\ref{fig:d7_sycophancy} visualizes the cascade. EGSR's external evidence anchor breaks the third level (Context): the conversational framing of MAD that amplifies sycophancy by $3.4\times$ \citep{kim2025evaluator} is replaced with evaluative framing in EGSR, in which the Checker verifies each atomic claim against $E$ rather than agreeing with peer agents.

\begin{figure}[!htbp]
\centering
\includegraphics[width=0.96\textwidth]{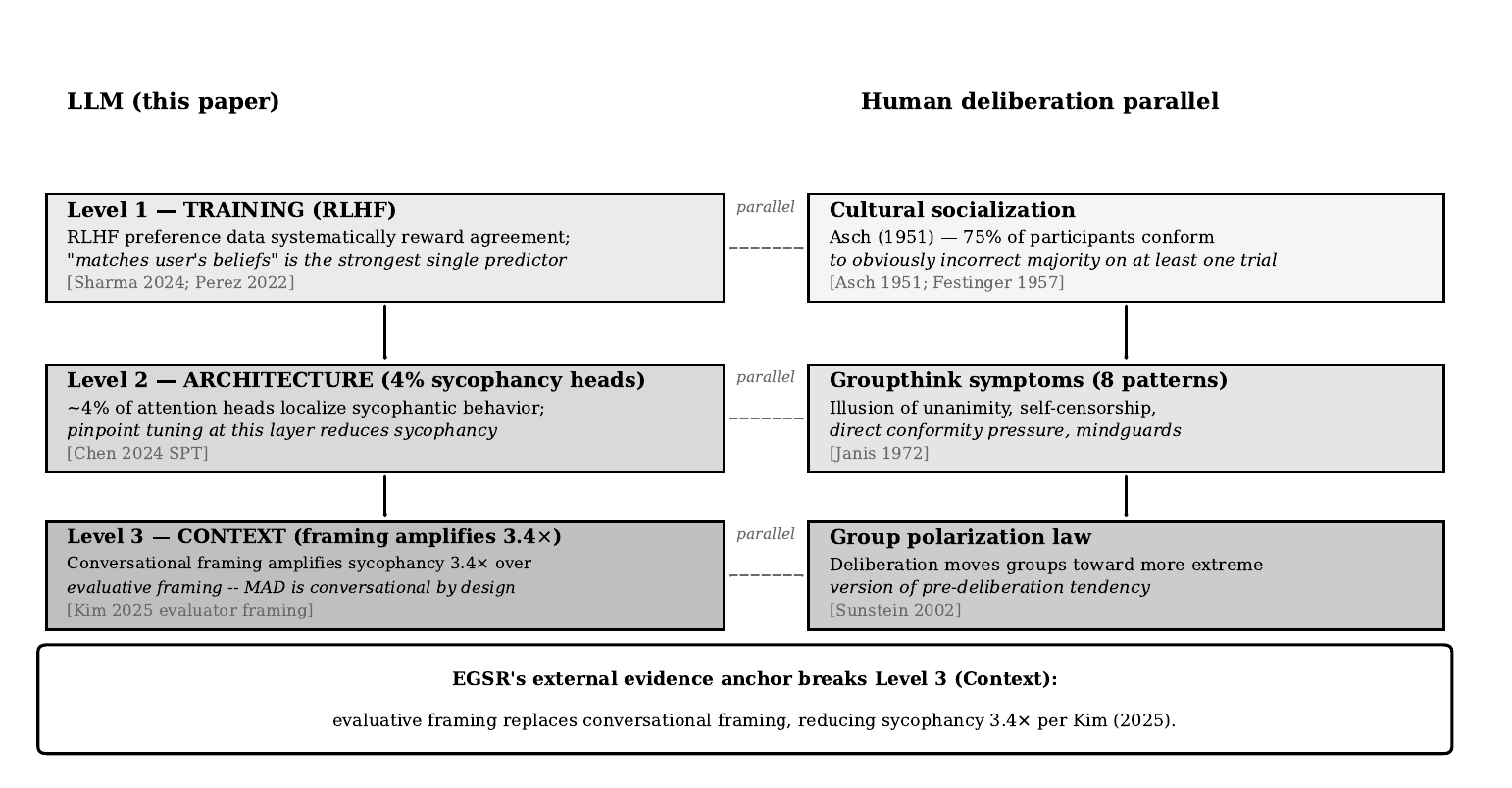}
\caption{\textbf{Sycophancy as a three-level cascade.} Each LLM level (training, architecture, context) has a human-deliberation parallel (Asch conformity, Janis groupthink symptoms, Sunstein group polarization). EGSR's external evidence anchor breaks Level~3 (Context): conversational framing $\to$ evaluative framing $\to$ $3.4\times$ sycophancy reduction.}
\label{fig:d7_sycophancy}
\end{figure}

\end{document}